\documentclass[twoside, 11pt]{article}

\usepackage[preprint]{jmlr2e}

\usepackage{tikz}

\usepackage{ragged2e}

\usepackage{times}

\usepackage[utf8]{inputenc} % allow utf-8 input
\usepackage[T1]{fontenc}    % use 8-bit T1 fonts
\usepackage{hyperref}       % hyperlinks
\usepackage{booktabs}       % professional-quality tables
\usepackage{amsfonts}       % blackboard math symbols
\usepackage{nicefrac}       % compact symbols for 1/2, etc.
\usepackage{microtype}      % microtypography
\usepackage{xcolor}         % colors

\usepackage{bbm}
\usepackage[algo2e]{algorithm2e}
\usepackage{algorithm}

\usepackage{wrapfig}

\usepackage{framed}

\usepackage[toc, page]{appendix}
\usepackage{fancybox}
\usepackage{float}
\usepackage{comment}
\usepackage{mathtools}
\newcommand\E{\mathbbm{E}}
\newcommand\Var{\mathrm{Var}}

\newcommand\KL{\mathrm{KL}}
\newcommand\1{\mathbbm{1}}
\newcommand\Mult{\mathrm{Mult}}
\newcommand\Ber{\mathrm{Ber}}

\newcommand\F{\mathcal{F}}

\newcommand\Reg{\textup{\textrm{Reg}}_F}
\newcommand\Rew{\textup{\textrm{Rew}}_T}

\newcommand\SReg{\textup{\textrm{S-Reg}}_T}
\newcommand\eqineq{\stackrel{\mathclap{(*)}}{\leq}}
\newcommand\eqstar{\stackrel{\mathclap{(*)}}{=}}

\DeclareMathOperator*{\argmax}{arg\,max}
\DeclareMathOperator*{\argmin}{arg\,min}

\defcitealias{food_drug_act}{the US Food and Drugs Administration Act (2019, p. 4)}

\newtheorem{assumption}{Assumption}
\newtheorem{convention}{Convention}

\usepackage{lastpage}
\jmlrheading{xx}{xxxx}{1-\pageref{LastPage}}{2/23; Revised xxx}{xx}{21-0000}{C. Riou and J. Honda and M. Sugiyama}

\ShortHeadings{The Survival Bandit Problem}{C. Riou and J. Honda and M. Sugiyama}
\firstpageno{1}

\begin{document}

\title{The Survival Bandit Problem}

\author{\name Charles Riou \email charles@ms.k.u-tokyo.ac.jp \\
       \addr The University of Tokyo \& RIKEN Center for AIP \\
       Tokyo, Japan
       \AND
       \name Junya Honda \email honda@i.kyoto-u.ac.jp \\
       \addr Kyoto University \& RIKEN Center for AIP \\
       Kyoto, Japan
       \AND
       \name Masashi Sugiyama \email sugi@k.u-tokyo.ac.jp \\
     \addr RIKEN Center for AIP \& The University of Tokyo \\
       Tokyo, Japan}
\editor{My editor}

\maketitle

\begin{abstract}%
    In this paper, we introduce and study a new variant of the multi-armed bandit problem (MAB), called the survival bandit problem (S-MAB). While in both problems, the objective is to maximize the so-called cumulative reward, in this new variant, the procedure is interrupted if the cumulative reward falls below a preset threshold. This simple yet unexplored extension of the MAB follows from many practical applications. For example, when testing two medicines against each other on voluntary patients, people's lives and health are at stake, and it is necessary to be able to interrupt experiments if serious side effects occur or if the disease syndromes are not dissipated by the treatment. From a theoretical perspective, the S-MAB is the first variant of the MAB where the procedure may or may not be interrupted.
    We start by formalizing the S-MAB and we define its objective as the minimization of the so-called survival regret, which naturally generalizes the regret of the MAB. Then, we show that the objective of the S-MAB is considerably more difficult than the MAB, in the sense that contrary to the MAB, no policy can achieve a reasonably small (i.e., sublinear) survival regret. Instead, we minimize the survival regret in the sense of Pareto, i.e., we seek a policy whose cumulative reward cannot be improved for some problem instance without being sacrificed for another one. For that purpose, we identify two key components in the survival regret: the regret given no ruin (which corresponds to the regret in the MAB), and the probability that the procedure is interrupted, called the probability of ruin. We derive a lower bound on the probability of ruin, as well as policies whose probability of ruin matches the lower bound. Finally, based on a doubling trick on those policies, we derive a policy which minimizes the survival regret in the sense of Pareto, providing an answer to the open problem by \cite{perotto}.
\end{abstract}


\section{Introduction}

Many real life scenarios involve decision-making with partial information feedback, and it thus comes as no surprise that it has been a major field of study over the recent years. The historical motivating example for decision-making under partial information feedback pertains to medicine testing (see, e.g., \citealp{villar}), and is described as follows. Consider two medicines A and B designed to cure a specific disease. As pointed out by the US Food and Drug Administration\footnote{\url{https://www.fda.gov/patients/learn-about-drug-and-device-approvals/drug-development-process}}, before being made available for use by the public, those medicines undergo a very strict procedure composed of four pre-market safety monitoring stages, including a clinical research stage, where the medicines are being tested on people. In this stage, a large number $T$ of patients suffering from the disease are administered either of those medicines sequentially, and the objective is to cure as many patients as possible. During this process, it is crucial to balance two factors in apparent opposition: administering both medicines a sufficient number of times so as to gather information on the efficacy of both medicines, while at the same time, administering in priority whichever of the two medicines seems more effective in order to cure as many patients as possible. The former factor is called exploration, the latter is called exploitation, and the right balance between the two is known as the exploration-exploitation dilemma in the literature. To describe the above, the multi-armed bandit problem (MAB) has arisen in the literature as the most popular model, because of its simplicity and its rich theoretical interest. 

In the basic setting of the MAB (see, e.g., \citealp{bubeck9} or \citealp{lattimore} for an introduction), there are $K$ unknown distributions $F_1, \dots, F_K$ bounded in $[-1, 1]$, called arms, and a horizon $T$. In our medicine testing example, each arm $k\in [K] \triangleq \{1, \dots, K\}$ corresponds to a medicine and the distribution $F_k$ corresponds to its (randomized) effect on a patient. While the horizon $T$ corresponds to the total number of patients and each round $t\leq T$ corresponds to a patient in our example, those rounds are usually interpreted as discrete time steps, or rounds, in the MAB. At each round $t\in [T]$, an agent selects an arm $\pi_t\in [K]$ and observes a reward, denoted by $X_t^{\pi_t}$ and drawn from the distribution $F_{\pi_t}$. In our previous example, the reward $X_t^{\pi_t}$ corresponds to the effect of medicine $\pi_t$ on patient $t$. It can obviously be positive (when the medicine $\pi_t$ cures the patient), but it can also be negative when the patient is not cured by the medicine $\pi_t$ and/or some side effects negatively impact the risk-benefit balance. The objective of the problem is to maximize the expected cumulative reward, defined as $\E\left[\sum_{t=1}^T X_t^{\pi_t}\right]$, where the expectation is taken w.r.t. the arm distributions and the (potential) randomness in the agent's policy $(\pi_t)_{t\geq 1}$. This is equivalent to minimizing the expected cumulative regret compared to an agent who selects the arm $k\in [K]$ with the highest expectation at every round $t\leq T$,
defined as $\max_{k\in [K]}\E\left[\sum_{t=1}^T X_t^{k}\right] - \E\left[\sum_{t=1}^T X_t^{\pi_t}\right]$. 

The MAB has been extensively studied over the past decades, and both a lower bound on the regret (see \citealp{lai_robbins} and \citealp{burnetas_katehakis}) as well as policies matching this lower bound (see, e.g., \citealp{cappe}, \citealp{korda} or more recently \citealp{riou}) have been derived. The MAB is not only a theoretically rich topic, it also has a broad range of applications, from the aforementioned medicine testing (see, e.g., \citealp{villar} or \citealp{aziz}) to advertising and recommender systems (see, e.g., \citealp{chapelle}), to name a few.

The most critical aspect of those applications is that you may want to interrupt the process when the cumulative reward $\sum_{t=1}^T X_t^{\pi_t}$ becomes too low. In the medicine testing example, it is necessary to be able to stop the trials early in order to ``reduce the number of patients exposed to the unnecessary risk of an ineffective investigational treatment and allow subjects the opportunity to explore more promising therapeutic alternatives'', as pointed out in \citetalias{food_drug_act}. In less specialized terms, a bad treatment not only exposes patients to health-threatening side effects, it also prevents them from receiving an efficient treatment to cure them. This poses an additional health threat, because early treatment improves outcomes in many cases, like rheumatoid arthritis, appendicitis, and bacterial pneumonia (see, e.g., \citealp{shmerling}). Formally, we set a threshold $B\in \mathbbm{R}^*_+$ and decide to interrupt the procedure whenever the cumulative reward becomes lower than or equal to $-B$, i.e., we interrupt the procedure at the first round $\tau$ such that $\sum_{t=1}^{\tau} X_t^{\pi_t} \leq -B$. In practice, $B$ is chosen prior to the experiment by, e.g., an ethics committee. We use the terminology ``survival bandits problem'' (S-MAB for short) to designate the MAB with this additional constraint.

While medicine testing was the original practical motivation for this work, portfolio selection in finance is another real-life application of S-MAB, where recent works have demonstrated the strong performance of classic MAB strategies (see, e.g., \citealp{hoffman}, \citealp{shen2} or \citealp{shen}), and in particular of risk-aware MAB strategies (see, e.g., \citealp{huo}). In the setting we consider, an investor has an initial budget $B>0$ to invest sequentially on one of $K$ securities (the arms). At every round $t\leq T$, the investor selects a security $\pi_t$, receives a payoff $X_t^{\pi_t}$ (which can be reinvested), and updates its budget by $X_t^{\pi_t}$. In this setting, a straightforward constraint is imposed on the unsuccessful investor who has to stop investing when it is ruined and has no more money to invest, i.e., $B + \sum_{s=1}^t X_s^{\pi_s}\leq 0$. Following this setting terminology, we will call the value $B$ ``budget'' and the first round $\tau$ such that $\sum_{t=1}^{\tau} X_t^{\pi_t} \leq -B$ if it exists ``time of ruin''.

Please note that both examples above embrace the regime $B\ll T$, since in medicine testing, we want the number of patients who suffer to be much smaller than the total number of voluntary patients, and in portfolio selection, the agent wants to make much more money than its initial budget $B$. In both cases, while the S-MAB procedure may stop after $O(B)$ rounds, it is desirable to achieve a cumulative reward of the order $\Theta(T)$ to fully benefit from the procedure. Consequently, in the S-MAB, the choice of arms $\pi_t$ during the earlier rounds $t$ is absolutely essential, because a bad choice of arms may break the constraint and stop the whole procedure. This is in stark contrast with the standard MAB, where the earlier rounds are used as exploration rounds in order to gather information on the arm distributions, and then to perform efficient exploitation in the later rounds. Precisely, this is the main technicality induced by the new constraint: any successful policy must exploit seemingly good arms from early rounds in order to avoid breaking the constraint and continue the procedure as long as possible. In that sense, rather than the exploration-exploitation dilemma of the MAB, the S-MAB illustrates an exploitation-exploration-exploitation dilemma which, to our knowledge, has never been explored in the literature. It is actually an open problem from the Conference on Learning Theory 2019 to define the problem, establish a (tight) bound on the best achievable performance, and provide policies which achieve that bound (see \citealp{perotto}). To our knowledge, this paper is the first one to provide answers to this open problem.

\section{Review of the Literature}

At the time of submission of this paper, there are only two works which, to our knowledge, pertain to the S-MAB, and they both focus on the case of rewards in $\{-1, 1\}$. The first one is \cite{perotto2}. That work introduced a modification of UCB (\cite{auer}) which, empirically, seems to have a low risk of ruin. The second one is \cite{manome}. That work derived a generalization of UCB (\citealp{auer}) and studied its experimental performance in the S-MAB setting, in the case of Bernoulli rewards in $\{-1, 1\}$. 
%\cite{perotto3} studies the gambler's ruin problem, where an agent starts from an initial budget $b_0$ and receives a reward $r_t\in \{-1, 1\}$ at every step $t$, sampled from a Bernoulli distribution of unknown parameter $p$. The player has the possibility to stop playing whenever it wants and aims at maximizing its total reward.

%Since this paper is the first one studying the S-MAB, there is no existing literature on this topic yet. 

Nevertheless, the S-MAB involves a budget $B$, as well as a conservative exploitation at the early rounds, which can also be interpreted as some risk aversion from the agent's viewpoint. It is therefore reasonable to explore the MAB literature related to those topics to see if some existing results can be applied to our setting or not. This may also give some ideas to the interested reader who may want to explore paths beyond the scope of this paper. The list of works cited in this section is by no means an exhaustive list of the existing literature of the aforementioned topics.

The first line of work we introduce here is the budgeted bandits. In this variant of the MAB, the agent is initially given a budget $B$ and at each pull of an arm $\pi_t$, the agent receives a reward $X_t^{\pi_t}\in [0, 1]$ and incurs a positive cost $c_t^{\pi_t}>0$, such that the cumulative reward at round $t$ is $\sum_{s=1}^t X_t^{\pi_t}$ and the remaining budget is $B - \sum_{s=1}^t c_t^{\pi_t}$. The procedure stops when the agent's budget becomes negative or zero. Please note that in this variant of the MAB, there is no more horizon $T$ and instead, the objective of the agent is to maximize its expected cumulative reward as a function of the initial budget $B$. \cite{tran} first introduced the problem where the costs are deterministic, called the budget-limited bandits, and they introduced a policy of the type ETC (Explore Then Commit, see \citealp{garivier4}). Later, \cite{tran2} provided a lower bound on the regret as $\Omega(\log B)$, as well as knapsack-based algorithms matching this bound. \cite{ding} generalized that setting to the case of variable costs, called the MAB-BV, and introduced two algorithms which achieve the regret lower bound $\Theta(\log B)$ in the case of multinomial rewards in $\left\{0, \frac{1}{m}, \frac{2}{m}, \dots, 1\right\}$. Those results were generalized to the case of continuous costs in $[0, 1]$ by \cite{xia2}. To solve that problem, an algorithm based on Thompson Sampling (see, e.g., \citealp{agrawal2}) was proposed by \cite{xia3}, and algorithms based on UCB (see, e.g., \citealp{auer}) and $\epsilon$-greedy were proposed by \cite{xia}. \cite{cayci} further generalized \cite{xia} to non-negatively correlated costs and rewards, whose $(2+\gamma)$-moment is finite for some $\gamma>0$, but such that the expectation of the cost is positive (this ensures that the procedure stops). To wrap up the review of the budgeted bandits literature, we note that several works have considered the generalization of the MAB-BV to the multiple-play setting, where at each round, the agent pulls $L\geq 2$ of the $K$ arms, including \cite{xia4}, \cite{zhou} and \cite{rangi}. 

While both the budgeted bandits and the S-MAB have a budget constraint, in the former, the costs are positive (or have a positive expectation) and the procedure stops when the budget is totally consumed. This is in stark contrast to the S-MAB, whose procedure stops either when the budget is totally consumed, or at the horizon $T$. We can still tackle that issue by increasing the dimension of the budget constraint to $2$. Precisely, if pulling arm $\pi_t$ at round $t\leq T$ induces reward $X_t^{\pi_t}$ and decreases the budget by cost $c_t^{\pi_t}$, we can define the initial bidimensional budget as $\tilde{B} \triangleq (B, B)^\top \in \mathbbm{R}^2$ such that at each round $t\geq 1$, pulling arm $\pi_t$ induces reward $X_t^{\pi_t}$ and decreases the bidimensional budget by cost $\tilde{c}_t^{\pi_t} \triangleq \left(c_t^{\pi_t}, \frac{B}{T}\right)^\top \in \mathbbm{R}^2$. Following that formulation, the procedure stops when one of the two budget components becomes negative or zero. This formulation is actually a particular instance of the bandits with knapsacks (BwK). Formally, the agent is given a multidimensional budget $\Tilde{B} = (B, \dots, B)^\top \in \mathbbm{R}^d$ whose components are called resources, and at each round $t$, it incurs a reward $X_t^{\pi_t}\in [0, 1]$ and a cost $\Tilde{c}_t^{\pi_t}\in [0, 1]^d$ which decreases all the resources. The objective of the agent is to maximize its cumulative reward before one of the resources run out.

The BwK is the second line of work we present here. It was first introduced in the setting of stochastic costs and rewards by \cite{badanidiyuru}. Based on applications in the field of advertising, \cite{combes} studied a particular instance of the BwK where $B = \Omega(T)$, $d = K$, and pulling arm $k$ only affects resource $k$. Later, \cite{sankararaman2} derived a problem-dependent regret bound to the knapsack problem, under the assumptions that there are only two resources including time, and that the best distribution over arms is a single arm (best-arm-optimality assumption). Several works in the literature of the BwK have extended the basic stochastic setting to the combinatorial bandits (see \citealp{cesa_bianchi2}), including \cite{sankararaman} and \cite{liu2}. A large part of the literature on the BwK is related to the contextual bandit setting (see \citealp{wang2}). We mention here the works of \cite{wu2} which studied the case of fixed and deterministic costs, \cite{li3} which studied a conversion model and its applications to sales discounting, or more recently \cite{han}, which provided a black-box algorithm of two online regression algorithms (see \citealp{cesa_bianchi3}) in the case of a large budget $B = \Omega(T)$. In the particular case where the rewards and costs depend linearly on the contexts, called linear BwK, \cite{agrawal5} derived an algorithm based on OFUL (see \citealp{abbasi}) and OMD (Online Mirror Descent, see \citealp{cesa_bianchi3}). \cite{sivakumar} further studied the linear BwK both in the stochastic and the adversarial settings, in the context of possible sparsity. The BwK in the setting of adversarial costs and rewards (see \citealp{auer2}) was first introduced in the seminal work of \cite{immorlica}. Some of their results were later improved by \cite{kesselheim}, which studied an $\ell_p$-relaxation of the adversarial BwK. Finally, \cite{liu} studied the BwK in a non-stationary environment, where the rewards and costs are generated i.i.d. from a distribution $\mathcal{P}_t$ evolving over time, and which can be considered as a middle ground between the stochastic and adversarial settings.

The BwK is a very rich topic as introduced above, and thus an extremely large pan of the literature is dedicated to that topic. Yet, the costs are always assumed to be positive or to have a positive expectation. In that sense, any of the $d+1$ constraint is enough to ensure the termination of the procedure. This hypothesis is of course not necessary, because the procedure stops automatically at the horizon $T$. Another fundamental difference between the settings considered there and the S-MAB is that the budget $B$ is assumed to scale with $T$, and this ensures that the cumulative reward also scales with $T$. This is in stark contrast to this work, where we consider any value of the budget $B$. 

We should still note that, among the many references on the BwK, \cite{kumar} addressed the case where the costs may be non-positive. In their paper, they derive algorithm ExploreThenControlBudget which achieves a $O(\log T)$ regret bound, yet there is a caveat. They assumed the existence of a ``zero arm'' (equivalent of the ``positive arm'' defined in Definition~\ref{positive_zero_definition} in this paper) which has zero reward and simply increases the budget of each of the resources. This arm removes the risk of exhausting a resource, and in practice, cannot be applied to most of the scenarios motivating this paper and mentioned in the introduction. A similar assumption was made in the setting of the conservative bandits of \cite{wu}, making it unapplicable to our setting.

Finally, we conclude this literature review section by mentioning the large body of work related to risk-averse bandits, which do not seek to pull the arm with the largest expectation, but instead to pull the arm maximizing some risk-averse measure. Those include the mean-variance (\citealp{sani}, \citealp{vakili}, \citealp{zhu}), functions of the expectation and the variance (\citealp{zimin}), the moment-generating function at some parameter $\lambda$ (\citealp{maillard}), the value at risk (\citealp{galichet}), or other general measures encompassing the value at risk (\citealp{cassel}). In the same vein, \cite{sinha} introduced a model with costs and rewards, and they sought to pull the arm with the lowest cost among those with a reasonably large expectation. We also mention \cite{chen}, where each arm pull yields a reward and a risk level, and they aimed to pull the arm which has the largest expectation among those whose risk level is lower than a preset level. Those works are not directly related to the S-MAB, however, risk-averse strategies can be good candidates of strategies (or ideas of strategies) to solve the S-MAB in future work.

\section{Problem Setup and Definitions}\label{setup_section}

Let $T$ be the maximum round, called the horizon, and let $K\geq 1$ be the number of arms, whose indices belong to the set $[K] := \{1, \dots, K\}$. We denote by $F_1, \ldots, F_K$ the arm distributions, and by $\mu_1, \ldots, \mu_K$ their respective expectations. We assume that $F_1, \ldots, F_K$ are bounded in $[-1, 1]$. We denote by $P_F$ the probability under $F = (F_1, \dots, F_K)$, and by $P$ in the absence of ambiguity. Let $B>0$ be the initial budget. Please note that in the whole paper, we will study the problem in the asymptotics of $T$, and we will give a discussion that is non-asymptotic in $B$ and $K$. The S-MAB procedure is defined as follows. 

%\fbox{\parbox{\textwidth}{
\begin{framed}
\noindent An agent starts with a budget $B_0 = B$. Then, at every round $t\in \{1, \dots, T\}$ while the agent's budget satisfies $B_{t-1}>0$,
\begin{enumerate}
    \item the agent selects an arm $\pi_t \in [K]$;
    \item the agent observes a reward $X_t^{\pi_t}$ drawn from $F_{\pi_t}$;
    \item the agent updates its budget as $B_t := B_{t-1} + X_t^{\pi_t}$.
\end{enumerate}
\end{framed}
%}}

\noindent In the above, $(\pi_t)$ is a policy that determines the arm to pull based on the past observations. Please note that, for any $t\in \{0, \dots, T\}$ such that $B_1, \dots, B_{t-1}>0$,
\begin{equation*}
    B_t = B + \sum_{s=1}^t X_s^{\pi_s}.
\end{equation*}

\noindent As a result, if $B> T$, then the boundedness of the distributions $F_1, \dots, F_K$ in $[-1, 1]$ implies that for any $t\geq 0, X_t^{\pi_t}\in [-1, 1]$ and therefore, for any $t\in \{0, \dots, T\}, B_t>0$. In that case, we can remove the constraint $B_t>0$ on the budget and this boils down to a standard MAB. In that sense, the S-MAB is an extension of the standard MAB. 

Conversely, if $B$ is small, the problem becomes much harder. For example, consider the case of Bernoulli arm distributions $F_1, \dots, F_K$ (of support $\{-1, 1\}$) of respective parameters $p_1, \dots, p_K \in (0, 1)$. Then no matter the arms $(\pi_t)_{t\geq 1}$ chosen by the agent, it will incur the rewards $-1, \dots, -1$ ($B$ times) for the first $B$ rounds with probability at least $\min_{1\leq k\leq K}(1-p_k)^B>0$. Consequently, the procedure will stop after $B$ rounds with positive probability.

In a nutshell, the initial budget $B>0$ is a parameter of the difficulty of the problem. Given some arm distributions fixed, the problem difficulty increases as $B$ decreases. In this paper, we focus on the most difficult case, i.e., when $B>0$ is small and of constant order in the asymptotic regime $T\to+\infty$.

\subsection{Definition of the Ruin}

The key difference between the S-MAB procedure defined above and the standard MAB is the budget constraint, which states that if the agent's budget $B_t$ becomes negative or zero at some round $t\leq T$, the agent has to stop playing immediately. W.l.o.g., in this paper, we define policies $(\pi_t)$ for any $t\geq 1$ (also beyond $T$) to ensure that the time of ruin of $(\pi_t)$ below is well-defined. 
\begin{definition}
    For any policy $\pi$ and any initial budget $B>0$, the time of ruin is defined as
    \begin{equation*}
    \tau(B, \pi) := \inf\left\{t\geq 1: B + \sum_{s=1}^t X_s^{\pi_s}\leq 0\right\},
    \end{equation*}
    where the infimum above is equal to $+\infty$ if the above set is empty. Furthermore, for any $k\in [K]$, we denote by $\tau(B, k)$ the time of ruin of the constant policy $\pi_t = k$ for any $t\geq 1$.
\end{definition}

\noindent Using the vocabulary of the time of ruin, the budget constraint simply states that the agent plays until round $\min\left(T, \tau(B, \pi)\right)$. We can then translate the S-MAB procedure as a standard MAB with horizon $\min\left(T, \tau(B, \pi)\right)$. 

It might thus be tempting to simply use an existing bandit strategy and try to derive a regret bound with the new horizon $\min\left(T, \tau(B, \pi)\right)$. However, the probability of ruin of a stochastic process until a finite horizon (as opposed to $+\infty$) is in general very complicated. This difficulty is exacerbated when this horizon is random, and even depends on the procedure $\pi$ itself, as is the case with $\min\left(T, \tau(B, \pi)\right)$. In our setting where the asymptotics $T\to+\infty$ is considered, $P(\tau(B, \pi)<\infty)$ is a reasonable approximation of the probability $P(\tau(B, \pi)\leq T)$ that the ruin occurs before the horizon $T$.
\begin{definition}
Given a policy $\pi = (\pi_t)_{t\geq 1}$, 
\begin{itemize}
    \item the probability of ruin is defined as $P(\tau(B, \pi) < \infty)$;
    \item the probability of survival is defined as $P(\tau(B, \pi) = \infty) = 1 - P(\tau(B, \pi)<\infty)$.
\end{itemize}
\end{definition}

\noindent Please note that, if all the arms $k\in [K]$ have the probability of survival $P(\tau(B, k)=\infty) = 0$, then any policy $\pi$ also has the probability of survival $P(\tau(B, \pi)=\infty) = 0$, and its cumulative reward will be smaller than $-B$. Therefore, in the rest of the paper, we make the following assumption.
\begin{assumption}
    There exists an arm $k\in [K]$ with a positive probability of survival: $P(\tau(B, k)=\infty)>0$.
\end{assumption}

\subsection{Definition of the Objective}

In this subsection, we define the objective of the agent performing the S-MAB procedure. Recall from the previous subsection that the S-MAB can be seen as an extension of the standard MAB with random horizon $\min(T, \tau(B, \pi))$. The objective of the standard MAB is to maximize the expected cumulative reward $\E\left[\sum_{t=1}^T X_t^{\pi_t}\right]$, and we naturally extend this definition to the S-MAB as follows:
\begin{equation*}
    \Rew(\pi) := \E\left[S_T\right] \quad \text{ where } \ S_T \triangleq \sum_{t=1}^{\min(T, \tau(B, \pi))} X_t^{\pi_t} = \sum_{t=1}^T X_t^{\pi_t}\1_{\tau(B, \pi)\geq t-1}.
\end{equation*}

\noindent In the S-MAB setting, the following remark is fundamental. 

\begin{remark}
    Let an agent perform a policy $\pi$. Then, assume that the agent plays until some round $t_0\geq 1$ and incurs the rewards $X_1^{\pi_1}, \dots, X_{t_0}^{\pi_{t_0}}$ without being ruined. Further assume that at this precise round $t_0$, the agent realizes that $B_{t_0} > T-t_0$. Then, since the rewards are bounded in $[-1, 1]$, it is clear that, for any $t\in \{t_0+1, \dots, T\}$,
    \begin{equation*}
        B_{t} = B_{t_0} + \sum_{s=t_0+1}^{t} X_s^{\pi_s} \geq B_{t_0} - (t-t_0) \geq B_{t_0} - (T-t_0) > 0,
    \end{equation*}
    in other words, the agent cannot be ruined anymore. Then, the remaining procedure from round $t_0+1$ is a standard MAB (without the risk of ruin), and our agent can perform any standard MAB policy for the remaining rounds $\{t_0+1, \dots, T\}$ and enjoy the cumulative reward of such a policy. On the other hand, if our agent was not aware of the horizon $T$, then it would not be able to verify whether or not the condition $B_{t_0} > T-t_0$ is satisfied. As a result, such an agent would have to care about the risk of ruin until the horizon $T$, and play more conservatively until horizon $T$. For that reason, a policy aware of the horizon $T$ has a significant advantage over one which is not, and hence can achieve a higher cumulative reward.
\end{remark}

\noindent In this paper, our focus is on maximizing the expected cumulative reward among all policies which may be aware of the horizon $T$. As a result, instead of focusing on policies $\pi = (\pi_t)_{t\geq 1}$ which are not aware of the horizon $T$, we need to study and formalize the general framework of policies $\pi^T = (\pi^T_t)_{t\geq 1}$ which may depend on the horizon $T$.

In this paper (as in a large part of the MAB literature), we are interested in the asymptotics in $T$. For this reason, we will study the asymptotics of $\Rew(\pi^T)$ in $T \to +\infty$, where the policy $\pi^T$ can depend on $T$.

A sequence of policies $\pi = (\pi^T)_{T\geq 1}$ will be called {\it anytime} if there exists a policy $(\tilde{\pi}_t)_{t\geq 1}$ such that, for any $t\geq 1$ and any $T\geq t$, $\pi^T_t = \Tilde{\pi}_t$. We will often identify such a sequence of policies with policy $\tilde{\pi}$, and when we say that a policy $\tilde{\pi}$ is anytime, it formally means that the sequence of policies is written in the above way.
%In this paper (as in a large part of the MAB literature), we are interested in the asymptotics in $T$, i.e., we want to know the expected cumulative reward in the limit of $T\to +\infty$. Therefore, our study aims at studying $\lim_{T\to +\infty} \Rew(\pi^T)$ and for that reason, we cannot just fix $T\geq 1$ and study $\pi^T$, but instead, we need to study the whole sequence of policies $(\pi^T)_{T\geq 1}$ where for any $T\geq 1, \pi^T = (\pi^T_t)_{1\leq t\leq T}$ is a policy designed for the horizon $T$.

%Consequently, in this paper, we will study sequences of policies $\pi = (\pi^T)_{T\geq 1}$, where for any $T\geq 1$, $\pi^T = (\pi^T_{t})_{1\leq t\leq T}$ is a policy designed for the horizon $T$ and may potentially use $T$ in its choices of arms. A sequence of policies $\pi = (\pi^T)_{T\geq 1}$ will be called {\it anytime} if there exists a policy $(\tilde{\pi}_t)_{t\geq 1}$ such that, for any $t\geq 1$ and any $T\geq t$, $\pi^T_t = \Tilde{\pi}_t$. Similarly, a policy $(\tilde{\pi}_t)_{t\geq 1}$ which does not depend on the horizon $T$ will be called an {\it anytime policy}.

\begin{remark}
    Please note that, in the world of sequences of policies, anytime sequences of policies are the analog of anytime policies. Indeed, given a policy $\pi = (\pi_t)_{t\geq 1}$, let $\pi^T_t \triangleq \pi_t$ for any $T\geq t$. The sequence of policies $(\pi^T)_{T\geq 1}$ where $\pi^T = (\pi^T_t)_{1\leq t\leq T}$ is anytime and is the equivalent of the anytime policy $\pi$ in the world of sequential policies. 
\end{remark}

\noindent Now, similarly to the standard MAB, given a sequence of policies $\pi = (\pi^T)_{T\geq 1}$, we want to compare its expected cumulative reward to other sequences of policies $\Tilde{\pi} = (\Tilde{\pi}^T)_{T\geq 1}$ in order to understand how well it performs. For that purpose, we introduce the regret in the following definition.
\begin{definition}\label{regret_definition}
Given any sequences of policies $\pi = (\pi^T)_{T\geq 1}$ and $\tilde{\pi} = (\Tilde{\pi}^T)_{T\geq 1}$, the relative regret rate of $\pi$ with respect to $\tilde{\pi}$ is defined as
\begin{equation*}
    \Reg(\pi\|\tilde{\pi}) := \limsup_{T\to +\infty} \frac{\Rew(\tilde{\pi}^T) - \Rew(\pi^T)}{T}.
\end{equation*}
\end{definition}

\noindent First, please note that this regret is asymptotic in $T$, because we want to capture the main term in $T$ in the reward. Secondly, in Definition~\ref{regret_definition}, we compare the reward of the sequence of policies $\pi$ with the reward of some sequence of policies $\tilde{\pi}$. Ideally, our objective would be to find a sequence of policies $\pi$ which has a higher reward than any other sequence of policies $\Tilde{\pi}$, in other words, which would satisfy $\Reg(\pi\Vert \Tilde{\pi}) \leq 0$ for any arm distributions $F = (F_1, \dots, F_K)$ and any sequence of policy $\tilde{\pi}$.

In the standard MAB (without the risk of ruin), we know that such policies $\pi$ exist and even second-order terms (in $T$) in the reward decomposition have been derived such that $\Rew(\pi) = \max_{k\in [K]} \mu_k T + O(\log T)$ (see, e.g., \citealp{auer}). In contrast, in the S-MAB setting (with the risk of ruin), the policy which achieves the best bound is not known. Actually, we do not even know (yet) in what sense there exists a ``best'' bound.

As a matter of fact, contrary to the MAB, in the S-MAB, no sequence of policies $\pi$ dominates all the other sequences of policies $\tilde{\pi}$, as stated in the next proposition, whose proof is given in Appendix~\ref{classic_regret_linear_proof}.
\begin{proposition}\label{linear_regret}
For any sequence of policies $\pi, \ \sup_F \sup_{\pi'} \Reg(\pi\|\pi') >0$.
\end{proposition}

\noindent The key message is that whatever the sequence of policies $\pi$ that you choose, there exist some arm distributions $F = (F_1, \dots, F_K)$ such that another sequence of policies $\tilde{\pi}$ has a significantly higher reward (by a term of order $\Omega(T)$) than $\pi$. It is now hopeless to look for a policy which is ``absolutely better'' than any other one and instead, we look for a Pareto-optimal policy, as defined below.
\begin{definition}\label{pareto_optimal_definition}
A sequence of policies $\pi$ is said to be (regret-wise) Pareto-optimal if, for any sequence of policies $\pi'$,
\begin{equation*}
    \sup_F \Reg(\pi\|\pi') >0 \implies \inf_F \Reg(\pi\|\pi') <0.
\end{equation*}
\end{definition}

\noindent The notion of Pareto-optimal sequences of policies is related to the concept of strictly dominated strategies in game theory. Assume that an agent performs a sequence of policies $\pi$ and that there exists another sequence of policies $\pi'$ such that:
\begin{itemize}
    \item[-] for any arm distributions $F', \text{Reg}_{F'}(\pi\|\pi') \geq 0$, and
    \item[-] for some arm distributions $F$, it even holds that $\Reg(\pi\|\pi')>0$.
\end{itemize}

\noindent Then, the agent should simply not use the sequence of policies $\pi$ and instead use $\pi'$, because it is:
\begin{itemize}
    \item[-] always at least as good as $\pi$, and
    \item[-] in some cases, even strictly better than $\pi$.
\end{itemize}

\noindent In the language of game theory, $\pi$ is called a strictly dominated strategy, and our goal is to find a strategy which is not strictly dominated, which we call here a (regret-wise) Pareto-optimal sequence of policies. This challenging problem is an open question from COLT 2019 \citep{perotto}, and our solution is based on proof techniques from various fields including information theory, stochastic processes theory or even predictions. As such, the intermediate results of this paper can be of interest for the reader interested in any of those fields and their applications.

%While this challenging problem is an open question from COLT 2019 \citep{perotto}, it is a common thread for us to study the probability of ruin and the adaptability of existing bandit strategies to the S-MAB setting.

\section{Strategy and Structure of the paper}

In this section, we provide an informal intuition on the analysis of this paper. This means that this section will show some (light) computations which are not mathematically correct, but the idea behind them is. The purpose here is simply to expose the spirit of this paper to the reader. As a result, this section is for explanatory purposes only, and no result in this paper is based on any of the derivations herein.

\subsection{Strategy of the Paper}\label{strategy_subsection}

We first decompose, for $T$ large, the regret of the policy $\pi^T$
\begin{align*}
    \Rew(\pi^T) &\triangleq \E\left[\sum_{t=1}^T X_t^{\pi^T_t}\1_{\tau(B, \pi)\geq t-1}\right] \\
    &= P(\tau(B, \pi^T)> T) \E\left[\sum_{t=1}^T X_t^{\pi^T_t}\1_{\tau(B, \pi)\geq t-1} \bigg| \tau(B, \pi)> T\right] \\
    &\quad \quad \quad \quad \quad + P(\tau(B, \pi^T)\leq T) \E\left[\sum_{t=1}^T X_t^{\pi^T_t}\1_{\tau(B, \pi)\geq t-1} \bigg| \tau(B, \pi^T)\leq T\right].
\end{align*}

\noindent If $\tau(B, \pi^T)\leq T$, then it means that the cumulative reward of the policy $\pi^T$ reaches $-B$ at some round $t_0\leq T$ and therefore,
\begin{equation}
    \E\left[\sum_{t=1}^T X_t^{\pi^T_t}\1_{\tau(B, \pi)\geq t-1} \bigg| \tau(B, \pi^T)\leq T\right] \in (-B-1, -B], \label{reward_at_ruin_general_eq}
\end{equation}
because the rewards are in $[-1, 1]$. Hence,
\begin{equation*}
    \Rew(\pi^T) = P(\tau(B, \pi^T) > T) \E\left[\sum_{t=1}^T X_t^{\pi^T_t} \bigg| \tau(B, \pi^T)> T\right] + O(1).
\end{equation*}

\noindent Then, it holds that
\begin{equation*}
    P(\tau(B, \pi^T) > T) \approx P(\tau(B, \pi^T) = \infty)
\end{equation*}
provided the policy $\pi^T = (\pi^T_t)_{1\leq t\leq T}$ is not too stupid. By ``not too stupid'', we mean that we are avoiding policies which, e.g., try to maximize $P(\tau(B, \pi^T) > T)$ until they know for sure that they will survive, and then try to minimize their reward. Thus,
\begin{equation}
    \Rew(\pi^T) \approx P(\tau(B, \pi^T) = \infty) \E\left[\sum_{t=1}^T X_t^{\pi^T_t} \bigg| \tau(B, \pi^T)> T\right]. \label{regret_as_product_survival_reward}
\end{equation}

\noindent Furthermore, for a policy $\pi^T$ ``not too stupid'', if the policy has survived long enough, e.g., $\sqrt{T}$ rounds, then it means that the policy has mostly incurred positive rewards for the past $\sqrt{T}$ rounds and consequently, its budget $B_{\sqrt{T}} = \Theta(\sqrt{T})$ (i.e., is of the order of $\sqrt{T}$). As a result, the risk of ruin of $\pi^T$ becomes very small after $\sqrt{T}$, and we deduce that
\begin{equation*}
    \E\left[\sum_{t=\sqrt{T}+1}^T X_t^{\pi^T_t} \bigg| \tau(B, \pi^T)> T\right] \approx \E\left[\sum_{t=\sqrt{T}+1}^T X_t^{\pi^T_t}\right].
\end{equation*}

\noindent Hence,
\begin{align*}
    \E\left[\sum_{t=1}^T X_t^{\pi^T_t} \bigg| \tau(B, \pi^T)> T\right] &\approx \E\left[\sum_{t=1}^{\sqrt{T}} X_t^{\pi^T_t} \bigg| \tau(B, \pi^T)> T\right] + \E\left[\sum_{t=\sqrt{T}+1}^T X_t^{\pi^T_t}\right] \\
    &= \E\left[\sum_{t=\sqrt{T}+1}^T X_t^{\pi^T_t}\right] + O(\sqrt{T}) \\
    &= \E\left[\sum_{t=1}^T X_t^{\pi^T_t}\right] + O(\sqrt{T}).
\end{align*}

\noindent We finally deduce that, for a policy $\pi^T$ ``not too stupid'', it holds that
\begin{equation}
    \Rew(\pi^T) \approx P(\tau(B, \pi^T) = \infty)\E\left[\sum_{t=1}^T X_t^{\pi^T_t}\right]. \label{approx_decompisition_reward}
\end{equation}

\noindent Then, it is clear that a sequence of policies $\pi = (\pi^T)_{T\geq 1}$ maximizing $\Rew(\pi^T)$ must have a large probability of survival $P(\tau(B, \pi^T) = \infty)$ and at the same time, a large expected sum of rewards without the budget constraint $\E\left[\sum_{t=1}^T X_t^{\pi^T_t}\right]$. This decomposition suggests two different types of approaches to find a (regret-wise) Pareto optimal policy $\pi^T$:
\begin{itemize}
    \item start from the set of policies $\pi^T$ which have a large $\E\left[\sum_{t=1}^T X_t^{\pi^T_t}\right]$: those are standard MAB policies. Take one of them, and adapt it so that it also has a large $P(\tau(B, \pi^T) = \infty)$ while maintaining a large $\E\left[\sum_{t=1}^T X_t^{\pi^T_t}\right]$.
    \item start by finding policies $\pi^T$ which have a large $P(\tau(B, \pi^T) = \infty)$. Take one such policy and adapt it so that it also has a large $\E\left[\sum_{t=1}^T X_t^{\pi^T_t}\right]$ while maintaining a large $P(\tau(B, \pi^T) = \infty)$.
\end{itemize}

\noindent The former type of approach consists in adapting existing MAB strategies to a new variant of the MAB. This type of approach, which uses existing results that we already know well to tackle new problems, is mathematically valid and reasonable, and for that reason, is rather common in the existing literature. The latter type of approach is much more difficult, however, from an innovation perspective, it is more appealing. But there is one more reason why this approach may be preferable. As explained earlier in this subsection, if the policy $\pi^T$ has managed to survive after, e.g., $\sqrt{T}$ rounds, then the risk of ruin almost disappears. Therefore, a good policy $\pi^T$ will:
\begin{itemize}
    \item first maximize $P(\tau(B, \pi^T) = \infty)$ until the risk of ruin disappears;
    \item then, maximize $\E\left[\sum_{t=1}^T X_t^{\pi^T_t}\right]$.
\end{itemize}

\noindent This seems to suggest that the second type of approach is more natural for the S-MAB, and this is the approach taken in this paper.

\subsection{Reward Models}

In this paper, we will mostly focus on multinomial arm distributions of support $\{-1, 0, 1\}$ (simply referred to as multinomial arm distributions in the rest of this paper). This subsection discusses why this setting is of particular interest.

\subsubsection{Why it is reasonable to study integer rewards}

Let $\pi^T$ be a policy, and assume that $\tau(B, \pi^T)<\infty$. We are going to quickly study $\sum_{t=1}^{\tau(B, \pi^T)} X_t^{\pi^T_t}$. On the one hand, if the arm distributions are multinomial, then the rewards are in $\{-1, 0, 1\}$ and 
\begin{align*}
    \tau(B, \pi^T) &\triangleq \inf\left\{t\geq 1: \sum_{s=1}^t X_s^{\pi^T_s} \leq -B \right\} \\
    &= \inf\left\{t\geq 1: \sum_{s=1}^t X_s^{\pi^T_s} \leq - \lceil B\rceil \right\} \\
    &= \inf\left\{t\geq 1: \sum_{s=1}^t X_s^{\pi^T_s} = - \lceil B\rceil \right\}.
\end{align*}

\noindent Hence, whatever the policy $\pi^T$, if $\tau(B, \pi^T)<\infty$, then
\begin{equation*}
    \sum_{t=1}^{\tau(B, \pi^T)} X_t^{\pi^T_t} = -\lceil B\rceil.
\end{equation*}

\noindent On the other hand, in the general case of arm distributions bounded in $[-1, 1]$, if $\tau(B, \pi^T)<\infty$, we cannot deduce anything further than \eqref{reward_at_ruin_general_eq}, i.e.,
\begin{equation*}
    \sum_{t=1}^{\tau(B, \pi^T)} X_t^{\pi^T_t} \in (-B-1, -B]
\end{equation*}
without any further information on the $\pi^T$ or the arm distributions. In fact, $\sum_{t=1}^{\tau(B, \pi^T)} X_t^{\pi^T_t}$ is in general not deterministic and depends on the arm distributions, as well as on the (possible) randomization on $\pi^T$. This hints at the major technicality of the general case compared to the multinomial case. Without proof techniques involving the choice of the policy $\pi^T$, the value of $\sum_{t=1}^{\tau(B, \pi^T)} X_t^{\pi^T_t}$ cannot be controlled in expectation or in probability. As a result, it seems difficult to achieve a tight lower bound on the probability of ruin {\it of all policies} without first ``finding'' a ``best'' policy $\pi^T$.

Now, should you guess a ``best'' policy $\pi^T$, this policy would be extremely sensitive to the arm distributions. Indeed, it might easily get ``trapped'' such that if $\tau(B, \pi^T)<\infty$,
\begin{equation*}
    \sum_{t=1}^{\tau(B, \pi^T)} X_t^{\pi^T_t} \in [-B-\epsilon, -B]
\end{equation*}
for some small $0<\epsilon\ll 1$, unless this policy tries to learn the arm distributions support (this is difficult when the budget $B$ is small, as in the case of this article, and it is left for future work). Overall, we so far expect a probability of ruin as a function of the budget: $\text{Bound}(B)$.

On the other hand, it is intuitive that a good policy $\pi^T$ should use its budget as much as possible, and hopefully, at the time of ruin, $\sum_{t=1}^{\tau(B, \pi^T)} X_t^{\pi^T_t}$ should be rather close to $-B-1$. Therefore, we expect a lower bound on the probability of ruin of all policies to be a function of $B+1$, potentially $\text{Bound}(B+1)$, which does not quite match the above. But in our case, a suboptimal constant factor in the probability of ruin $P(\tau(B, \pi^T)<\infty)$ creates a difference of $\Omega(T)$ in the cumulative reward, as it can be seen on~\eqref{approx_decompisition_reward}.

For the above reasons, we will only consider the case of multinomial arm distributions in Sections~\ref{survival_proba_study} to \ref{experiment_section}, and generalize our results to the case of rewards in $[-1, 1]$ in Section~\ref{generalization_section}.

\subsubsection{Why Bernoulli rewards are not sufficient}\label{bernoulli_vs_multinomial_subsection}

As this paper is the first one to tackle the S-MAB, it may seem reasonable to start with the easiest setting of Bernoulli distributions of support $\{-1, 1\}$, referred to as Bernoulli distributions in the rest of this paper. In this subsection, we explain why we consider the more complex case of multinomial arm distributions, with $0$ included in the support. As explained in Section~\ref{strategy_subsection}, the optimal policy $\pi^T$ should optimize two objectives at the same time:
\begin{itemize}
    \item maximize the probability of survival;
    \item maximize the cumulative reward $\E\left[\sum_{t=1}^T X_t^{\pi^T_t}\right]$.
\end{itemize}

\noindent While the S-MAB requires to be more conservative than the traditional MAB, in the case of Bernoulli arm distributions, this boils down to finding and pulling as often as possible the ``best arm'', defined here as the (Bernoulli) arm with the largest parameter $P(X=1)$. This intuitively makes sense, because at each round $t\geq 1$, the best arm maximizes both of the objectives above.

On the other hand, in the case of multinomial arm distributions, those two objectives may be maximized by two different arms:
\begin{itemize}
    \item the arm which has the largest probability of survival $k^P \triangleq \argmax_{k\in [K]} \frac{P_{F_k}(X = +1)}{P_{F_k}(X = -1)}$ (see Lemma~\ref{constant_survival});
    \item and the arm which has the largest expectation $k^E \triangleq \argmax_{k\in [K]} \big\{P_{F_k}(X = +1) - P_{F_k}(X = -1)\big\}$.
\end{itemize}

\noindent Therefore, a good policy $\pi^T$ should:
\begin{itemize}
    \item first find arm $k^P$ with pure exploitation, and pull it many times;
    \item then, find arm $k^E$ with a conservative exploration-exploitation, and pull it many times.
\end{itemize}

\noindent In the case of multinomial arms, the learning paradigm of the S-MAB is hence different from the one of the standard MAB, and could be described as an exploitation-exploration-exploitation dilemma, as a more complex version of the standard exploration-exploitation dilemma of the standard MAB. Such a paradigm was the initial motivation behind this work. For that reason, we chose the model of multinomial distributions, which brings considerable interest while much more difficulty than the Bernoulli model.

\subsection{Structure of the Paper}

We start by studying the case of multinomial arm distributions of support $\{-1, 0, 1\}$ (referred to as multinomial distributions in the sequel) in Sections~\ref{survival_proba_study} to \ref{experiment_section}, and then we will extend some of our results to the general case of arm distributions bounded in $[-1, 1]$ in Section~\ref{generalization_section}. The outline of the paper is as follows:
\begin{itemize}
    \item in Section~\ref{survival_proba_study}, we study the probability of ruin and derive the first main result of this paper: a tight non-asymptotic lower bound (in the sense of Pareto) on the probability of ruin in Theorem~\ref{non_asymptotic_lower_bound_theorem}. 
    \item in Section~\ref{exploit_framework_study}, we introduce EXPLOIT, a framework (or set) of policies whose probability of ruin matches the lower bound of Section~\ref{survival_proba_study}. We further show that those policies cannot be regret-wise Pareto-optimal, and hence the need to adapt them to achieve our desired objective.
    \item in Section~\ref{survival_regret_study}, we introduce the policy EXPLOIT-UCB-DOUBLE, which performs a doubling trick over an EXPLOIT policy. We show the second main result of this paper in Theorem~\ref{teaser_theorem}: EXPLOIT-UCB-DOUBLE is regret-wise Pareto-optimal. We corroborate this theoretical result with experiments showing the experimental performance of EXPLOIT-UCB-DOUBLE. This section provides an answer to an open problem from COLT 2019 (\citealp{perotto}). 
    \item in Section~\ref{generalization_section}, we generalize the results of Sections~\ref{survival_proba_study}--\ref{survival_regret_study} to the general case of arm distributions bounded in $[-1, 1]$. In particular, we generalize the result of Theorem~\ref{non_asymptotic_lower_bound_theorem} on the probability of ruin, which is the third main result of this paper, and further relate it to the probability of ruin of i.i.d. stochastic processes, which is a result of independent interest. Finally, we discuss the challenges of extending Theorem~\ref{teaser_theorem} to that setting.
\end{itemize}

\section{Study of the Probability of Ruin}\label{survival_proba_study}

In this section, we study the probability of ruin of anytime policies in the case of multinomial arm distributions (of support $\{-1, 0, 1\}$). We first derive one of the main results of this paper: a tight lower bound on the probability of ruin (Theorem~\ref{non_asymptotic_lower_bound_theorem}), which will be generalized to distributions bounded in $[-1, 1]$ in Proposition~\ref{non_asymptotic_lower_bound_corollary} of Section~\ref{generalization_section}. We further relate this bound to the probability of ruin of i.i.d.~random walks on $\{-1, 0, 1\}$ in Lemma~\ref{proba_lemma}, which is a result of independent interest.

\subsection{Trivial Case and Assumption}

Consider the following example.
\begin{example}
    Assume that there exists an arm $k\in [K]$ whose distribution $F_k$ is such that $P_{X\sim F_k}(X\geq 0) = 1$. Then, if the initial budget $B$ satisfies $B> K-1$, the policy $\pi$ which chooses the arm with the highest empirical average of rewards, i.e., defined by
    \begin{equation*}
        \forall t\geq 1, \ \pi_t \triangleq \argmax_{k\in [K]} \frac{\sum_{s=1}^{t-1} X_s^{\pi_t} \1_{\pi_s = k}}{\sum_{s=1}^{t-1} \1_{\pi_s = k}},
    \end{equation*}
    has no risk of ruin: $P(\tau(B, \pi) <\infty) = 0$.
\end{example}

\noindent The above case is simple, because the distribution $F_k$ only yields non-negative rewards. In this case, the trivial bound on the probability of ruin
\begin{equation*}
    P(\tau(B, \pi)<\infty) \geq 0
\end{equation*}
is tight, and we eliminate this case by assuming that there is no positive or zero arm distribution, as defined below.
\begin{definition}\label{positive_zero_definition}
A distribution $F$ is called a zero distribution (resp. a positive distribution) if $P_{X\sim F}(X = 0) = 1$ (resp. if $P_{X\sim F}(X\geq 0) = 1$ and $P_{X\sim F}(X> 0) >0$). We say that an arm $k$ is a zero arm (resp. a positive arm) if its distribution $F_k$ is zero (resp. positive).
\end{definition}

\noindent In this section, we make the following assumption.
\begin{assumption}
    There is no positive or zero arm.
\end{assumption} 

\noindent Please note that the policies introduced in Sections~\ref{exploit_framework_study} to \ref{survival_regret_study} will achieve $P(\tau(B, \pi)<\infty) = 0$ in the case where there is a positive or zero arm.

\subsection{Main Results on the Probability of Ruin}\label{general_results_proba_ruin_study}

Let $\mathcal{F}_{\{-1,0,1\}}$ be the set of multinomial distributions of support $\{-1, 0, 1\}$ which are not positive or zero (see Definition~\ref{positive_zero_definition}). Similarly to the cumulative reward, for any $F = (F_1, \dots, F_K)\in \mathcal{F}_{\{-1,0,1\}}^K$ and any policies $\pi, \pi'$, we define the {\it relative risk of ruin} of $\pi$ with respect to $\pi'$ as
\begin{equation*}
    P_{\mathrm{ruin}}(\pi\|\pi') := P_F(\tau(B, \pi)<\infty) - P_F(\tau(B, \pi')<\infty),
\end{equation*}
where the dependency on $F$ is omitted in the notation. Please note that, since we study this problem for fixed $B$ small with regards to $T$, there is no limit in the definition of $P_{\mathrm{ruin}}(\pi\|\pi')$. Yet, similarly to Proposition~\ref{linear_regret}, we can prove that no policy $\pi$ achieves $\sup_{\pi'} \sup_{F} P_{\mathrm{ruin}}(\pi\|\pi') \leq 0$, and for that reason, we focus on a policy which is Pareto-optimal in the sense of the probability of ruin, which is formalized as follows.
\begin{definition}\label{ruin_pareto_optimal_definition}
A policy $\pi$ is said to be ruin-wise Pareto-optimal if, for any policy $\pi'$,
\begin{equation*}
    \sup_F P_{\mathrm{ruin}}(\pi\|\pi') >0 \implies \inf_F P_{\mathrm{ruin}}(\pi\|\pi') <0.
\end{equation*}
\end{definition}

\noindent Before stating the main result of this section, we need to define, for any arm distributions $F = (F_1, \dots, F_K)\in \mathcal{F}_{\{-1,0,1\}}^K$,
\begin{equation}
    \gamma(F_k) := \inf_{Q: \E_{X\sim Q}[X]<0}\frac{\KL(Q \Vert F_k)}{\E_{X\sim Q}[-X]} \geq 0.\label{gamma_definition}
\end{equation}

\noindent The main result of this section is a Pareto-type lower bound on the probability of ruin.
\begin{theorem}\label{non_asymptotic_lower_bound_theorem}
Let $(\alpha_k)_{k\in [K]}$ be such that for any $k, \ \alpha_k > 0$ and $\sum_{k=1}^K \alpha_k = 1$. For any policy $\pi$,
\begin{multline}
    \inf_{F\in \mathcal{F}_{\{-1,0,1\}}^K} \left\{P_{F}(\tau(B, \pi)<\infty) - \exp\left(- B\sum_{k=1}^K \alpha_k \gamma(F_k)\right)\right\} < 0 \\
    \implies \sup_{F\in \mathcal{F}_{\{-1,0,1\}}^K} \left\{P_{F}(\tau(B, \pi)<\infty) - \exp\left(- B\sum_{k=1}^K \alpha_k \gamma(F_k)\right)\right\} >0. \label{lower_multinomial}
\end{multline}
\end{theorem}

\noindent The main ingredient of the proof of Theorem~\ref{non_asymptotic_lower_bound_theorem} is given in Section \ref{lower_bound_proba_study}, and the rest of the proof is given in Appendix~\ref{lower_bound_proba_long_proof}. 

Please note that this theorem gives a lower bound on the probability of ruin. A weaker version of Theorem~\ref{non_asymptotic_lower_bound_theorem} is that there exist some arm distributions $F\in \mathcal{F}_{\{-1,0,1\}}^K$ such that
\begin{equation}
    P_{F}(\tau(B, \pi)<\infty) \geq \exp\left(- B\sum_{k=1}^K \alpha_k \gamma(F_k)\right). \label{weaker_easy_theorem_bound}
\end{equation}

\noindent But Theorem~\ref{non_asymptotic_lower_bound_theorem} is a little stronger than that. Actually, it states that there are two cases:
\begin{itemize}
    \item either the lower bound \eqref{weaker_easy_theorem_bound} holds for {\it all} arm distributions $F\in \mathcal{F}_{\{-1,0,1\}}^K$;
    \item or it holds with strict inequality for some $F\in \mathcal{F}_{\{-1,0,1\}}^K$.
\end{itemize}

\noindent We conclude this subsection by a lemma providing an insightful interpretation of the lower bound~\eqref{weaker_easy_theorem_bound}.
\begin{lemma}\label{proba_lemma}
For any distributions $F = (F_1, \dots, F_K) \in \mathcal{F}_{\{-1, 0, 1\}}^K$ and any arm $k\in [K]$,
\begin{equation*}
    \frac{1}{B} \log P_{F_k}(\tau(B, k)<\infty) = -\gamma(F_k).
\end{equation*}
\end{lemma}

\noindent Let $(\alpha_k)_{k\in [K]}$ be such that for any $k\in [K], \alpha_k B\in \mathbbm{N}$. An application of Lemma~\ref{proba_lemma} to the budget $\alpha_k B$ for any $k$ yields
\begin{equation}
    \exp\left(- B\sum_{k=1}^K \alpha_k \gamma(F_k)\right) = \prod_{k=1}^K P_{F_k}(\tau(\alpha_k B, k)<\infty).\label{link_part_exploit}
\end{equation}

\noindent This gives another expression of the lower bound in \eqref{weaker_easy_theorem_bound}, which is easier to interpret. This lower bound is the product on the arms $k\in [K]$ of the probabilities of ruin of the constant policy $\pi_t = k$ with budget $\alpha_k B$. Each of the factors in~\eqref{link_part_exploit} can further be computed explicitly, using Lemma~\ref{constant_survival}:
\begin{equation*}
    \exp\left(- B\sum_{k=1}^K \alpha_k \gamma(F_k)\right) = \prod_{k=1}^K \min\left(1, \left(\frac{P_{X\sim F_k}(X=-1)}{P_{X\sim F_k}(X=1)}\right)^{\alpha_k B}\right).
\end{equation*}

\noindent Whereas the statements of Lemma~\ref{proba_lemma} use the KL divergence through the definition of $\gamma(F_k)$ in \eqref{gamma_definition}, the probability of ruin of a stochastic process is usually analyzed using the moment-generating function, which is also found in the proof of this lemma given in Appendix~\ref{proba_lemma_appendix}. The relation between them is discussed in Lemma~\ref{lem_kl_lambda} in Appendix~\ref{append_kl}, and we interchangeably use both representations.

\subsection{Sketch and Main Ingredient of the Proof of Theorem~\ref{non_asymptotic_lower_bound_theorem}}\label{lower_bound_proba_study}

The proof of the non-asymptotic bound of Theorem~\ref{non_asymptotic_lower_bound_theorem} is given in Appendix~\ref{lower_bound_proba_long_proof}. This proof consists of (i) derivation of an asymptotic bound given in Lemma~\ref{fundamental_lemma_proba_ruin} below, and (ii) turning it into a non-asymptotic bound by using a sub-additivity argument on the probability of ruin. Please note that the proof of this lemma is conducted in the generality of distributions bounded in $[-1, 1]$ (not necessarily multinomial).
\begin{lemma}\label{fundamental_lemma_proba_ruin}
Fix an arbitrary policy $\pi$ and distributions $(Q_1, \dots, Q_K)$ such that $\E_{X\sim Q_k}[X]<0$ for all $k\in[K]$. Then, there exists a probability vector $\beta(Q) = (\beta_1(Q), \dots, \beta_K(Q))$ such that for any distributions $(F_1, \dots, F_K)$,
\begin{equation*}
    \liminf_{B\rightarrow +\infty} \frac{1}{B} \log P_{(F_{1}, \dots, F_{K})}\left(\tau(B, \pi)<\frac{3B}{\Delta_Q}\right) \geq - \sum_{k=1}^K \beta_k(Q)\frac{\KL(Q_k \Vert F_{k})}{\E_{X\sim Q_k}[-X]},
\end{equation*}
where $\Delta_Q =\min_{i\in[K]}\E_{X\sim Q_i}[-X]>0$.
\end{lemma}

\begin{proof}
Let $Q = (Q_1, \dots, Q_K)$ be a vector of distributions such that
$\E_{X\sim Q_k}[X]<0$ for all $k\in[K]$, and let $\Delta_Q := \min_{i\in[K]}\E_{X\sim Q_i}[-X]>0$. We denote by $N_k(\tau)$ the number of pulls of arm $k$ until $\tau(B, \pi)$, and by $n_k$ its realization. Denoting by $Y_k^n$ the reward of the $n$-th pull of arm $k$, let
\begin{equation*}
    \mathcal{H}_{\tau}=\left(\left(Y_1^1, Y_1^2, \dots, Y_1^{N_1(\tau)}\right), \left(Y_2^1, Y_2^2, \dots, Y_2^{N_2(\tau)}\right), \dots, \left(Y_K^1, Y_K^2, \dots, Y_K^{N_K(\tau)}\right)\right),
\end{equation*}
and $h_t$ be its realization. Please note that for any realization $h_t$,
$\left|B+\sum_{k\in[K]}\sum_{m\in[n_k]}y_k^m\right|\leq 1$. We further denote by $T(Q)$ the set of ``typical'' realizations $h_t$ satisfying
\begin{equation}
    \left\{
    \begin{array}{lll}
        &\left|\sum_{k=1}^K\left(n_k \KL(Q_k\Vert F_k)-\sum_{m=1}^{n_k} \log \frac{d Q_k}{d F_k}(y_k^m)\right) \right| \leq \frac{t}{B^{\frac{1}{4}}}, \\
        &\left|\sum_{k=1}^K\left(n_k \E_{X\sim Q_k}[X]-\sum_{m=1}^{n_k} y_k^m\right)\right|\leq \frac{t\Delta_Q}{B^{\frac{1}{4}}}, \\
        &\sum_{k=1}^K \sum_{m=1}^{n_k} y_k^m \leq -\frac{t\Delta_Q}{2}.
    \end{array}
\right. \label{typicality_hypothesis}
\end{equation}

\noindent Such realizations are ``typical'' under $Q$ in the sense that $\lim_{B\rightarrow +\infty}Q(\mathcal{H}_{\tau}\in T(Q))= 1$ (shown by, e.g., Hoeffding's inequality, see Appendix~\ref{justification_lim_q} for details). We can see from~\eqref{typicality_hypothesis} that any typical $h_t$ satisfies
\begin{equation}
    t\leq \frac{3B}{\Delta_Q} \ \text{ and } \ \left|\sum_{k=1}^K \frac{n_k}{B} \E_{X\sim Q_k}[-X] -1\right| \leq \frac{4}{B^{\frac{1}{4}}}. \label{r_expect_normalized}
\end{equation}

\noindent In particular, denoting $r(h_t) := \frac{(n_1, \dots, n_K)}{B}$, \eqref{r_expect_normalized} implies that $r(h_t)$ can take at most $O(\text{poly}(B))$ values, and hence there exists $\tilde{r}$ such that 
\begin{equation}
    \lim_{B\to +\infty}\frac{1}{B} \log Q\left( r(\mathcal{H}_{\tau}) = \tilde{r} | \mathcal{H}_{\tau} \in T(Q)\right) = 0. \label{lim_q_r_tilde}
\end{equation}

\noindent By performing a change of distribution and using \eqref{typicality_hypothesis}, we can bound
\allowdisplaybreaks
\begin{align*}
    &P_{(F_1, \dots, F_K)}\left(\tau(B, \pi)<\frac{3B}{\Delta_Q}\right) \\
    &\geq P_{(F_1, \dots, F_K)}\left(\mathcal{H}_{\tau} \in T(Q), r(\mathcal{H}_{\tau}) = \tilde{r}\right) \\
    &= \sum_{\substack{h_t \in T(Q): \\
    r(h_t) = \tilde{r}}} Q(h_t)\exp\left(- \sum_{k=1}^K \sum_{m=1}^{n_k} \log \frac{d Q_k}{d F_k}(y_k^m)\right) \\
    &\geq \sum_{\substack{h_t \in T(Q): \\
    r(h_t) = \tilde{r}}} Q(h_t) \exp\left(- \sum_{k=1}^K n_k \KL(Q_k \Vert F_k) - \frac{t}{B^{\frac{1}{4}}}\right) \\
    &\geq \sum_{\substack{h_t \in T(Q): \\
    r(h_t) = \tilde{r}}} Q(h_t) \exp\left\{-B \left(\sum_{k=1}^K \tilde{r}_k \KL(Q_k \Vert F_k) + \frac{3}{\Delta_Q B^{\frac{1}{4}}}\right)\right\} \\
    &= \exp\left\{-B \left(\sum_{k=1}^K \tilde{r}_k \E_{X\sim Q_k}[-X] \frac{\KL(Q_k \Vert F_k)}{\E_{X\sim Q_k}[-X]} + \frac{3}{\Delta_Q B^{\frac{1}{4}}}\right)\right\} Q\left(\mathcal{H}_{\tau} \in T(Q), r(\mathcal{H}_{\tau}) = \tilde{r}\right).
\end{align*}

\noindent For any $k\in [K]$, we introduce the normalized version of $\tilde{r}_k \E_{X\sim Q_k}[-X]$, which we denote by 
\begin{equation*}
    \beta_k(Q) := \frac{\tilde{r}_k \E_{X\sim Q_k}[-X]}{\sum_{j=1}^K \tilde{r}_j \E_{X\sim Q_j}[-X]}.
\end{equation*}

\noindent Eq.~\eqref{r_expect_normalized} implies that $\left|\sum_{j=1}^K \tilde{r}_j \E_{X\sim Q_j}[-X] - 1\right| \leq \frac{4}{B^{\frac{1}{4}}}$, and hence
{\allowdisplaybreaks[4]%
\begin{align*}
    &\frac{1}{B} \log P_{(F_1, \dots, F_K)}\left(\tau(B, \pi)<\frac{3B}{\Delta_Q}\right) \\
    &\geq -\left(1+\frac{4}{B^{\frac{1}{4}}}\right)\sum_{k=1}^K \beta_k(Q) \frac{\KL(Q_k \Vert F_k)}{\E_{X\sim Q_k}[-X]} - \frac{3}{\Delta_Q B^{\frac{1}{4}}} + \frac{1}{B} \log Q\left(\mathcal{H}_{\tau} \in T(Q), r(\mathcal{H}_{\tau}) = \tilde{r}\right) \\
    &\xrightarrow{B\to +\infty} -\sum_{k=1}^K \beta_k(Q) \frac{\KL(Q_k \Vert F_k)}{\E_{X\sim Q_k}[-X]},
\end{align*}}
by \eqref{lim_q_r_tilde}, concluding the proof of Lemma~\ref{fundamental_lemma_proba_ruin}.
\end{proof}

\section{The EXPLOIT framework}\label{exploit_framework_study}

In this section, we introduce a framework of anytime policies, called EXPLOIT, which achieve the lower bound on the probability of ruin given in Theorem~\ref{non_asymptotic_lower_bound_theorem}. We further study the regret of all such policies.

In Sections~\ref{exploit_framework_study} and \ref{survival_regret_study}, our study is conducted in the case of multinomial arm distributions, and therefore, it holds that
\begin{equation*}
    P(\tau(B, \pi)\leq T) = P(\tau(\lceil B\rceil, \pi)\leq T)
\end{equation*}
for any policy $\pi$. As a result, in these sections, we assume w.l.o.g. that $B\in \mathbbm{N}$.

\subsection{Definition of the EXPLOIT Framework}\label{exploit_introduction_study}

Let $B_1, \dots, B_K$ be arbitrary positive integers such that $B_1 + \dots + B_K = B$. Applying~\eqref{link_part_exploit} to $\alpha_k \triangleq \frac{B_k}{B}$ for any $k\in [K]$, the lower bound of Theorem~\ref{non_asymptotic_lower_bound_theorem} implies that 
\begin{equation*}
    P_F(\tau(B, \pi)<\infty) \geq \prod_{k=1}^K P_{F_k}(\tau(B_k, k)<\infty)
\end{equation*}
for some arm distributions $F\in \mathcal{F}_{\{-1, 0, 1\}}^K$. This right hand-side of this inequality corresponds to the probability of ruin of any policy $\pi$ which allocates budget $B_1$ to arm $1$, $B_2$ to arm $2$ and so on, and plays $\pi_t = k$ only if arm $k$ has not exceeded its budget $B_k$. We say that such policies belong to the EXPLOIT framework, which is formalized below.
\begin{definition}\label{exploit_definition}
Given some positive integers $B_1, \dots, B_K$ such that $B_1 + \dots + B_K = B$, we say that a policy $\pi = (\pi_t)_{t\geq 1}$ belongs to the framework EXPLOIT$(B_1, \dots, B_K)$ if, at any round $t\geq 1$,
\begin{equation*}
    \pi_t \in \left\{k\in [K]: B_k + \sum_{s=1}^{t-1} X_s^{\pi_s} \1_{\pi_s = k} > 0\right\}.
\end{equation*}
\end{definition}

\begin{remark}
    The initial choice of the budget shares $B_1, \dots, B_K$ may depend on some possible prior information we may have on the arms. Without any prior information on the arms, we have no reason to choose $B_1 \gg B_2$ e.g., and we will choose the budget shares as close as possible. If $B = n_K K + b$, with $0\leq b<K$ is the Euclidian division of $B$ by $K$, a possibility is to choose $B_1 = \dots = B_b = n_K+1$ and $B_{b+1} = \dots B_K = n_K$.
\end{remark}

\noindent All the policies in EXPLOIT$(B_1, \dots, B_K)$ have the same probability of ruin, given in the next proposition.
\begin{proposition}\label{exploit_proba_ruin_proposition}
    Given some positive integers $B_1, \dots, B_K$ such that $B_1 + \dots + B_K = B$, all the policies in EXPLOIT$(B_1, \dots, B_K)$ achieve the same probability of ruin, given by
    \begin{equation*}
        p^{\mathrm{EX}}(B_1, \dots, B_K) \triangleq \prod_{k=1}^K P_{F_k}(\tau(B_k, k)<\infty) = \exp\left(- \sum_{k=1}^K B_k \gamma(F_k)\right).
    \end{equation*}
\end{proposition}

\noindent Importantly, the probability of ruin of the policies in EXPLOIT match the lower bound of Theorem~\ref{non_asymptotic_lower_bound_theorem}, and therefore, all the policies in EXPLOIT are ruin-wise Pareto-optimal. This is a major strength of the EXPLOIT framework: given $B_1, \dots, B_K$, you can choose any policy in EXPLOIT$(B_1, \dots, B_K)$ to try to maximize the expected cumulative reward without sacrificing the probability of ruin. 

\begin{remark}
    When $B$ is a multiple of $K$, the policy $\pi$ which selects arm $\pi_t \in \argmax_{k\in [K]}\Big\{\sum_{s=1}^{t-1}$ $X_s^{\pi_s}\1_{\pi_s = k}\Big\}$ with the highest cumulative reward belongs to EXPLOIT$\left(\frac{B}{K}, \dots, \frac{B}{K}\right)$. While it is intuitive that such a policy is very conservative, it was {\it a priori} not obvious that many policies (all the policies in EXPLOIT$\left(\frac{B}{K}, \dots, \frac{B}{K}\right)$) achieve the same probability of ruin. Actually, this result becomes false in the general case of rewards in $[-1, 1]$ (not necessarily integers) studied in Section~\ref{generalization_section}.
\end{remark}

%\noindent We call this framework EXPLOIT because when $B$ is a multiple of $K$, a policy $\pi$ such that $\pi_t \in \argmax_{k\in [K]}\left\{\sum_{s=1}^{t-1} X_s^{\pi_s}\1_{\pi_s = k}\right\}$ which selects an arm with the highest cumulative reward belongs to EXPLOIT$\left(\frac{B}{K}, \dots, \frac{B}{K}\right)$. While it is intuitive that such a policy is extremely conservative, it was {\it a priori} not so obvious that many policies (namely all the policies in EXPLOIT) also achieve the same probability of ruin. This is a major strength of the framework EXPLOIT: given $B_1, \dots, B_K$, you have the freedom to choose any policy in EXPLOIT$(B_1, \dots, B_K)$ (e.g., to try to maximize the expected cumulative reward) without sacrificing the probability of ruin. More on that will be given in the rest of this section. For now, we formalize the result on the probability of ruin of policies in EXPLOIT$(B_1, \dots, B_K)$.

\begin{convention}
    From now on, given $B = n_K K + b$ with $0\leq b< K$, we consider the ``more symmetric'' case $B_1 = \dots = B_b = n_K+1$ and $B_{b+1} = \dots = B_K = n_K$. We use the shortcut notation EXPLOIT instead of EXPLOIT$(B_1, \dots, B_K)$ in that specific instance and denote $p^{\mathrm{EX}} \triangleq p^{\mathrm{EX}}(B_1, \dots, B_K)$. 
\end{convention}

\subsection{Expected Cumulative Reward of Policies in EXPLOIT}\label{exploit_cumulative_reward_study}

By nature, EXPLOIT policies are very conservative. Precisely, they allow a budget of $B_k$ for the exploration of each arm $k\in [K]$. In the previous subsection, we showed that, thanks to this limited exploration, they are ruin-wise Pareto-optimal, i.e., they achieve a small probability of ruin (in the sense of Pareto-optimality). In this subsection, we show that the downside of this limited exploration is that the expected cumulative reward of EXPLOIT policies is fairly low, upper-bounded as shown in the following proposition, whose proof is in Appendix~\ref{exploit_reward_upper_bound_proposition_proof}. 
\begin{proposition}\label{exploit_reward_upper_bound_proposition}
Assume that the budget $B$ is a multiple of $K$. W.l.o.g., assume that $\mu_1\geq \dots \geq \mu_{K}$. Then, for any policy $\pi$ in EXPLOIT,
\begin{equation}
    \E\left[\sum_{t=1}^T X_t^{\pi_t}\1_{\tau(B, \pi)\geq t-1}\right] \leq \left(1 - p^{\mathrm{EX}}\right) \sum_{k=1}^{K} w_k \mu_k \times T + o(T) \eqineq \left(1 - p^{\mathrm{EX}}\right) \max_{k\in[K]}\mu_k \times T + o(T), \label{exploit_reward_upper_bound}
\end{equation}
where for any $k\in [K], w_k = \frac{P\left(\tau\left(\frac{B}{K}, k\right)=\infty\right)\prod_{j=1}^{k-1} P\left(\tau\left(\frac{B}{K}, j\right)<\infty\right)}{1-p^{\mathrm{EX}}}$. Besides, when two arms have positive and different expectations, $(*)$ is a strict inequality.
\end{proposition}

\noindent As we will see in Section~\ref{survival_regret_study}, there exists a non-anytime policy whose expected cumulative reward is equal to the right hand-side of~\eqref{exploit_reward_upper_bound}. This implies that no EXPLOIT policy is regret-wise Pareto-optimal (although they are all ruin-wise Pareto-optimal).

As explained before, a policy in EXPLOIT will allow a budget share $B_k$ for the exploration of each arm $k\in [K]$. Therefore, it may stop pulling arm $k$ after only $B_k$ pulls even without encountering ruin before the horizon $T$. This lack of exploration of arm $k$ penalizes the cumulative reward of arm $k$ and is reflected in the coefficient $(1-p^{\mathrm{EX}})w_k$ in the middle term of the bound of Proposition~\ref{exploit_reward_upper_bound_proposition}.

\subsection{Bandit Algorithms in the EXPLOIT framework}\label{exploit_bandit_study}

We know that all the policies in EXPLOIT are ruin-wise Pareto-optimal and achieve the same probability of ruin (this is the result of Proposition~\ref{exploit_proba_ruin_proposition}). In this subsection, we exhibit an EXPLOIT policy which achieves the highest possible cumulative reward within EXPLOIT, given as the middle bound in~\eqref{exploit_reward_upper_bound}. Though such a policy would not be regret-wise Pareto-optimal as suggested from Proposition~\ref{exploit_reward_upper_bound_proposition}, it will serve as a basis for the construction of such an optimal policy. For any arm $k\in [K]$ and any round $t$, we introduce $N_k(t) := \sum_{s=1}^t \1_{\pi_s = k} $ as the number of pulls of arm $k$ until round $t$, and $\hat{X}^k_t := \frac{1}{N_k(t)}\sum_{s=1}^t \1_{\pi_s = k} X_s^{\pi_s}$ as the empirical mean of arm $k$ at round $t$. 

%\LinesNumbered
\SetAlgoVlined
\begin{algorithm}[t]
\vspace{0mm}
\SetNoFillComment
\DontPrintSemicolon
\For{$t=1,\ldots, T$}{
    Set $\mathcal{A}_t := \left\{k\in [K]: \sum_{s=1}^{t-1} X_s^{\pi_s} \1_{\pi_s = k} \geq -\frac{B}{K}+1\right\}$. \\
    \eIf{$\mathcal{A}_t \ne \emptyset$}{
        Pull arm $\argmax_{k\in \mathcal{A}_t} \hat{X}_{t-1}^k + \sqrt{\frac{6 \log (t-1)}{N_k(t-1)}}$.}{
        Pull arm $\argmax_{k\in [K]} \hat{X}_{t-1}^k$.}
 }
 \caption{EXPLOIT-UCB$(B)$}
 \label{survival_bandit_algo}
\end{algorithm}

We start from the classic bandit algorithm UCB (Upper Confidence Bound, see \citealp{auer}) and make it ``fit'' in the EXPLOIT framework. This defines EXPLOIT-UCB$(B)$ in Algorithm~\ref{survival_bandit_algo} as the policy which, at each round $t\geq 1$, performs UCB among the arms whose cumulative reward is larger than $-\frac{B}{K}$. As EXPLOIT-UCB$(B)$ is in EXPLOIT, it naturally achieves the optimal bound on the probability of ruin. In addition, it also achieves the middle bound in \eqref{exploit_reward_upper_bound} on the reward, which is asymptotically optimal among EXPLOIT policies, as shown in the next proposition, whose proof is in Appendix~\ref{exploit_bandit_reward_optimal_proof}.
\begin{proposition}\label{exploit_bandit_reward_optimal}
Under the hypotheses of Proposition~\ref{exploit_reward_upper_bound_proposition}, the expected cumulative reward of EXPLOIT-UCB$(B)$ satisfies
\begin{equation*}
    \E\left[\sum_{t=1}^T X_t^{\pi_t}\1_{\tau(B, \pi)\geq t-1}\right] = \left(1 - p^{\mathrm{EX}}\right) \sum_{k=1}^{K} w_k \mu_k \times T + o(T),
\end{equation*}
where for any $k\in [K], w_k = \frac{P\left(\tau\left(\frac{B}{K}, k\right)=\infty\right)\prod_{j=1}^{k-1} P\left(\tau\left(\frac{B}{K}, j\right)<\infty\right)}{1-p^{\mathrm{EX}}}$.
\end{proposition}

\noindent More than the bound itself, the above result states that a standard MAB algorithm ``made to fit in EXPLOIT'' achieves the best possible reward within EXPLOIT. Here is the intuition behind it. All EXPLOIT policies have the same risk of ruin, and when there is ruin, all policies receive the total reward $-B$. Therefore, a good EXPLOIT policy can only make a difference in the case when there is no ruin. In that case, it should achieve a high cumulative reward, i.e., behave closely to a good standard MAB policy like UCB.

\begin{remark}
    The choice of the constant $6$ in the square root is different from the original UCB and is only made for simplicity of the proof. 
\end{remark}

\section{A Regret-wise Pareto-optimal Policy}\label{survival_regret_study}

In this section, we make a slight modification on the policy EXPLOIT-UCB$(B)$ (see Algorithm~\ref{survival_bandit_algo}) so that it achieves a large cumulative reward, while keeping its probability of ruin small. The resulting policy is EXPLOIT-UCB-DOUBLE (given in Algorithm~\ref{double_exploit_algo}) and is proven to be regret-wise Pareto-optimal, hence giving an answer to the open problem in \cite{perotto}.

\subsection{A (Regret-wise) Pareto-optimal policy: EXPLOIT-UCB-DOUBLE}\label{exploit_dbl_study}

We start from EXPLOIT-UCB$(B)$ (Algorithm~\ref{survival_bandit_algo}). As this policy belongs to EXPLOIT, it is ruin-wise Pareto-optimal. However, as shown in Proposition~\ref{exploit_reward_upper_bound_proposition}, its cumulative reward is rather low, because its exploration is limited by the budget it allocates to each arm. We tackle this issue by performing a kind of doubling trick (see, e.g., \citealp{cesa_bianchi}) on the budget, which relaunches the exploration when the cumulative reward is large enough.

Let $n\in\mathbb{N}$ be a hyperparameter chosen in advance and for any integer $j\geq 0$, let 
\begin{equation*}
    t_j := \inf\left\{t\in \{0, \ldots, \min(\tau(B, \pi), T)\}: B + \sum_{s=1}^t X_s^{\pi_s}> jnB^2 \right\}, %\label{tj_definition}
\end{equation*}
with the convention that $t_j = \min(\tau(B, \pi), T) + 1$ if the above set is empty. At each round $t$, EXPLOIT-UCB-DOUBLE (see Figure~\ref{exploit_dbl_tree} and Algorithm \ref{double_exploit_algo}) performs 
\begin{itemize}
    \item EXPLOIT-UCB$(B)$ pretending that the initial budget is $B$ if $t<t_1$;
    \item EXPLOIT-UCB$(jnB^2)$ pretending that the initial budget is $jnB^2$ if $t_j\leq t< t_{j+1}$.
\end{itemize}

\noindent A more visual description of EXPLOIT-UCB-DOUBLE is given in Figure~\ref{exploit_dbl_tree} and its pseudo-code is given in Algorithm~\ref{double_exploit_algo}. We denote by $\pi^n$ the policy associated with EXPLOIT-UCB-DOUBLE with input parameter $n$.

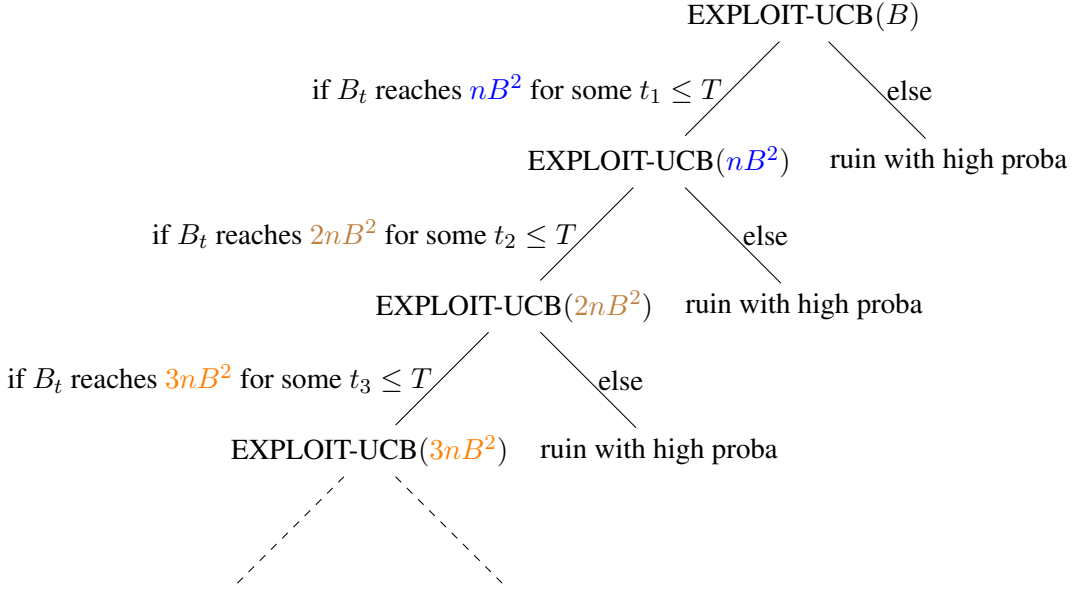
\begin{figure}[!h]
\centering
\begin{tikzpicture}
[sibling distance=10em,level distance=5em]
\node {EXPLOIT-UCB$(B)$}
    child {
        node {EXPLOIT-UCB$(\textcolor{blue}{nB^2})$}
        child {
            node {EXPLOIT-UCB$(\textcolor{brown}{2nB^2})$}
            child {
                node {EXPLOIT-UCB$(\textcolor{orange}{3nB^2})$}
                child {
                    node {}
                    edge from parent [dashed]}
                child {
                    node {}
                    edge from parent [dashed]}
                edge from parent node [left] {if $B_t$ reaches $\textcolor{orange}{3nB^2}$ for some $t_3\leq T$}}
            child {
                node {ruin with high proba}
                edge from parent node [right] {else}}
            edge from parent node [left] {if $B_t$ reaches $\textcolor{brown}{2nB^2}$ for some $t_2\leq T$}}
        child {
            node {ruin with high proba}
            edge from parent node [right] {else}}
        edge from parent node [left] {if $B_t$ reaches $\textcolor{blue}{nB^2}$ for some $t_1\leq T$}}
    child {
        node {ruin with high proba}
        edge from parent node [right] {else}};
\end{tikzpicture}
\caption{Description of EXPLOIT-UCB-DOUBLE}
%\smallskip
\small
\justify{EXPLOIT-UCB-DOUBLE (E-U-D) starts by performing EXPLOIT-UCB$(B)$ (root of the above tree). If $B_t$ never reaches $\textcolor{blue}{nB^2}$, then there is ruin with high probability. Otherwise, $B_t = \textcolor{blue}{nB^2}$ for some $t = t_1$, and E-U-D reinitializes the budget as $\textcolor{blue}{nB^2}$ and performs EXPLOIT-UCB$(\textcolor{blue}{nB^2})$. Again, if $B_t$ never reaches $\textcolor{brown}{2nB^2}$, then there is ruin with high probability. Otherwise, $B_t = \textcolor{brown}{2nB^2}$ for some $t = t_2$, and E-U-D reinitializes the budget as $\textcolor{brown}{2nB^2}$ and performs EXPLOIT-UCB$(\textcolor{brown}{2nB^2})$. Again, if $B_t$ never reaches $\textcolor{orange}{3nB^2}$, then there is ruin with high probability. Otherwise, $B_t = \textcolor{orange}{3nB^2}$ for some $t = t_3$, and E-U-D reinitializes the budget as $\textcolor{orange}{3nB^2}$ and performs EXPLOIT-UCB$(\textcolor{orange}{3nB^2})$, and so on, until either $t = T$ or there is ruin.}
\label{exploit_dbl_tree}
\end{figure}

The underlying principle of EXPLOIT-UCB-DOUBLE is that, as long as the cumulative reward is low, it performs safely like an EXPLOIT policy in order to minimize the probability of ruin. Then, progressively as the cumulative reward becomes larger, EXPLOIT-UCB-DOUBLE allocates more budget for exploration and behaves more similarly to UCB, so that it starts the cumulative reward maximization.

\SetAlgoVlined
\begin{algorithm}[t]
\vspace{0mm}
\SetNoFillComment
\DontPrintSemicolon
\SetKwInOut{Input}{input}
\Input{ parameter $n\in\mathbb{N}$; budget $B>0$.}
$j := 0; \ t_0 := 0$. \\
%$n \geq 1, \ j := 0$. \\
 \For{$t=1, \ldots, T$}{
    \If{$B + \sum_{s=1}^{t-1} X_s^{\pi^n_s}>(j+1)nB^2$}{
        Set $j := j+1$ and then, set $t_j := t-1$.
    }
    Set $\mathcal{A}_t := \left\{k\in [K]: \sum_{s=t_j + 1}^{t-1} X_s^{\pi^n_s} \1_{\pi^n_s = k} \geq -\frac{B + \sum_{s=1}^{t_j} X_s^{\pi^n_s}}{K} +1\right\}$; \\
    \For{$k=1, \dots, K$}{
    Set $N_k(t-1) := \sum_{s=1}^{t-1} \1_{\pi^n_s = k}$ and $\hat{X}_{t-1}^k := \frac{1}{N_k(t-1)}\sum_{s=1}^{t-1} X_s^{\pi^n_s}\1_{\pi^n_s = k}$.
    }
    %$\forall k \in [K], \ N_k(t-1) := \sum_{s=1}^{t-1} \1_{\pi_s = k}; \ \hat{X}_{t-1}^k := \frac{1}{N_k(t-1)}\sum_{s=1}^{t-1} X_s^{\pi_s}\1_{\pi_s = k}$. \\
    \eIf{$\mathcal{A}_t \ne \emptyset$}{
        %Pull arm $\argmax_{k\in \mathcal{A}_t} \frac{\sum_{s=1}^{t-1} X_s^{\pi_s}\1_{\pi_s = k}}{\sum_{s=1}^{t-1} \1_{\pi_s = k}} + \sqrt{\frac{6 \log (t-1)}{\sum_{s=1}^{t-1} \1_{\pi_s = k}}}$.}{
        Pull arm $\argmax_{k\in \mathcal{A}_t} \hat{X}_{t-1}^k + \sqrt{\frac{6 \log (t-1)}{N_k(t-1)}}$.}{
        %Pull arm $\argmax_{k\in [K]} \frac{\sum_{s=1}^{t-1} X_s^{\pi_s}\1_{\pi_s = k}}{\sum_{s=1}^{t-1} \1_{\pi_s = k}}$.}
        Pull arm $\argmax_{k\in [K]} \hat{X}_{t-1}^k$.}
 }
 \caption{EXPLOIT-UCB-DOUBLE}
 \label{double_exploit_algo}
\end{algorithm}

The next proposition, whose proof is given in Appendix~\ref{proba_survival_exploit_dbl_proof}, shows that the probability of ruin of EXPLOIT-UCB-DOUBLE is close to the one of an EXPLOIT policy.
\begin{proposition}\label{proba_survival_exploit_dbl}
Given $n\geq 1$, the probability of ruin of EXPLOIT-UCB-DOUBLE is bounded as follows:
\begin{equation*}
    P\left(\tau(B, \pi^n) < \infty\right) \leq p^{\mathrm{EX}} + \frac{(p^{\mathrm{EX}})^{nB}}{1 - (p^{\mathrm{EX}})^{nB}} \eqstar p^{\mathrm{EX}} + o_{T\to+\infty}(1),
\end{equation*}
where $(*)$ holds if $n\xrightarrow{T\to+\infty} +\infty$.
\end{proposition}

\noindent Proposition~\ref{proba_survival_exploit_dbl} shows that the progressively increasing exploration of EXPLOIT-UCB-DOUBLE is reasonable from the perspective of the probability of ruin, because it maintains the probability of ruin almost as small as the one of an EXPLOIT policy, which is ruin-wise Pareto-optimal. 

The next proposition shows that this progressively increasing exploration is also reasonable from the perspective of the reward maximization. The next proposition, whose proof can be found in Appendix~\ref{reward_exploit_dbl_proof}, shows that, in the case when there is no ruin, the expected cumulative reward of EXPLOIT-UCB-DOUBLE is asymptotically equal to the best possible expected reward.
\begin{proposition}\label{reward_exploit_dbl}
Let $n$ be such that $n\xrightarrow{T\to+\infty} +\infty$ and that $n = o(T^{1/4})$. Then, the reward given no ruin of EXPLOIT-UCB-DOUBLE with input parameter $n$ is bounded from below by
\begin{equation}\label{exploit_dbl_reward_given_no_ruin}
    \E\left[\sum_{t=1}^T X_t^{\pi^n_t} \1_{\tau(B, \pi^n)\geq t-1} \Bigg| \tau(B, \pi^n) \geq T\right] \geq \max_{k\in [K]} \mu_k T + o(T).
\end{equation}
\end{proposition}

\noindent The proof of Proposition~\ref{reward_exploit_dbl} also holds in the case $B = T$, that is, when there is no risk of ruin. Therefore, if you apply EXPLOIT-UCB-DOUBLE to a standard MAB (without a risk of ruin), then its expected cumulative reward will be equal to $\max_{k\in [K]} \mu_k T + o(T)$. In the standard MAB terminology, this means that its expected cumulative regret is sublinear: $\max_{k\in [K]} \mu_k T - \Rew(\pi^n) = o(T)$.

Now, if we gather the results of Proposition~\ref{proba_survival_exploit_dbl} and Proposition~\ref{reward_exploit_dbl}, the former states that the probability of ruin of EXPLOIT-UCB-DOUBLE is very small (it is almost ruin-wise Pareto-optimal), and when it does not ruin, the latter states that its cumulative reward is asymptotically maximal. Therefore, its cumulative reward is almost asymptotically optimal, and for any value of the input parameter $n$, EXPLOIT-UCB-DOUBLE is ``almost'' regret-wise Pareto-optimal. This is formalized in the next theorem.
\begin{theorem}\label{fake_teaser_theorem}
For any sequence of policies $\pi'$, the anytime policy EXPLOIT-UCB-DOUBLE (given in Algorithm~\ref{double_exploit_algo}) with input parameter $n$ satisfies
\begin{equation*}
    \sup_{F\in \mathcal{F}_{\{-1, 0, 1\}}^K} \Reg(\pi^n\|\pi') > 0 \implies \inf_F \Reg(\pi^n\|\pi') < \frac{(p^{\mathrm{EX}})^{nB}}{1 - (p^{\mathrm{EX}})^{nB}}\max_{k\in [K]} \mu_k.
\end{equation*}
\end{theorem}

\noindent The reason why EXPLOIT-UCB-DOUBLE with arbitrary $n$ is only ``almost'' regret-wise Pareto-optimal is that, for a very large horizon $T$, the additional exploration of EXPLOIT-UCB-DOUBLE induces an additional risk of ruin of order $\frac{(p^{\mathrm{EX}})^{nB}}{1 - (p^{\mathrm{EX}})^{nB}}$, which is of constant order if $n$ is not allowed to depend on $T$. Yet, if $n$ is allowed to depend on $T$, then we can drop the ``almost'' in the above theorem.

Following Propositions~\ref{proba_survival_exploit_dbl} and \ref{reward_exploit_dbl}, any $n = o(T^{1/4})$ such that $n\xrightarrow{T\to+\infty} +\infty$ is a valid choice. Yet, within that range, a larger $n$ will increase $t_1$ and hence, EXPLOIT-UCB-DOUBLE will behave longer like an EXPLOIT policy, undermining its cumulative regret in the long term. As a result, we recommend the subpolynomial (in $T$) $n = \log T$, which also showed a better practical performance than other values of $n$ (see the experiments in Section~\ref{experiment_section}), and we will formulate our final theorems for the specific choice $n = \log T$.
\begin{theorem}\label{teaser_theorem}
The (non-anytime) sequence of policies EXPLOIT-UCB-DOUBLE with input parameter $n = \log T$ is regret-wise Pareto-optimal.
\end{theorem}

\noindent The proof of Theorems~\ref{fake_teaser_theorem} and \ref{teaser_theorem} is given in Appendix~\ref{final_trivial_steps_section}. 

\begin{remark}
    In the above theorem, $n$ depends on $T$ and therefore, the sequence of policies is no longer anytime, as explained in Section~\ref{setup_section}. This is the price to pay to achieve the regret-wise Pareto-optimality. 
\end{remark}

\noindent Finally, please note that Theorems~\ref{fake_teaser_theorem} and \ref{teaser_theorem} provide answers to the open problem by \citealp{perotto}. Some discussion on the extent of the results described in this subsection is provided in the next subsection.

\subsection{Discussion}

In this subsection, we summarize the theoretical strengths and limitations of EXPLOIT-UCB-DOUBLE. 

\subsubsection{Cumulative Reward given no ruin}

EXPLOIT-UCB is an EXPLOIT policy, and as such, it has the advantage to be ruin-wise Pareto-optimal, but this comes at the cost of a budget-limited exploration. Because of this limited exploration, its cumulative reward is in general smaller than $(1-p^{\mathrm{EX}})\max_{k\in [K]} \mu_k T + o(T)$ (Proposition~\ref{exploit_reward_upper_bound_proposition}). By~\eqref{regret_as_product_survival_reward}, this implies that the expected cumulative reward given no ruin of EXPLOIT-UCB is in general smaller than $\max_{k\in [K]} \mu_k T + o(T)$:
\begin{equation*}
    \max_{k\in [K]} \mu_k T - \E\left[\sum_{t=1}^T X_t^{\pi_t} \1_{\tau(B, \pi)\geq t-1} \Bigg| \tau(B, \pi) \geq T\right] = \Omega(T).
\end{equation*}

\noindent Now, assume that you have to deploy a policy $\pi$ for a critical application of the standard MAB (with no risk of ruin). Since the application is critical, you decide to apply a conservative strategy, and for that, you set some arbitrary $B>0$ and you apply EXPLOIT-UCB (for any $t\geq \tau(B, \pi)$, the choice of arms is arbitrary). Then, the above result shows that the policy $\pi$ will likely have a linear regret, which is not satisfactory.

However, EXPLOIT-UCB-DOUBLE solves that issue by progressively increasing the exploration as its cumulative reward grows (Proposition~\ref{reward_exploit_dbl}). This is very important because it means that EXPLOIT-UCB-DOUBLE can be reasonably applied to a standard MAB setting and achieve a sublinear regret (in the sense of the standard MAB). This result is valid for any choice of input parameter $n$ for EXPLOIT-UCB-DOUBLE.

\subsubsection{Probability of ruin and Regret-wise Pareto Optimality}

As a matter of fact, the input parameter $n$ controls the risk of ruin of EXPLOIT-UCB-DOUBLE. On the one hand, if $n$ is chosen constant and independent of $T$, then EXPLOIT-UCB-DOUBLE is an anytime policy, but it is not regret-wise Pareto-optimal. This comes as no surprise, because it cannot compete against all the non-anytime sequences of policies. Yet, it is ``almost'' regret-wise Pareto-optimal, up to the term $\frac{(p^{\mathrm{EX}})^{nB}}{1 - (p^{\mathrm{EX}})^{nB}}\max_{k\in [K]} \mu_k$, which is exponentially small in $n$ and $B$ (Theorem~\ref{fake_teaser_theorem}). On the other hand, if $n = \log T$ depends on $T$, then EXPLOIT-UCB-DOUBLE is not an anytime sequence of policies anymore, but it becomes regret-wise Pareto-optimal (Theorem~\ref{teaser_theorem}).

From a practical perspective, how to choose the parameter $n$? Naturally, if the horizon $T$ is known before the experiment, then Theorem~\ref{teaser_theorem} recommends the choice $n = \log T$ and provides good theoretical guarantees for it. Yet, if $T\leq 10^{10}$, then $1\leq n\leq 23$ and there are not so many candidates for a good choice of $n$. In practice, $n = 1$ is a decent choice and brings a decent reward (which is almost as good as $n = \log T$). Experiments are given in the next subsection.

\subsubsection{Anytime vs Regret-wise Pareto Optimality}

\textit{Is there an anytime policy which is regret-wise Pareto-optimal?} As explained in Section~\ref{bernoulli_vs_multinomial_subsection}, in the case of multinomial arms, unlike in the case of Bernoulli arms, the arm $k^P$ with the largest probability of survival, and the arm $k^E$ with the largest expectation, may be different. In that case, even the best policy which knows the arm distributions is not trivial. Intuitively, such a policy should pull arm $k^P$ until a certain round $\Tilde{t}$ to minimize the risk of ruin, and then pull arm $k^E$ until $T$, but $\Tilde{t}$ should depend on $T$. It is therefore intuitive that the best such policy is not anytime, and we raise the following open question:
\begin{conjecture}
    No anytime policy is regret-wise Pareto-optimal for all arm distributions in $\mathcal{F}_{\{-1, 0, 1\}}$.
\end{conjecture}

\noindent Please note that, as explained in Section~\ref{bernoulli_vs_multinomial_subsection}, such a matter does not arise in the case of Bernoulli arm distributions, because in that case, $k^E = k^P$. The case of Bernoulli distributions does not capture this subtlety, and this is why we focused on the more difficult case of multinomial distributions in this paper.

\subsection{Experiments}\label{experiment_section}

In order to evaluate the practical performance of the algorithms introduced in this paper, we define the survival regret as
\begin{equation*}
    \SReg(\pi) = \left(1 - \exp\left(- \frac{B}{K}\sum_{k=1}^K \gamma(F_k)\right)\right) \max_{k\in [K]} \mu_k T - \Rew(\pi),
\end{equation*}
which is a hypothetical regret with respect to a policy that achieves the Pareto-optimal ruin probability and always pulls the arm with the highest expected reward given no ruin.

The setting chosen consists of $K = 3$ multinomial arms of common support $\{-1, 0, 1\}$ and parameters $(P(X=-1), P(X=0), P(X=1))$ respectively equal to $(0.4, 0.1, 0.5)$, $(0.05, 0.85, 0.1)$ and $(0.6, 0, 0.4)$, a horizon $T = 20000$ and a budget $B = 9$. Please note that this is an example such that arm 1 has the largest expected reward, while arm 2 has the largest probability of survival, that is, the lowest $\gamma(F_k)$. The curves are averages over $200$ simulations. Simulation results for other settings are given in Appendix \ref{experiments_appendix}.

We compare the survival regret of EXPLOIT-UCB and EXPLOIT-UCB-DOUBLE (with parameters $n\in \{1, \lceil\log T\rceil, 100\}$, where $\lceil\log T\rceil = 10$) to the classic bandit algorithms UCB (\citealp{auer}) and Multinomial Thompson Sampling (MTS), which is a generalization of Thompson sampling to multinomial rewards (\citealp{riou}). 
%\begin{wrapfigure}{r}{60mm}
 %\begin{center}
 %\vspace{-7mm}
%\includegraphics[width=70mm]{smab_long_full007.pdf}
%\vspace{-6mm}
%\caption{Comparison of the cumulative reward of EXPLOIT-UCB-DOUBLE, EXPLOIT-UCB, UCB1 and MTS}
%\label{fig:smab_hope}
 %\end{center}\vspace{-10mm}
%\end{wrapfigure}
As expected, EXPLOIT-UCB-DOUBLE with the parameter $n = \lceil\log T\rceil$ clearly outperforms all the other algorithms. Nevertheless, EXPLOIT-UCB-DOUBLE with other values of $n\in \{1, 100\}$ still has a decent performance and are, at least in practice, good anytime algorithms. On the other hand, MTS has the worst performance due to frequent ruins: we observed that its proportion of ruins culminates at around $30\%$, with an average number of survived rounds $\min\{\tau(B, \pi), T\}$ slightly above $14000$, which is in stark contrast to the ones of EXPLOIT-UCB and EXPLOIT-UCB-DOUBLE, 
at around $19000$. EXPLOIT-UCB has a performance which is comparable to MTS, and while its average time of ruin is very high (above $18000$), its quasi absence of exploration makes it frequently pull a suboptimal arm until the horizon $T$, inducing a larger regret.
%As expected, EXPLOIT-UCB-DOUBLE with the parameter $n = \lceil\log T\rceil$ clearly outperforms all the other algorithms. Nevertheless, EXPLOIT-UCB-DOUBLE with other values of $n\in \{1, 100\}$ still has a decent performance and are, at least in practice, good anytime algorithms. On the other hand, MTS has the worst performance due to frequent ruins: we observed that its proportion of ruins culminates at around $25\%$, with an average number of survived rounds $\min{\tau(B, \pi), T}$ slightly below $15000$, which is in stark contrast to the ones of EXPLOIT-UCB and EXPLOIT-UCB-DOUBLE, at around $19000$. EXPLOIT-UCB has a performance which is comparable to MTS, and while its average time of ruin is very high (above $19000$), its quasi absence of exploration makes it frequently pull a suboptimal arm until the horizon $T$, inducing a large regret.
%\begin{figure}[t]
    %\centering
    %\vspace{0mm}
    %\includegraphics[scale=0.8]%{smab_long_full007.pdf}
    %\caption{Comparison of the regret of EXPLOIT-UCB-DOUBLE (E.-U.-D.), EXPLOIT-UCB, UCB and MTS.}
    %\vspace{-5mm}
    %\label{fig:smab_hope}
%\end{figure}

\begin{figure}[t]
\begin{minipage}[c]{0.5\linewidth}
    %\begin{figure}[H]
        \centering
        \vspace{0mm}
        \includegraphics[scale=0.8]{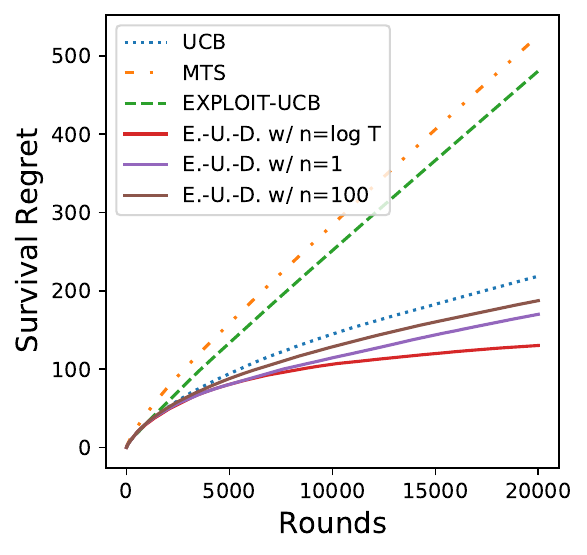}%{smab_long_full007.pdf}
        %\label{fig:smab_hope}
    %\end{figure}
\end{minipage}
\hfill
\begin{minipage}[c]{0.46\linewidth}
    \begin{table}[H]
        \begin{tabular}{|c|c|c|} 
         \hline
          %\multicolumn{3}{|c|}{Averages on $200$ Trials} \\ [0.5ex] 
         %\hline\hline
         Policies & \% of ruins & $\min\{\tau, T\}$ \\
         \hline\hline
         UCB & 10\% & 18011.93 \\
         \hline
         MTS & 30\% & 14030.095 \\
         \hline
         EX-UCB & 5.5\% & 18907.15 \\
         \hline
         EX-D$(\log T)$ & 4.5\% & 19105.545 \\
         \hline
         EX-D$(1)$ & 8\% & 18407.91 \\
         \hline
         EX-D$(100)$ & 7.5\% & 18507.75 \\ 
         \hline
        \end{tabular}
    %\caption{Proportion of Ruin of the Algorithms.}
    \end{table}
\end{minipage}
\label{fig:smab_hope}
\caption{Comparison of the regret, percentage of ruins and time of ruin of EXPLOIT-UCB-DOUBLE, EXPLOIT-UCB, UCB and MTS.}
\vspace{-5mm}
\end{figure}

\section{Extension to the case of non-integer rewards}\label{generalization_section}

This section is structured in two subsections: 
\begin{enumerate}
    \item first, we generalize our results on the probability of ruin to the case of rewards in $[-1, 1]$;
    \item secondly, we generalize our policies to the case of rewards in $[-1, 1]$ and explain the challenges in deriving regret-wise and ruin-wise Pareto-optimality-type theoretical guarantees.
\end{enumerate}

\subsection{Generalized Results on the Probability of Ruin}%\label{general_results_proba_ruin_study}

Let $\mathcal{F}_{[-1,1]}$ be the set of distributions bounded in $[-1, 1]$ which are not positive or zero (see Definition~\ref{positive_zero_definition}). Theorem~\ref{non_asymptotic_lower_bound_theorem} can be generalized as follows.
\begin{proposition}\label{non_asymptotic_lower_bound_corollary}
Let $(\alpha_k)_{k\in [K]}$ be as in Theorem~\ref{non_asymptotic_lower_bound_theorem}. For any policy $\pi$,
\begin{multline}
    \inf_{F\in \mathcal{F}_{[-1, 1]}^K} \left\{P_{(F_{1}, \ldots, F_{K})}(\tau(B, \pi)<\infty) - \exp\left(- (B+1)\sum_{k=1}^K \alpha_k \gamma(F_k)\right)\right\} <0 \\
    \implies \sup_{F\in \mathcal{F}_{[-1, 1]}^K} \left\{P_{(F_{1}, \ldots, F_{K})}(\tau(B, \pi)<\infty) - \exp\left(- (B+1)\sum_{k=1}^K \alpha_k \gamma(F_k)\right)\right\} >0. \label{lower_general}
\end{multline}
\end{proposition}

\noindent The proof of Proposition~\ref{non_asymptotic_lower_bound_corollary} follows the proof of Theorem~\ref{non_asymptotic_lower_bound_theorem}, except for the subadditivity argument, and it is given in Appendix~\ref{lower_bound_proba_long_proof}, next to the proof of Theorem~\ref{non_asymptotic_lower_bound_theorem}. 

The above bound is most likely not tight, because the subadditivity argument that we used is not tight. Actually, this is the main technicality which separates the general case from the multinomial case. In the multinomial case, it is known that at the time of ruin, the cumulative reward is exactly $-\lceil B \rceil$. In the general case, however, the cumulative reward at the time of ruin is stochastic and depends on the arm distributions as well as the policy $\pi$. 

Similarly to the case of rewards in $\{-1, 0, 1\}$, we state a lemma providing an insightful interpretation of the bound \eqref{lower_general}. While in the case of rewards in $\{-1, 0, 1\}$, Lemma~\ref{proba_lemma} gave an easy expression of $\gamma(F_k)$, in the general case, it is expressed as a limit.
\begin{lemma}%\label{proba_lemma}
For any $k\in [K]$,
\begin{equation*}
    \forall F_k\in \mathcal{F}_{[-1,1]}, \quad \frac{1}{B} \log P_{F_k}(\tau(B, k)<\infty) \leq \liminf_{B'\rightarrow \infty} \frac{1}{B'} \log P_{F_k}(\tau(B', k)<\infty) = -\gamma(F_k).
\end{equation*}
\end{lemma}

\noindent This lemma means that $\gamma(F_k)$ can be related to the probability of ruin $P_{F_k}(\tau(B,k)<\infty)$ of the stochastic process with increments i.i.d. from $F_k$.

\subsection{Generalized Results on the Regret}

We first generalize the EXPLOIT framework to rewards in $[-1, 1]$, and then give the results on EXPLOIT-UCB-DOUBLE based on the new EXPLOIT framework.

\subsubsection{Generalization of the EXPLOIT Framework}

For the sake of clarity, we generalize the framework EXPLOIT$\left(B_1, \dots, B_K\right)$ only in the case $B_1 = \dots = B_K = \frac{B}{K}$ (but we could similarly generalize it to any $(B_1, \dots, B_K)$), and we will refer to this generalization as EXPLOIT. Let $(Y_k^s)_{s\geq 1}$ be the rewards from arm $k\in [K]$. The main problem in the general case is that the sum of rewards is not necessarily an integer, and hence if you pull arm $k$ until the first round $t$ such that $\sum_{s=1}^t Y_k^s \leq -\frac{B}{K}$, then we have
\begin{equation*}
    \sum_{s=1}^t Y_k^s = -\frac{B}{K} - \kappa, \ \text{ with } \ \kappa\in [0, 1).
\end{equation*}

\noindent If $\kappa$ is large, then this will reduce the budget of the other arms. To remedy that issue, we conservatively decide to stop the exploration when $\sum_{s=1}^t Y_k^s < -\frac{B}{K}+1$, so that each arm has a budget share between $\frac{B}{K}-1$ and $\frac{B}{K}$. This is formalized as follows.
\begin{definition}%\label{exploit_definition}
For any $k\in [K]$, denoting by $\left(Y_k^s\right)_{s\geq 1}$ the rewards from arm $k$, let $\tau_k^{<} := \inf\big\{t\geq 1:$
$\sum_{s=1}^t Y_k^s < -\frac{B}{K} + 1\big\}$. We say that a policy $\pi$ belongs to the EXPLOIT framework if: 
\begin{itemize}
    \item[-] for any $t< \sum_{k=1}^K \tau_k^{<}, \ \pi_t \in \left\{k\in [K]: \sum_{s=1}^t \1_{\pi_s = k}< \tau_k^{<}\right\}$, and
    \item[-] for any $t\geq \sum_{k=1}^K \tau_k^{<}, \ \pi_t = \argmax\left\{\sum_{s=1}^t X_s^{\pi_s} \1_{\pi_s = k}\right\}$ (in case of tie, $\pi_t$ is the smallest arm index).
\end{itemize}
\end{definition}

\noindent We decompose $B = n_K K + b$ for an integer $n_K$ and $0\leq b< K$, and let $\alpha_k(B) \triangleq \frac{n_K + \1_{k+1<b}}{B}$ for any $k\in [K]$, which we denote by $\alpha_k$ in the absence of ambiguity on the initial budget $B$. Similarly to the multinomial case, all the distributions in EXPLOIT have the same probability of ruin $p^{\mathrm{EX}}\left(\frac{B}{K}, \dots, \frac{B}{K}\right)$. However, in the general case, this value is not deterministic. Yet, it can be bounded easily following the definition of EXPLOIT.
\begin{proposition}\label{exploit_proba_optimality}
It holds that
\begin{equation*}
    \exp\left(- B\sum_{k=1}^K \alpha_k(B) \gamma(F_k)\right) \leq p^{\mathrm{EX}}\left(\frac{B}{K}, \dots, \frac{B}{K}\right) \leq \exp\left(- B\sum_{k=1}^K \left(\alpha_k(B) - \frac{1}{B}\right) \gamma(F_k)\right).
\end{equation*}
\end{proposition}

\noindent The RHS of Proposition~\ref{exploit_proba_optimality} does not match the lower bound of Proposition~\ref{non_asymptotic_lower_bound_corollary}. However, if we apply it to the initial budget $B' = B+K+1$, it yields 
\begin{equation*}
    p^{\text{EX}}\left(\frac{B'}{K}, \dots, \frac{B'}{K}\right) \leq \exp\left(- (B+1)\sum_{k=1}^K \alpha_k(B+1) \gamma(F_k)\right),
\end{equation*}
which coincides with the lower bound in Proposition~\ref{non_asymptotic_lower_bound_corollary}. Therefore, the looseness of the ruin probability in the general case  (w.r.t. the bound of Proposition~\ref{non_asymptotic_lower_bound_corollary}) corresponds to a budget loss of at most $K+1$.

\subsubsection{Generalization of EXPLOIT-UCB-DOUBLE}

We define EXPLOIT-UCB and EXPLOIT-UCB-DOUBLE based on the newly defined EXPLOIT framework for rewards in $[-1, 1]$. The pseudo-code for EXPLOIT-UCB-DOUBLE remains the same as in the multinomial case and is given in Algorithm~\ref{double_exploit_algo}. Interestingly, most of the results can be extended from the multinomial case to the general case. 

\begin{proposition}
    Let $\pi^n$ be the policy associated to the policy EXPLOIT-UCB-DOUBLE with input parameter $n$. Its probability of ruin is upper bounded by
    \begin{equation*}
        P\left(\tau(B, \pi^n) < \infty\right) \leq p^{\mathrm{EX}} + \frac{(p^{\mathrm{EX}})^{nB}}{1 - (p^{\mathrm{EX}})^{nB}}.
    \end{equation*}
    The cumulative reward given no ruin of EXPLOIT-UCB-DOUBLE is bounded from below by
    \begin{equation*}
        \E\left[\sum_{t=1}^T X_t^{\pi^n_t} \1_{\tau(B, \pi^n)\geq t-1} \Bigg| \tau(B, \pi^n) \geq T\right] \geq \max_{k\in [K]} \mu_k T + o(T).
    \end{equation*}
    As a consequence of the above results, it holds that, for any sequence of policies $\pi'$,
    \begin{equation*}
        \sup_{F\in \mathcal{F}_{\{-1, 0, 1\}}^K} \Reg(\pi^n\|\pi') > 0 \implies \inf_F \Reg(\pi^n\|\pi') < \frac{(p^{\mathrm{EX}})^{nB}}{1 - (p^{\mathrm{EX}})^{nB}}\max_{k\in [K]} \mu_k.
    \end{equation*}
\end{proposition}

\noindent As explained before, the main difficulty is that in general, for any arm $k, \ \sum_{s=1}^{\tau_k^{<}} Y_k^s$ is not deterministic (using the notations of Definition~\ref{exploit_definition}). As a result, even for EXPLOIT policies, the probability of ruin cannot be decomposed as a product of independent ruin probabilities of the arms. The same reason leads to the looseness in the subadditivity bound (see Lemma~\ref{sub_additivity_lemma}), leading to the $(B+1)$ factor in the bound of Proposition~\ref{non_asymptotic_lower_bound_corollary} instead of $B$. For that reason, we believe that the bound of Proposition~\ref{non_asymptotic_lower_bound_corollary} is not tight and conjecture that EXPLOIT-UCB-DOUBLE is regret-wise Pareto-optimal even in the general case.
\begin{conjecture}
    EXPLOIT-UCB-DOUBLE is regret-wise Pareto optimal in the general case of arm distributions in $\mathcal{F}_{[-1, 1]}$.
\end{conjecture}

\noindent Finally, note that in the general case as well, the reward given no ruin of EXPLOIT-UCB-DOUBLE (see Proposition~\ref{reward_exploit_dbl}) is asymptotically optimal and equal to $\max_{k\in [K]} \mu_k T$, which makes it worth applying to more standard bandit settings where an algorithm with a stronger exploitation component is desired.

\section{Conclusion}

In this paper, we introduced the S-MAB, an extension of the MAB with a risk of ruin, which naturally follows from many practical applications but is considerably more difficult to study. For example, contrary to the MAB, no policy can achieve a sublinear regret in the standard sense, because every single pull increases considerably the probability of ruin. Our contributions are threefold:
\begin{itemize}
    \item[-] we formally defined the problem, defined the objective to achieve with the regret-wise Pareto-optimality and introduced the key notion to our problem with the time of ruin and the probability of ruin. Furthermore, we explained how an optimal policy needs to minimize the probability of ruin while at the same time maximize the cumulative reward given no ruin, which are two concepts in apparent opposition;
    \item[-] we studied the probability of ruin, on which we provided both a lower bound (Theorem~\ref{non_asymptotic_lower_bound_theorem}) and policies achieving this lower bound (EXPLOIT policies);
    \item[-] using a doubling trick over an EXPLOIT policy, we derived a policy which is almost regret-wise Pareto-optimal, and can be made exactly Pareto-optimal if the policy knows the horizon before starting the procedure. This provides an answer to an open problem from COLT 2019 (see \citealp{perotto}).
\end{itemize}

\noindent Along the way, we raised several open questions which we keep for future work: in the case of integer rewards, is there a policy which is regret-wise Pareto-optimal and does not depend on the horizon? In the general case of bounded rewards, most of our results extend, except that the lower bound on the probability of ruin is seemingly not tight and we did not prove that EXPLOIT-UCB-DOUBLE is regret-wise Pareto-optimal with $n = \log T$. Can we improve upon those? 

This is a fairly new and yet unexplored problem, but we believe that it is very rich and paves the way to a myriad of new questions.

%In this paper, we introduced the survival MAB, an extension of the MAB with a risk of ruin, which naturally follows from many practical applications but is considerably more difficult to study. For example, contrary to the MAB, no policy can achieve a sublinear regret in the standard sense, because every single pull increases considerably the probability of ruin. We started by providing a Pareto-type lower bound on the probability of ruin, as well as policies achieving this bound. Building upon a doubling trick on such policies, we finally derived a Pareto-optimal policy in the sense of the regret (at least in the case of rewards in $\{-1, 0, 1\}$), giving an answer to an open problem from COLT 2019 (see \citealp{perotto}).

%\bibliographystyle{plainnat}
\bibliography{references}

\newpage

\newpage
\appendix
\allowdisplaybreaks[1]

%\section{Basic Properties on the Reward and the Probability of Ruin}

\section{General Classic Results}

\subsection{A Classic Lemma in the Theory of Stochastic Processes}

The following result is classic in the theory of stochastic processes and we state it without a proof. It will be used throughout the paper.

\begin{lemma}\label{constant_survival}
Consider the setting of $K\geq 1$ arms of multinomial distributions $F_1, \dots, F_K$ (of common support $\{-1, 0, 1\}$). Let $B>0$ be a positive integer. Then,
\begin{equation*}
    P(\tau(B, k)<\infty) = \min\left(1, \left(\frac{P_{X\sim F_k}(X=-1)}{P_{X\sim F_k}(X=1)}\right)^B\right).
\end{equation*}
\end{lemma}

\section{Proof of the Linearity of the Classic Regret (Proposition~\ref{linear_regret})}\label{classic_regret_linear_proof}

In this appendix, we prove a slightly stronger version of Proposition~\ref{linear_regret}, namely that Proposition~\ref{linear_regret} even holds in the case where the supremum on $F$ is taken on Bernoulli arm distributions of support $\{-1, 1\}$. Let $\F_{\{-1, 1\}}^K$ be the set of $K$-tuples of Bernoulli arm distributions of support $\{-1, 1\}$.
\begin{proposition}\label{linear_regret_append}
Assume that the initial budget $B>2$. For any policy $\pi$, it holds that
\begin{equation*}
    \sup_{F\in \F_{\{-1, 1\}}^K} \sup_{\tilde{\pi}} \ \Reg(\pi\|\tilde{\pi}) > 0.
\end{equation*}
\end{proposition}

\begin{proof}
Let $\pi = (\pi^T)_{T\geq 1}$ be a sequence of policies, and let $k_0 \in [K]$ such that there are infinitely many $T\geq 1$ such that
\begin{equation}
    P(\pi_{1}^T = k_0) \leq \frac{1}{K}. \label{first_sample_ill_chosen}
\end{equation}

\noindent We denote by $S$ the set of all such $T\geq 1$:
\begin{equation*}
    S := \left\{T\geq 1: P(\pi_{1}^T = k_0) \leq \frac{1}{K}\right\},
\end{equation*}
and we note that $|S| = \infty$. We then let $F_1, \dots, F_K$ be the Bernoulli arm distributions, with respective parameters $p_1, \dots, p_K$ such that $k_0$ is the only optimal arm:
\begin{equation*}
    p_{k_0} > \max_{k\ne k_0} p_k.
\end{equation*}

\noindent W.l.o.g., we can assume that $k_0 = 1$. Recall that, by definition, for any $k\in [K]$,
\begin{equation*}
    p_k = P_{X\sim F_k}(X = +1) = 1 - P_{X\sim F_k}(X = -1).
\end{equation*}

\noindent Let $\Delta := \frac{1-p_1}{p_1}$. 
Denoting by $\Rew(1)$ the reward of the (optimal) policy $\pi_t = 1$ for any $t\geq 1$, it holds that
$$\Rew(1) - \Rew(\pi^T) = \E\left[\sum_{t=1}^T X_t^1 \1_{\tau(B, 1)\geq t-1}\right] - \E\left[\sum_{t=1}^T X_t^{\pi_t} \1_{\tau(B, \pi)\geq t-1}\right].$$

\noindent Let us then compute the cumulative expected reward of the policy pulling only arm 1:
\begin{align*}
    \E\left[\sum_{t=1}^T X_t^1\1_{\tau(B, 1)\geq t-1}\right] &= \mu_1 \sum_{t=1}^T P\left(\tau(B, 1)\geq t-1\right) \\
    &= \mu_1 P\left(\tau(B, 1) \geq T\right) T + o(T) \\
    &= \mu_1 P\left(\tau(B, 1) = \infty\right) T + o(T).
\end{align*}

\noindent Then, using Lemma~\ref{constant_survival}, we deduce
\begin{align}
    \E\left[\sum_{t=1}^T X_t^1\1_{\tau(B, 1)\geq t-1}\right] &= \mu_1 \left(1-\Delta^{\lceil B\rceil}\right) T + o(T) \nonumber \\
    &= (2p_1 -1)\left(1-\Delta^{\lceil B\rceil}\right)T + o(T). \label{one_arm_sublinear}
\end{align}

\noindent The cumulative reward of $\pi^T$ is upper-bounded by
\begin{align*}
    &\E\left[\sum_{t=1}^T X_t^{\pi^T_t}\1_{\tau(B, \pi^T)\geq t-1}\right] \\
    &= \E\left[X_1^{\pi_1^T} + \sum_{t=2}^T X_t^{\pi^T_t}\1_{\tau(B, \pi^T)\geq t-1}\right] \\
    &\leq \E\left[X_1^{\pi_1^T} + \sum_{t=2}^T X_t^{1}\1_{\tau(B, \pi^T)\geq t-1}\right] \\
    &= \E\left[X_1^{\pi_1^T}\right] + \E\left[\sum_{t=2}^T X_t^{1}\1_{\forall s\leq t-1, B + \sum_{r=1}^s X_r^{\pi^T_r}>0}\right] \\
    &= \E\left[X_1^{\pi_1^T}\right] + P\left(X_1^{\pi_1^T} = 1\right) \E\left[\sum_{t=2}^T X_t^{1}\1_{\forall s\leq t-1, B + \sum_{r=1}^s X_r^{\pi^T_r}>0} \Big| X_1^{\pi_1^T} = 1\right] \\
    &\quad \quad \quad \quad \quad \ + \left(1 - P\left(X_1^{\pi_1^T} = 1\right)\right) \E\left[\sum_{t=2}^T X_t^{1}\1_{\forall s\leq t-1, B + \sum_{r=1}^s X_r^{\pi^T_r}>0} \Big| X_1^{\pi_1^T} = -1\right] 
    \allowdisplaybreaks[4]\\
    &= \E\left[X_1^{\pi_1^T}\right] + P\left(X_1^{\pi_1^T} = 1\right) \E\left[\sum_{t=2}^T X_t^{1}\1_{\forall s\leq t-1, B + 1 + \sum_{r=2}^s X_r^{1}>0} \right] \\
    &\quad \quad \quad \quad \quad \ + \left(1 - P\left(X_1^{\pi_1^T} = 1\right)\right) \E\left[\sum_{t=2}^T X_t^{1}\1_{\forall s\leq t-1, B - 1 + \sum_{r=2}^s X_r^{1}>0} \right].
\end{align*}

\noindent But then, using \eqref{one_arm_sublinear}, we know that
$$\E\left[\sum_{t=2}^T X_t^{1}\1_{\forall s\leq t-1, B + 1 + \sum_{r=2}^s X_r^{1}>0} \right] = (2p_1-1)\left(1 - \Delta^{\lceil B\rceil+1}\right)T + o(T)$$
and
$$\E\left[\sum_{t=2}^T X_t^{1}\1_{\forall s\leq t-1, B - 1 + \sum_{r=2}^s X_r^{1}>0} \right] = (2p_1-1)\left(1 - \Delta^{\lceil B\rceil-1}\right)T + o(T).$$

\noindent Therefore, we can deduce that the agent's total cumulative reward is upper-bounded by
\begin{multline*}
    \E\left[\sum_{t=1}^T X_t^{\pi_t}\1_{\forall s\leq t-1, B + \sum_{r=1}^s X_r^{\pi^T_r}>0}\right] \leq (2p_1 - 1) T \times  \\
    \left(P\left(X_1^{\pi_1^T} = 1\right)\left(1-\Delta^{\lceil B\rceil +1}\right) + \left(1 - P\left(X_1^{\pi_1^T} = 1\right)\right)\left(1-\Delta^{\lceil B\rceil -1}\right)\right) + o(T).
\end{multline*}

\noindent We then decompose 
\begin{align*}
    P\left(X_1^{\pi_1^T} = 1\right) &= \sum_{k=1}^K P\left(\pi_1^T = k\right) P\left(X_1^{\pi_1^T} = 1 \right) \\
    &= \sum_{k=1}^K p_k P\left(\pi_1^T = k\right), 
\end{align*}
and by~\eqref{first_sample_ill_chosen}, if $T\in S$, then
\begin{equation*}
    P\left(X_1^{\pi_1^{T}} = 1\right) \leq \frac{1}{K}p_1 + \left(1 - \frac{1}{K}\right) \max_{k\ne 1} p_k < p_1.
\end{equation*}

\noindent Now, please note that the function
\begin{equation*}
    p \mapsto p\left(1-\Delta^{\lceil B\rceil +1}\right) + (1-p)\left(1-\Delta^{\lceil B\rceil -1}\right)
\end{equation*}
is strictly increasing on $[0, p_1]$ and equal to $1 - \Delta^{\lceil B\rceil}$ at $p = p_1$. As a result, there exists $\delta>0$ such that, for any $T\in S$,
\begin{align*}
    \E\left[\sum_{t=1}^{T} X_t^{\pi^{T}_t}\1_{\forall s\leq t-1, B + \sum_{r=1}^s X_r^{\pi^{T}_r}>0}\right] &\leq (2p_1 - 1)T \left(1 - \Delta^{\lceil B\rceil} - \delta\right) + o(T) \\
    &= \E\left[\sum_{t=1}^{T} X_t^1\1_{\tau(B, 1)\geq t-1}\right] - (2p_1 - 1) \delta T + o(T).
\end{align*}

\noindent By definition of the regret, this implies
\begin{equation*}
    \sup_{\Tilde{\pi}}\Reg(\pi\Vert\Tilde{\pi}) \geq (2p_1 - 1)\delta, 
\end{equation*}
concluding the proof of the proposition.
\end{proof}

\begin{remark}
In the proof of this proposition we considered an instance such that $p_1<1$ so that $\Delta>0$. In the case $p_1 = 1$ or more generally when there exists a positive arm, it is possible to achieve zero regret if the initial budget is $B>K-1$ with, for instance, the policy $\pi$ which performs the classic bandit algorithm UCB (see, e.g., \citealp{bubeck9}) when $B_{t-1}>K-1$, and pulls one of the arms which has not yielded a negative reward otherwise.
\end{remark}

\section{Properties of KL Divergence}\label{append_kl}
In this appendix we prepare lemmas on the KL divergence used for the analysis of the ruin probability.

\begin{lemma}\label{lem_kl_lambda}
Let $P\in\mathcal{F}_{[-1,1]}$ be a distribution over $[-1,1]$ with a positive expectation and let $\Lambda(\lambda) := \log \E\left[e^{\lambda X}\right]$ be its logarithmic moment-generating function. Then, there exists $\lambda'<0$ such that $\Lambda(\lambda')=0$ and it satisfies
\begin{align*}
    \inf_{Q: E_{X\sim Q}[X]<0} \frac{\KL(Q\Vert P)}{\E_{X\sim Q}[-X]}= -\lambda'.
\end{align*}
\end{lemma}

\begin{proof}
First we have
\begin{equation*}
    \lim_{\lambda\to-\infty}\Lambda(\lambda)=\infty\nonumber
\end{equation*}
since $P[X<0]>0$ because the distribution $P$ is not positive or zero by definition of $\mathcal{F}_{[-1, 1]}^K$.
Therefore, by the continuity of $\Lambda(\lambda)$
there exists $\lambda'<0$ such that
$\Lambda(\lambda')=0$
since $\Lambda(0)=0$ and $\Lambda'(0)=\E_{X\sim P}[X]>0$.

Now, let
$$\Lambda^*(x) := \sup_{\lambda}\left\{\lambda x - \Lambda(\lambda)\right\}$$
be the Fenchel-Legendre transform of $\Lambda(\lambda)$
and define $x'=\Lambda'(\lambda')$.
Then we have $\Lambda^*(x')=\lambda'x'-\Lambda(\lambda')=\lambda'x'$.
Therefore,
\begin{align}
\inf_{x<0} \frac{\Lambda^*(x)}{-x}
&\le
\frac{\Lambda^*(x')}{-x'}
\nonumber\\
&=
\frac{\lambda'x'}{-x'}\nonumber\\
&= -\lambda'.\nonumber
\end{align}
On the other hand,
\begin{align}
\inf_{x<0} \frac{\Lambda^*(x)}{-x}
&=
\inf_{x<0} \sup_{\lambda}\frac{\lambda x -\Lambda(\lambda)}{-x}\nonumber\\
&\ge
\inf_{x<0} \frac{\lambda' x -\Lambda(\lambda')}{-x}\nonumber\\
&=
\inf_{x<0} \frac{\lambda' x}{-x}\nonumber\\
&=
-\lambda'.\nonumber
\end{align}
Therefore we see that
\begin{align}
\inf_{x<0} \frac{\Lambda^*(x)}{-x}
&=
-\lambda'.\nonumber
\end{align}
It is well-known as the relation between Cramer's theorem and Sanov's theorem (see, e.g., \citealp{dembo_zeitouni}) that for $x< \E_{X\sim P}[X]$,
\begin{equation*}
    \Lambda^*(x) = \inf_{Q: \E_{X\sim Q}[X]\leq x} \KL(Q\Vert P),
\end{equation*}
which concludes the proof of the lemma.
\end{proof}

\begin{lemma}\label{lem_P_exist}
Let $Q$ be an arbitrary distribution such that $\E_{X\sim Q}[X]<0$
and fix $\epsilon>0$.
Then, there exists $P$ such that $\E_{X\sim P}[X]>0$ and
\begin{align}
\frac{\KL(Q \Vert P)}{\E_{X\sim Q}[-X]}
\le
\inf_{Q': \E_{X\sim Q'}[X]<0}\frac{\KL(Q' \Vert P)}{\E_{X\sim Q'}[-X]}(1+\epsilon).\label{kl_lambda_goal}
\end{align}
\end{lemma}
\begin{proof}
%charles
Let $p\in \left(0, \min\left\{\frac{\E_{X\sim Q}[-X]}{1 + \E_{X\sim Q}[-X]}, 1 - \exp\left(-\frac{(\E_{X\sim Q}[X])^2}{2(\frac{1}{\epsilon}+1)}\right)
\right\}\right)$, and
%For sufficiently small $p\in(0,1/2]$, 
let $Q_{p}=(1-p)Q+p\delta_{\{1\}}$ be the mixture of $Q$ and the point mass at $X=1$.
Let $\Lambda_p(\lambda)=\log \E_{X\sim Q_p} [e^{\lambda X}]$ be the logarithmic moment-generating function of $Q_p$
and $\lambda^*>0$ be such that $\Lambda_p(\lambda^*)=0$.
Such $\lambda^*$ exists and satisfies $\Lambda_p'(\lambda^*)>0$
since $\Lambda_p(0)=0$ and
%charles
$\Lambda_p'(0)=(1-p)\E_{X\sim Q}[X]+p<0$ by $p\leq \frac{\E_{X\sim Q}[-X]}{1 + \E_{X\sim Q}[-X]}$.
%$\Lambda_p'(0)=(1-p)\E_{X\sim Q_p}[X]+p<0$
%for sufficiently small $p$.

Let $P_{p}$ be the distribution such that $\mathrm{d}P_{p}/\mathrm{d}Q_{p}(x)=e^{\lambda^* x-\Lambda_p(\lambda^*)}=e^{\lambda^* x}$.
Then we have $\E_{X\sim P_p}[X]=\Lambda_p'(\lambda^*)>0$.
Here note that 
%charles
$\E_{X\sim P_{p}}[e^{-\lambda^* X}]=\E_{X\sim Q_{p}}[e^{-\lambda^* X} e^{\lambda^*X-\Lambda_p(\lambda^*)}]=1$.
%$\E_{X\sim P_{p}}[e^{-\lambda^* X}]=\E_{X\sim Q_{p}}[e^{-\lambda'X} e^{\lambda^*X-\Lambda_p(\lambda^*)}]=1$.
Therefore, by Lemma~\ref{lem_kl_lambda} we have
\begin{equation}
    \inf_{Q': \E_{X\sim Q'}[X]<0}\frac{\KL(Q' \Vert P_p)}{\E_{X\sim Q'}[-X]} = \lambda^*.
    \label{kl_lambda_equiv}
\end{equation}

\begin{comment}
%charles
Then, we denote by $\mathcal{H}_{\tau}=(Y_1^1, \dots, Y_1^{\tau})$ the sample path until the time of ruin, where $Y_1^n$ is the reward of the $n$-th pull of arm $1$. We denote by $h_t$ a realization of $\mathcal{H}_{\tau}$. Please note that since the rewards are in $[-1, 1]$,
$$-B \leq 1^\top h_t \leq -B+1$$

\noindent We can then write
\begin{align*}
    P_p(\tau<\infty) &= \sum_{h_t} P_p(h_t) \\
    &= \sum_{h_t} \exp(\lambda^* \overbrace{1^\top h_t}^{\leq -B+1}) Q_p(h_t) \\
    &\leq e^{-(B-1)\lambda^*} Q_p(\tau<\infty) \\
    &= e^{-(B-1)\lambda^*},
\end{align*}

\noindent We deduce, using Lemma \ref{proba_lemma}, that
\begin{equation*}
    \inf_{Q': \E_{X\sim Q'}[X]<0}\frac{\KL(Q' \Vert P_p)}{\E_{X\sim Q'}[-X]} \geq \liminf_{B\to +\infty} \frac{-1}{B} \log P_p(\tau<\infty) \geq -\lambda^*.
\end{equation*}
\end{comment}

\noindent On the other hand,
\begin{align}
\KL(Q \Vert P_p)
&=
\E_{X\sim Q}\left[
\1_{X<1}\log\frac{\mathrm{d}Q}{\mathrm{d}P_p}(X)
+
\1_{X=1}\log\frac{\mathrm{d}Q}{\mathrm{d}P_p}(X)
\right]
\nonumber\\
&=
\E_{X\sim Q}\left[
\1_{X<1}\log\frac{1}{1-p}\frac{\mathrm{d}Q_p}{\mathrm{d}P_p}(X)
\right]
+
Q(X=1)\log\frac{Q(X=1)}{P_p(X=1)}
\nonumber\\
&\le
\log\frac{1}{1-p}
+
\E_{X\sim Q}\left[
\1_{X<1}\log\frac{\mathrm{d}Q_p}{\mathrm{d}P_p}(X)
\right]
+
Q(X=1)\log\frac{Q(X=1)}{Q_p(X=1)e^{\lambda^*\cdot 1}}
\nonumber\\
&=
\log\frac{1}{1-p}
+
\E_{X\sim Q}\left[
\1_{X<1}(-\lambda^* X)
\right]
-Q(X=1)\lambda^*\cdot 1
+
Q(X=1)\log\frac{Q(X=1)}{Q_p(X=1)}
\nonumber\\
&=
\log\frac{1}{1-p}
+
\lambda^*\E_{X\sim Q}\left[
-X
\right]
+
Q(X=1)\log\frac{Q(X=1)}{Q_p(X=1)}
\nonumber\\
&\le
\log\frac{1}{1-p}
+
\lambda^*\E_{X\sim Q}\left[
-X
\right],\nonumber
\end{align}
which, combined with \eqref{kl_lambda_equiv}, implies that
\begin{align}
\frac{\KL(Q \Vert P_p)}{\E_{X\sim Q}[-X]}
&\le
\inf_{Q': \E_{X\sim Q'}[X]<0}\frac{\KL(Q' \Vert P)}{\E_{X\sim Q'}[-X]}
+
\frac{\log\frac{1}{1-p}}{\E_{X\sim Q}[-X]}.
\label{kl_lambda_upper}
\end{align}
Comparing \eqref{kl_lambda_equiv} and \eqref{kl_lambda_upper} with \eqref{kl_lambda_goal},
we see that
it is sufficient to show that
\begin{align}
\frac{\log\frac{1}{1-p}}{\E_{X\sim Q}[-X]}
\le
\epsilon
\lambda^*\label{kl_tobeproved2}
\end{align}
to show \eqref{kl_lambda_goal}.
Note that we obtain from Pinsker's inequality that
\begin{align}
\lambda^*
&\ge\frac{\KL(Q\Vert P_p)-\log\frac{1}{1-p}}{\E_{X\sim Q}[-X]}\nonumber\\
&\ge\frac{
2\left(\frac{\E_{X\sim Q}[X]}{2}-\frac{\E_{X\sim P_p}[X]}{2}\right)^2
-\log\frac{1}{1-p}}{\E_{X\sim Q}[-X]}\nonumber\\
&\ge
\frac{
\frac{(\E_{X\sim Q}[X])^2}{2}
-\log\frac{1}{1-p}}{\E_{X\sim Q}[-X]},\nonumber
\end{align}
recalling that $P_p$ and $Q$ are supported over $[-1,1]$, and have positive and negative expectations, respectively.
Therefore we obtain \eqref{kl_tobeproved2} since
\begin{align}
\frac{1}{\lambda^*}
\frac{\log\frac{1}{1-p}}{\E_{X\sim Q}[-X]}
&\le
\frac{\log\frac{1}{1-p}}{
\frac{(\E_{X\sim Q}[X])^2}{2}
-\log\frac{1}{1-p}}
\le
\epsilon,\nonumber
\end{align}
where the last inequality follows from
since $p\leq 1 - \exp\left(-\frac{(\E_{X\sim Q}[X])^2}{2(\frac{1}{\epsilon}+1)}\right)$.
\end{proof}

\section{Detailed Proof of the Lower Bound on the Probability of Ruin (Theorem~\ref{non_asymptotic_lower_bound_theorem} and Proposition~\ref{non_asymptotic_lower_bound_corollary})}\label{lower_bound_proba_long_proof}

In this section, we give a detailed proof of Theorem~\ref{non_asymptotic_lower_bound_theorem} and Proposition~\ref{non_asymptotic_lower_bound_corollary}. The proof of the lower bound both in the case of multinomial arms of support $\{-1, 0, 1\}$ (Theorem~\ref{non_asymptotic_lower_bound_theorem}) and in the general case of rewards bounded in $[-1, 1]$ (Proposition~\ref{non_asymptotic_lower_bound_corollary}) stems from the asymptotic lower bound, which is common to both aforementioned cases and is given in Theorem~\ref{asymptotic_lower_bound_theorem}. The passage from the asymptotic to the non-asymptotic bound relies on sub-additivity properties, which is given in Lemma~\ref{sub_additivity_lemma} and for which formulas differ depending on the case considered.

If for all the arms $k\in [K]$,
$$\inf_{Q_k: \E_{X\sim Q_k}[X]<0}\frac{\KL(Q_k \Vert F_{k})}{\E_{X\sim Q_k}[-X]} = 0,$$
then the result becomes trivial. This is why we are going to make the following assumption in the proof:
\begin{assumption}\label{existence_good_arm_assumption}
There exists an arm $k\in [K]$ such that $P(\tau(B, k)=\infty)>0$.
\end{assumption}

\subsection{Details of the Proof of Lemma~\ref{fundamental_lemma_proba_ruin}}\label{justification_lim_q}

In this subsection, we provide the justification for
\begin{equation*}
    \lim_{B\rightarrow +\infty}Q(\mathcal{H}_{\tau}\in T(Q))= 1,
\end{equation*}
which was omitted in the main text. 

For any $t\geq 1$ and any $n = (n_1, \dots, n_K)$ such that $\sum_{k=1}^K n_k = t$, we introduce the following random events:
\begin{equation*}
    \begin{array}{lll}
        U(n, t) &:= \left\{\left|\sum_{k=1}^K\left(n_k \KL(Q_k\Vert F_k)-\sum_{m=1}^{n_k} \log \frac{d Q_k}{d F_k}(y_k^m)\right) \right| \leq \frac{t}{B^{\frac{1}{4}}}\right\}, \\
        V(n, t) &:= \left\{\left|\sum_{k=1}^K\left(n_k \E_{X\sim Q_k}[X]-\sum_{m=1}^{n_k} y_k^m\right)\right|\leq \frac{t\Delta_Q}{B^{\frac{1}{4}}}\right\}, \\
        W(n, t) &:= \left\{\sum_{k=1}^K \sum_{m=1}^{n_k} y_k^m \leq -\frac{t\Delta_Q}{2}\right\}.
    \end{array}
\end{equation*}

\noindent Let $h_t$ be a realization of $\mathcal{H}_{\tau}$. Then, please note that the probability of the event $\left\{h_t\in T(Q)\right\}$ is uniformly bounded independently of the policy $\pi$ by the probability of the following event:
\begin{equation*}
    \forall n = (n_1, \dots, n_K) \ \text{ s.t. } \ \sum_{k=1}^K n_k = t: \ U(n, t), \ V(n, t), \ W(n, t).
\end{equation*}

\noindent For any $k\in [K]$, let 
\begin{equation*}
    d_k := \max_{y_1\in [-1, 1]} \log \frac{dQ_k}{dF_k}(y_1) - \min_{y_2\in [-1, 1]} \log \frac{dQ_k}{dF_k}(y_2) \quad \text{ and } \quad D := \max_{k\in [K]} d_k.
\end{equation*}

\noindent Then, a direct application of Hoeffding's inequality gives the bounds
\begin{align*}
    &Q(U(n, t)^c) \leq 2\exp\left(-\frac{2t}{D\sqrt{B}}\right), \\
    &Q(V(n, t)^c) \leq 2\exp\left(-\frac{t\Delta_Q^2}{2\sqrt{B}}\right), \\
    &Q(W(n, t)^c) \leq \exp\left(-\frac{t\Delta_Q^2}{2}\right).
\end{align*}

\noindent Let $C := \max\left\{\frac{D}{2}, \frac{2}{\Delta_Q^2}\right\}$, this implies
\begin{equation*}
    \max\left\{Q(U(n, t)^c), Q(V(n, t)^c), 2Q(W(n, t)^c)\right\} \leq 2\exp\left(-\frac{t}{C\sqrt{B}}\right).
\end{equation*}

\noindent Using this result, as well as a union bound, we can then bound the probability
\begin{align}
    Q(h_t\notin T(Q)) &\leq Q\left(\exists n = (n_1, \dots, n_K): U(n, t)^c \text{ or } V(n, t)^c \text{ or } W(n, t)^c \right) \nonumber \\
    &\leq \sum_{\substack{n = (n_1, \dots, n_K) \\ n_1 + \dots + n_K = t}} Q\left(U(n, t)^c \text{ or } V(n, t)^c \text{ or } W(n, t)^c \right) \nonumber \\
    &\leq \sum_{\substack{n = (n_1, \dots, n_K) \\ n_1 + \dots + n_K = t}} \left\{Q\left(U(n, t)^c\right) + G\left(V(n, t)^c\right) + Q\left(W(n, t)^c \right)\right\} \nonumber \\
    &\leq 5(t+1)^K \exp\left(-\frac{t}{C\sqrt{B}}\right). \label{computation_fixed_t}
\end{align}

\noindent We can now bound the desired probability by first using the decomposition
\begin{equation*}
    Q(\mathcal{H}_{\tau}\notin T(Q)) = Q\left(\tau > \frac{3B}{\Delta_Q}, \mathcal{H}_{\tau}\notin T(Q)\right) + Q\left(\tau \leq \frac{3B}{\Delta_Q}, \mathcal{H}_{\tau}\notin T(Q)\right).
\end{equation*}

\noindent Then, the first term can be easily bounded using Hoeffding's inequality again:
\begin{align*}
    Q\left(\tau > \frac{3B}{\Delta_Q}, \mathcal{H}_{\tau}\notin T(Q)\right) &\leq Q\left(\tau > \frac{3B}{\Delta_Q}\right) \\
    &= Q\left(\forall t\in \left\{1, \dots, \frac{3B}{\Delta_Q}\right\}, \exists n = (n_1, \dots, n_K): \ W(n, t)^c\right) \\
    &\leq Q\left(\exists n = (n_1, \dots, n_K): \ W(n, B)^c\right) \\
    &\leq \sum_{\substack{(n_1, \dots, n_K) \\ n_1 + \dots + n_K = B}} Q\left( W(n, B)^c\right) \\
    &\leq (B+1)^K \exp\left(-\frac{t}{C\sqrt{B}}\right).
\end{align*}

\noindent The second term in the decomposition is decomposed using a union bound and \eqref{computation_fixed_t}:
\begin{align*}
    Q\left(\tau \leq \frac{3B}{\Delta_Q}, \mathcal{H}_{\tau}\notin T(Q)\right) &\leq Q\left(\exists t\in \left\{B, \dots, \frac{3B}{\Delta_Q}\right\}: \ h_t\notin T(Q)\right) \\
    &\leq \sum_{t=B}^{\frac{3B}{\Delta_Q}} Q(h_t\notin T(Q)) \\
    &\leq \sum_{t=B}^{\frac{3B}{\Delta_Q}} 5(t+1)^K \exp\left(-\frac{t}{C\sqrt{B}}\right).
\end{align*}

\noindent Then, for any $t\geq B\geq \left(2KC\right)^2$, it holds that $5(t+1)^K \exp\left(-\frac{t}{C\sqrt{B}}\right) \leq 5\exp\left(-\frac{t}{2C\sqrt{B}}\right)$ and therefore
\begin{align*}
    Q\left(\tau \leq \frac{3B}{\Delta_Q}, \mathcal{H}_{\tau}\notin T(Q)\right) &\leq 5\sum_{t=B}^{\frac{3B}{\Delta_Q}} \exp\left(-\frac{t}{2C\sqrt{B}}\right) \\
    &= 5\times \frac{e^{-\frac{\sqrt{B}}{2C}} - e^{-\frac{\frac{3B}{\Delta_Q} + 1}{2C\sqrt{B}}}}{1 - e^{-\frac{1}{2C\sqrt{B}}}}.
\end{align*}

\noindent We then deduce that, for $B\geq (2KC)^2$,
\begin{align*}
    Q(\mathcal{H}_{\tau}\notin T(Q)) \leq (B+1)^K e^{-\frac{\sqrt{B}}{C}} + 5\times \frac{e^{-\frac{\sqrt{B}}{2C}} - e^{-\frac{\frac{3B}{\Delta_Q} + 1}{2C\sqrt{B}}}}{1 - e^{-\frac{1}{2C\sqrt{B}}}},
\end{align*}
and therefore, that
\begin{equation*}
    \lim_{B\to+\infty} Q(\mathcal{H}_{\tau}\in T(Q)) = 1.
\end{equation*}
$\blacksquare$ \hfill

\subsection{Asymptotic Lower Bound}\label{lower_bound_proba_asymptotic_theorem_section}

The main result of this subsection is the asymptotic lower bound on the probability of ruin. This result will serve as a basis in the proof of the non-asymptotic lower bound of both Theorem~\ref{non_asymptotic_lower_bound_theorem} and Proposition~\ref{non_asymptotic_lower_bound_corollary}, and for that reason, it is conducted in the general case of arm distributions in $\mathcal{F}_{[-1, 1]}^K$.

\begin{theorem}\label{asymptotic_lower_bound_theorem}
Let $(\alpha_k)_{k\in [K]}$ such that for any $k\in [K], \ \alpha_k > 0$ and $\sum_{k=1}^K \alpha_k = 1$. There exists no policy $\pi$ such that, for any set of arms $(F_1, \dots, F_K)$, 
\begin{equation*}
    \liminf_{B\rightarrow +\infty} \frac{1}{B} \log P_{(F_{1}, \dots, F_{K})}\left(\tau(B, \pi)<\infty\right) \leq - \sum_{k=1}^K \alpha_k \inf_{Q_k: \E_{X\sim Q_k}[X]<0}\frac{\KL(Q_k \Vert F_{k})}{\E_{X\sim Q_k}[-X]},
\end{equation*}
with a strict inequality for some $(F_1, \dots, F_K)$.
\end{theorem}

\noindent We define, for any tuple of distributions $F = (F_1, \dots, F_K)$,
\begin{equation*}
    A_F(\alpha_1, \dots, \alpha_K) := \sum_{k=1}^K \alpha_k \inf_{Q_k: \E_{X\sim Q_k}[X]<0}\frac{\KL(Q_k \Vert F_{k})}{\E_{X\sim Q_k}[-X]},
\end{equation*}
which we write $A_F$ in the absence of ambiguity on the $(\alpha_k)_{k\in [K]}$. The inequality in Theorem~\ref{asymptotic_lower_bound_theorem} means that
\begin{equation}
    \liminf_{B\rightarrow +\infty} \sup_{F} \frac{\log P_{(F_{1}, \dots, F_{K})}\left(\tau(B, \pi)<\infty\right)}{B A_F} \geq -1. \label{asymptotic_corrected_inequality}
\end{equation}

\begin{proof}
Recall from Lemma~\ref{fundamental_lemma_proba_ruin} that for any $Q = (Q_1, \dots, Q_K)$ such that $\E_{Q_i}[X]<0$ for any $i\in [K]$,
\begin{equation}
    \liminf_{B \to +\infty} \frac{1}{B} \log P_{(F_1, \dots, F_K)}\left(\tau(B, \pi)<\infty\right) \geq - \sum_{k=1}^K \beta_k(Q) \frac{\KL(Q_k \Vert F_k)}{\E_{X\sim Q_k}[-X]}, \label{asymptotic_lower_bound_basis_inequality}
\end{equation}
where $\beta(Q) = \left(\beta_1(Q), \dots, \beta_K(Q)\right)$ satisfies
\begin{equation}\label{beta_property}
    \forall k\in [K], \beta_k(Q)\geq 0 \ \text{ and } \ \sum_{k=1}^K \beta_k(Q) = 1.
\end{equation}

\noindent Let us fix $(\alpha_k)_{k\in [K]}$ such that for any $k\in [K], \ \alpha_k > 0$ and $\sum_{k=1}^K \alpha_k = 1$. We are going to show that no policy $\pi$ can achieve both
\begin{multline}\label{pi_wide_inequality}
    \forall (F_1, \dots, F_K), \ \liminf_{B\rightarrow +\infty} \frac{1}{B} \log P_{(F_{1}, \dots, F_{K})}\left(\tau(B, \pi)<\infty\right) \\
    \leq - \sum_{k=1}^K \alpha_k\inf_{Q_k: \E_{X\sim Q_k}[X]<0}\frac{\KL(Q_k \Vert F_{k})}{\E_{X\sim Q_k}[-X]} 
\end{multline}
and
\begin{multline}\label{pi_strict_inequality}
    \exists (F_1, \dots, F_K), \ \liminf_{B\rightarrow +\infty} \frac{1}{B} \log P_{(F_{1}, \dots, F_{K})}\left(\tau(B, \pi)<\infty\right) \\
    < - \sum_{k=1}^K \alpha_k \inf_{Q_k: \E_{X\sim Q_k}[X]<0}\frac{\KL(Q_k \Vert F_{k})}{\E_{X\sim Q_k}[-X]}. 
\end{multline}

\noindent Let us then fix a policy $\pi$ such that there exists a distribution $\bar{P} = (\bar{P}_1, \dots, \bar{P}_K)$ and $\epsilon>0$ such that
\begin{equation}
    \liminf_{B\rightarrow +\infty} \frac{1}{B} \log P_{(\bar{P}_{1}, \dots, \bar{P}_{K})}\left(\tau(B, \pi)<\infty\right) \leq - \sum_{k=1}^K \alpha_k \bar{\gamma}_k - \epsilon, \label{p_bar_definition_inequality}
\end{equation}
where we denoted, for any $k\in [K]$,
\begin{equation*}
    \bar{\gamma}_k := \inf_{Q: \E_{X\sim Q}[X]<0} \frac{\KL(Q \Vert \bar{P}_k)}{\E_{X\sim Q}[-X]} \ \text{ and } \ \bar{\gamma}_{\max} := \max_{k\in [K]} \bar{\gamma}_k>0.
\end{equation*}

\noindent Please note that the positivity of $\bar{\gamma}_{\max}$ relies on Assumption~\ref{existence_good_arm_assumption}. We are going to show that there exists $\bar{P}^* = (\bar{P}^*_1, \dots, \bar{P}^*_K)$ such that, denoting
\begin{equation*}
    \bar{\gamma}^*_k := \inf_{Q: \E_{X\sim Q}[X]<0} \frac{\KL(Q \Vert \bar{P}^*_k)}{\E_{X\sim Q}[-X]}, \ \bar{\gamma}^*_{\min} := \min\{\bar{\gamma}^*_k: k\in [K], \bar{\gamma}^*_k>0\} \ \text{ and } \ \epsilon' := \frac{\epsilon\alpha_{\min} \bar{\gamma}^*_{\min}}{4(K-1)\bar{\gamma}_{\max}},
\end{equation*}
the following holds:
\begin{equation*}
    \liminf_{B\rightarrow +\infty} \frac{1}{B} \log P_{(\bar{P}^*_{1}, \dots, \bar{P}^*_{K})}\left(\tau(B, \pi)<\infty\right) \geq - \sum_{k=1}^K \alpha_k \bar{\gamma}^*_k + \epsilon'.
\end{equation*}

\noindent We define $\bar{Q} = (\bar{Q}_1, \dots, \bar{Q}_K)$ such that, for any $k\in [K], \E_{X\sim \bar{Q}_k}[X]<0$ and
\begin{equation}\label{q_bar_basic_property}
    \frac{\KL(\bar{Q}_k \Vert \bar{P}_k)}{\E_{X\sim \bar{Q}_k}[-X]} \leq \bar{\gamma}_k + \frac{\epsilon}{2}. 
\end{equation}

\noindent Denoting $\alpha_{\min} := \min_{k\in [K]} \alpha_k >0$, we then introduce the set 
\begin{equation*}
    \mathcal{K} := \left\{k\in [K]: \beta_k(\bar{Q}) \leq \alpha_k - \frac{\alpha_{\min} \epsilon}{2(K-1) \bar{\gamma}_{\max}}\right\}.
\end{equation*}

\noindent Let us prove that $\mathcal{K}$ is not empty. Indeed, \eqref{p_bar_definition_inequality} can be re-written as
\begin{align}
    \liminf_{B\rightarrow +\infty} \frac{1}{B} \log P_{(\bar{P}_{1}, \dots, \bar{P}_{K})}\left(\tau(B, \pi)<\infty\right) &\leq - \sum_{k=1}^K \alpha_k \bar{\gamma}_k - \epsilon \nonumber \\
    &= -\sum_{k=1}^K \left(\alpha_k \bar{\gamma}_k + \alpha_k \epsilon\right)  \nonumber \\
    &= -\sum_{k=1}^K \left(\alpha_k + \frac{\alpha_k \epsilon}{\bar{\gamma}_k}\right)\bar{\gamma}_k. \label{k_not_full_upper}
\end{align}

\noindent Then, applying \eqref{asymptotic_lower_bound_basis_inequality} to $Q = \bar{Q}$ and $F = \bar{P}$ and using \eqref{q_bar_basic_property}, we deduce that
\begin{align}
    \liminf_{B \to +\infty} \frac{1}{B} \log P_{(\bar{P}_1, \dots, \bar{P}_K)}\left(\tau(B, \pi)<\infty\right) &\geq - \sum_{k=1}^K \beta_k(\bar{Q}) \frac{\KL(\bar{Q}_k \Vert \bar{P}_k)}{\E_{X\sim \bar{Q}_k}[-X]}  \nonumber \\
    &\geq - \sum_{k=1}^K \beta_k(\bar{Q}) \left(\bar{\gamma}_k + \frac{\epsilon}{2}\right)  \nonumber \\
    &= - \sum_{k=1}^K \left(\beta_k(\bar{Q}) + \frac{\alpha_k \epsilon}{2\bar{\gamma}_k}\right)\bar{\gamma}_k. \label{k_not_full_lower}
\end{align}

\noindent Then, we deduce from \eqref{k_not_full_upper} and \eqref{k_not_full_lower} that
\begin{equation*}
    -\sum_{k=1}^K \left(\alpha_k + \frac{\alpha_k \epsilon}{\bar{\gamma}_k}\right)\bar{\gamma}_k \geq - \sum_{k=1}^K \left(\beta_k(\bar{Q}) + \frac{\alpha_k \epsilon}{2\bar{\gamma}_k}\right)\bar{\gamma}_k,
\end{equation*}
or in other words, that 
\begin{equation*}
    \sum_{k=1}^K \left(\beta_k(\bar{Q}) - \alpha_k - \frac{\alpha_k \epsilon}{2\bar{\gamma}_k}\right) \underbrace{\bar{\gamma}_k}_{\geq 0} \geq 0.
\end{equation*}

\noindent This is equivalent to
\begin{equation*}
    \sum_{k: \bar{\gamma}_k>0} \left(\beta_k(\bar{Q}) - \alpha_k - \frac{\alpha_k \epsilon}{2\bar{\gamma}_k}\right) \underbrace{\bar{\gamma}_k}_{> 0} \geq 0.
\end{equation*}

\noindent We then deduce that there exists $k_0\in [K]$ such that $\beta_{k_0}(\bar{Q})\geq \alpha_k +\frac{\alpha_k \epsilon}{2\bar{\gamma}_{k_0}}$. With \eqref{beta_property}, it implies that
\begin{equation*}
    \sum_{j\ne k_0} \beta_j(\bar{Q}) = 1 - \beta_{k_0}(\bar{Q}) \leq 1 - \alpha_{k_0} - \frac{\alpha_{k_0}\epsilon}{2\bar{\gamma}_{k_0}} \leq \sum_{j\ne k_0} \alpha_j - \frac{\alpha_{\min}\epsilon}{2\bar{\gamma}_{\max}}.
\end{equation*}

\noindent We deduce that there exists $j\in [K]$ such that $\beta_j(\bar{Q}) \leq \alpha_j - \frac{\alpha_{\min} \epsilon}{2(K-1)\bar{\gamma}_{\max}}$, proving that $\mathcal{K}$ is not empty. Then, we define the distribution $\bar{P}^* = (\bar{P}^*_1, \dots, \bar{P}^*_K)$ as follows:
\begin{itemize}
    \item[-] for any $k\notin \mathcal{K}$, let $\bar{P}^*_k := \bar{Q}_k$, and please note that $\E_{X\sim \bar{P}^*_k}[X]<0$;
    \item[-] for any $k\in \mathcal{K}$, let $\bar{P}^*_k$ a distribution such that $\E_{X\sim \bar{P}^*_k}[X]>0$ and 
    \begin{equation}
        \frac{\KL(\bar{Q}_k \Vert \bar{P}^*_k)}{\E_{X\sim \bar{Q}_k}[-X]} \leq \bar{\gamma}^*_k\left(1 + \frac{\alpha_{\min} \epsilon}{4(K-1)\bar{\gamma}_{\max}}\right) = \bar{\gamma}^*_k + \frac{\alpha_{\min} \epsilon\bar{\gamma}^*_k}{4(K-1)\bar{\gamma}_{\max}}, \label{q_bar_star_basic_property}
    \end{equation}
    where $\bar{\gamma}^*_k = \inf_{Q: \E_{X\sim Q}[X]<0} \frac{\KL(Q \Vert \bar{P}^*_k)}{\E_{X\sim Q}[-X]}$ and $\bar{\gamma}^*_{\min} = \min\{\bar{\gamma}^*_k: k\in \mathcal{K}, \bar{\gamma}^*_k>0\}$.
    %honda
    Note that this distribution $\bar{P}^*_k$ indeed exists by Lemma~\ref{lem_P_exist}.
\end{itemize}

\noindent Since $\mathcal{K}\ne \emptyset$ and by definition of $\bar{P}^*$, we have
\begin{align*}
    \sum_{k=1}^K \beta_k(\bar{Q}) \frac{\KL(\bar{Q}_k \Vert \bar{P}^*_{k})}{\E_{X\sim \bar{Q}_k}[-X]} &= \sum_{k\notin \mathcal{K}} \beta_k(\bar{Q}) \underbrace{\frac{\KL(\bar{Q}_k \Vert \bar{Q}_{k})}{\E_{X\sim \bar{Q}_k}[-X]}}_{0} + \sum_{k\in \mathcal{K}} \beta_k(\bar{Q}) \frac{\KL(\bar{Q}_k \Vert \bar{P}^*_{k})}{\E_{X\sim \bar{Q}_k}[-X]} \\
    &= \sum_{k\in \mathcal{K}} \beta_k(\bar{Q}) \frac{\KL(\bar{Q}_k \Vert \bar{P}^*_{k})}{\E_{X\sim \bar{Q}_k}[-X]}.
\end{align*}

\noindent Then, \eqref{q_bar_star_basic_property} implies that
\begin{align*}
    \sum_{k\in \mathcal{K}} \beta_k(\bar{Q}) \frac{\KL(\bar{Q}_k \Vert \bar{P}^*_{k})}{\E_{X\sim \bar{Q}_k}[-X]} &\leq \sum_{k\in \mathcal{K}} \beta_k(\bar{Q}) \bar{\gamma}^*_k + \frac{\alpha_{\min} \epsilon}{4(K-1)\bar{\gamma}_{\max}} \sum_{k\in \mathcal{K}} \beta_k(\bar{Q})\bar{\gamma}^*_k \\
    &\leq \sum_{k\in \mathcal{K}} \beta_k(\bar{Q}) \bar{\gamma}^*_k + \frac{\alpha_{\min} \epsilon}{4(K-1)} \sum_{k\in \mathcal{K}} \frac{\bar{\gamma}^*_k}{\bar{\gamma}_{\max}}.
\end{align*}

\noindent By definition of $\mathcal{K}$, for any $k\in \mathcal{K}, \beta_k(\bar{Q})\leq \alpha_k - \frac{\alpha_{\min} \epsilon}{2(K-1)\bar{\gamma}_{\max}}$ and we deduce that
\begin{align}
    \sum_{k\in \mathcal{K}} \beta_k(\bar{Q}) \frac{\KL(\bar{Q}_k \Vert \bar{P}^*_{k})}{\E_{X\sim \bar{Q}_k}[-X]} &\leq \sum_{k\in \mathcal{K}} \alpha_k \bar{\gamma}^*_k - \frac{\alpha_{\min} \epsilon}{2(K-1)} \sum_{k\in \mathcal{K}} \frac{\bar{\gamma}^*_k}{\bar{\gamma}_{\max}} + \frac{\alpha_{\min} \epsilon}{4(K-1)} \sum_{k\in \mathcal{K}} \frac{\bar{\gamma}^*_k}{\bar{\gamma}_{\max}} \nonumber \\
    &\leq \sum_{k\in \mathcal{K}} \alpha_k \bar{\gamma}^*_k - \frac{\alpha_{\min} \epsilon\bar{\gamma}^*_{\min}}{4(K-1)\bar{\gamma}_{\max}}, \label{second_case_p_bar_star_decomposition}
\end{align}
where the inequality \eqref{second_case_p_bar_star_decomposition} comes from the fact that $\mathcal{K}$ is not empty. Injecting \eqref{second_case_p_bar_star_decomposition} in \eqref{asymptotic_lower_bound_basis_inequality} (with $P = P^*$ and $Q = \bar{Q}$), we have:
\begin{equation*}
    \liminf_{B\rightarrow +\infty} \frac{1}{B} \log P_{(\bar{P}^*_{1}, \dots, \bar{P}^*_{K})}\left(\tau(B, \pi)<\infty\right) \geq - \sum_{k=1}^K \alpha_k \bar{\gamma}^*_k + \frac{\alpha_{\min} \epsilon\bar{\gamma}^*_{\min}}{4(K-1)\bar{\gamma}_{\max}}.
\end{equation*}

\noindent Recall that, by definition,  
\begin{equation*}
    \epsilon' = \frac{\epsilon\alpha_{\min}\bar{\gamma}^*_{\min}}{4(K-1)\bar{\gamma}_{\max}}.
\end{equation*}

\noindent We deduce the following 
\begin{equation*}
    \liminf_{B\rightarrow +\infty} \frac{1}{B} \log P_{(\bar{P}^*_{1}, \dots, \bar{P}^*_{K})}\left(\tau(B, \pi)<\infty\right) \geq - \sum_{k=1}^K \alpha_k \bar{\gamma}^*_k + \epsilon',
\end{equation*}
which concludes the proof of the theorem.
\end{proof}

\subsection{Sub-additivity of the optimal log probability of ruin}

The passage from the asymptotic to the non-asymptotic lower bound on the probability of ruin relies on sub-additivity properties described in the next lemma, which is the main result of this subsection.

\begin{lemma}\label{sub_additivity_lemma}
Let $\tilde{t}\in \mathbbm{R}_+^* \cup \{+\infty\}$. For any $B_1, B_2>0$, we have
\begin{multline*}
    \inf_{\pi} \sup_{F} \frac{\log P\left(\tau(B_1+B_2+1, \pi)<\tilde{t}\right)}{A_F} \\
    \leq  \inf_{\pi} \sup_{F} \frac{\log P\left(\tau(B_1, \pi)<\tilde{t}\right)}{A_F} +  \inf_{\pi} \sup_{F} \frac{\log P\left(\tau(B_2, \pi)<\tilde{t}\right)}{A_F}.
\end{multline*}
In the case of multinomial arm distributions of support $\{-1, 0, 1\}$, if $B_1$ and $B_2$ are positive integers, the previous bound can be refined as
\begin{multline*}
    \inf_{\pi} \sup_{F} \frac{\log P\left(\tau(B_1+B_2, \pi)<\tilde{t}\right)}{A_F} \\
    \leq  \inf_{\pi} \sup_{F} \frac{\log P\left(\tau(B_1, \pi)<\tilde{t}\right)}{A_F} +  \inf_{\pi} \sup_{F} \frac{\log P\left(\tau(B_2, \pi)<\tilde{t}\right)}{A_F}.
\end{multline*}
\end{lemma}

\begin{proof}
Let
\begin{align*}
    &\pi_{1}^* \in \argmin_{\pi} \sup_{F} \frac{\log P\left(\tau(B_1, \pi)<\tilde{t}\right)}{A_F}, \\
    &\pi_{2}^* \in \argmin_{\pi} \sup_{F} \frac{\log P\left(\tau(B_2, \pi)<\tilde{t}\right)}{A_F}.
\end{align*}

\noindent Besides, we denote by $\tilde{\pi}$ the policy such that
\begin{equation*}
    \tilde{\pi}_{t} := \begin{cases}
    &(\pi_1^*)_t \text{ if } t<\min\left(\tau(B_1, \tilde{\pi}), \tilde{t}\right), \\
    &(\pi_2^*)_{t-\tau(B_1, \tilde{\pi})} \text{ (ignoring the previously observed rewards) otherwise.}
    \end{cases}
\end{equation*}

\noindent Let $B'_2 \in \{B_2, B_2+1\}$. Then, it is clear that
\begin{align*}
    &P\left(\tau(B_1+B'_2, \tilde{\pi})<\tilde{t}\right) \\
    &= P\left(\exists 1\leq t_{1+2}\leq \tilde{t}: B_1 + B'_2 + \sum_{s=1}^{t_{1+2}} X_s^{\tilde{\pi}_s}<0\right) \\
    &= P\left(\exists 1\leq t_{1}, t_{1+2}\leq \tilde{t}: B_1 + \sum_{s=1}^{t_1} X_s^{\tilde{\pi}_s}<0, \ B_1 + B'_2 + \sum_{s=1}^{t_{1+2}} X_s^{\tilde{\pi}_s}<0\right) \\
    &= P\left(\exists 1\leq t_{1}\leq \tilde{t}: B_1 + \sum_{s=1}^{t_1} X_s^{\tilde{\pi}_s}<0\right) \\
    &\quad \times P\left(\exists \tau(B_1, \tilde{\pi})\leq t_{1+2}\leq \tilde{t}: B_1 + B'_2 + \sum_{s=1}^{t_{1+2}} X_s^{\tilde{\pi}_s}<0 \Bigg| \tau(B_1, \tilde{\pi})<\tilde{t}\right) \\
    &= P(\tau(B_1, \pi^*_1)<\tilde{t}) \times P\left(\exists \tau(B_1, \tilde{\pi})\leq t_{1+2}\leq \tilde{t}: B_1 + B'_2 + \sum_{s=1}^{t_{1+2}} X_s^{\tilde{\pi}_s}<0 \Bigg| \tau(B_1, \pi^*_1)<\tilde{t}\right) \\
    &= P(\tau(B_1, \pi^*_1)<\tilde{t}) \\
    &\quad \times P\Bigg(\exists \tau(B_1, \tilde{\pi})\leq t_{1+2}\leq \tilde{t}: B_1 + \sum_{s=1}^{\tau(B_1, \pi^*_1)} X_s^{(\pi^*_1)_s} + B'_2 + \sum_{s=\tau(B_1, \pi^*_1)+1}^{t_{1+2}} X_s^{(\pi^*_2)_s}<0 \\
    &\quad \quad \quad \quad \quad \quad \quad \quad \quad \quad \quad \quad \quad \quad \quad \quad \quad \quad \quad \quad \quad \quad \quad \quad \quad \quad \quad \quad \quad \quad \Bigg| \tau(B_1, \pi^*_1)<\tilde{t}\Bigg).
\end{align*}

\noindent From there, we are going to study separately the general case and the case of multinomial distributions of support $\{-1, 0, 1\}$.
\paragraph{First case: } in the general case, we choose $B'_2 = B_2 + 1$, and since the rewards are bounded in $[-1, 1]$, then
\begin{equation*}
    \tau(B_1, \pi^*_1)<\tilde{t} \implies B_1 + \sum_{s=1}^{\tau(B_1, \pi^*_1)} X_s^{(\pi^*_1)_s} + 1\geq 0.
\end{equation*}

\noindent Hence,
\begin{align*}
    &P\left(\tau(B_1+B_2+1, \tilde{\pi})<\tilde{t}\right) \\
    &\leq P(\tau(B_1, \pi^*_1)<\tilde{t}) \\
    &\quad \times P\left(\exists \tau(B_1, \pi^*_1)\leq t_{1+2}\leq \tilde{t}: B_2 + \sum_{s=\tau(B_1, \pi^*_1)+1}^{t_{1+2}} X_s^{(\pi^*_2)_s}<0 \Bigg| \tau(B_1, \pi^*_1)<\tilde{t}\right) \\
    &\leq P(\tau(B_1, \pi^*_1)<\tilde{t}) \\
    &\quad \times P\left(\exists \tau(B_1, \pi^*_1)\leq t_{1+2}\leq \tilde{t}+\tau(B_1, \pi^*_1): B_2 + \sum_{s=\tau(B_1, \pi^*_1)+1}^{t_{1+2}} X_s^{(\pi^*_2)_s}<0 \Bigg| \tau(B_1, \pi^*_1)<\tilde{t}\right) \\
    &= P(\tau(B_1, \pi^*_1)<\tilde{t}) \times P(\tau(B_2, \pi^*_2)<\tilde{t}).
\end{align*}

\noindent This yields
\begin{align*}
    &\inf_{\pi} \sup_{F} \frac{\log P\left(\tau(B_1+B_2+1, \pi)<\tilde{t}\right)}{A_F} \\
    &\leq \sup_{F} \frac{\log P\left(\tau(B_1+B_2+1, \Tilde{\pi})<\tilde{t}\right)}{A_F} \\
    &\leq \sup_{F} \frac{\log P\left(\tau(B_1, \pi_1^*)<\tilde{t}\right)}{A_F} + \sup_{F} \frac{\log P\left(\tau(B_2, \pi_2^*)<\tilde{t}\right)}{A_F} \\
    &= \inf_{\pi} \sup_{F} \frac{\log P\left(\tau(B_1, \pi)<\tilde{t}\right)}{A_F} +  \inf_{\pi} \sup_{F} \frac{\log P\left(\tau(B_2, \pi_2^*)<\tilde{t}\right)}{A_F},
\end{align*}
which concludes the general case.

\paragraph{Second case: } in the case of multinomial arm distributions of support $\{-1, 0, 1\}$, 
\begin{equation*}
    \tau(B_1, \pi^*_1)<\tilde{t} \implies B_1 + \sum_{s=1}^{\tau(B_1, \pi^*_1)} X_s^{(\pi^*_1)_s} = 0.
\end{equation*}

\noindent Therefore, choosing $B'_2 = B_2$, we have
\begin{align*}
    &P\left(\tau(B_1+B_2, \tilde{\pi})<\tilde{t}\right) \\
    &\leq P(\tau(B_1, \pi^*_1)<\tilde{t}) \\
    &\quad \times P\left(\exists \tau(B_1, \pi^*_1)\leq t_{1+2}\leq \tilde{t}: B_2 + \sum_{s=\tau(B_1, \pi^*_1)+1}^{t_{1+2}} X_s^{(\pi^*_2)_s}<0 \Bigg| \tau(B_1, \pi^*_1)<\tilde{t}\right) \\
    &\leq P(\tau(B_1, \pi^*_1)<\tilde{t}) \\
    &\quad \times P\left(\exists \tau(B_1, \pi^*_1)\leq t_{1+2}\leq \tilde{t}+\tau(B_1, \pi^*_1) : B_2 + \sum_{s=\tau(B_1, \pi^*_1)+1}^{t_{1+2}} X_s^{(\pi^*_2)_s}<0 \Bigg| \tau(B_1, \pi^*_1)<\tilde{t}\right) \\
    &= P(\tau(B_1, \pi^*_1)<\tilde{t}) \times P(\tau(B_2, \pi^*_2)<\tilde{t}).
\end{align*}

\noindent This yields
\begin{align*}
    &\inf_{\pi} \sup_{F} \frac{\log P\left(\tau(B_1+B_2, \pi)<\tilde{t}\right)}{A_F} \\
    &\leq \sup_{F} \frac{\log P\left(\tau(B_1+B_2, \Tilde{\pi})<\tilde{t}\right)}{A_F} \\
    &\leq \sup_{F} \frac{\log P\left(\tau(B_1, \pi_1^*)<\tilde{t}\right)}{A_F} + \sup_{F} \frac{\log P\left(\tau(B_2, \pi_2^*)<\tilde{t}\right)}{A_F} \\
    &= \inf_{\pi} \sup_{F} \frac{\log P\left(\tau(B_1, \pi)<\tilde{t}\right)}{A_F} +  \inf_{\pi} \sup_{F} \frac{\log P\left(\tau(B_2, \pi_2^*)<\tilde{t}\right)}{A_F},
\end{align*}
which concludes the multinomial case and the proof of the lemma.
\end{proof}

\subsection{Proof of Theorem~\ref{non_asymptotic_lower_bound_theorem} and Proposition~\ref{non_asymptotic_lower_bound_corollary}}

Let $(\alpha_k)_{k\in [K]}$ such that for any $k\in [K], \alpha_k>0$ and $\sum_{k=1}^K \alpha_k = 1$. Recall that, by definition,
\begin{equation*}
    A_{F} = \sum_{k=1}^K \alpha_k \inf_{Q_k: \E_{X\sim Q_k}[X]<0}\frac{\KL(Q_k \Vert F_{k})}{\E_{X\sim Q_k}[-X]}>0.
\end{equation*}

\noindent Let $\pi$ be any policy, and $B_0>0$ an initial budget. For any $n\geq 1$, let us denote by $\pi^B$ the policy defined recursively on $\{B\geq B_0\}$, such that $\pi^{B_0} = \pi$ and for any $B\geq B_0, \pi_t^{B} = \pi_t$ for $t\leq \tau(B_0, \pi)$ and then $\pi_t^{B} = \pi_t^{B'}$ for $t\geq \tau(B_0, \pi)+1$, where $B' = B + \sum_{s=1}^{\tau(B, \pi)}$. Concretely, $\pi^B$ restarts $\pi$ every time it exhausts the budget $B_0$. 

From now, we are going to study separately the general case of rewards bounded in $[-1, 1]$ and the case of multinomial arms of support $\{-1, 0, 1\}$.

\paragraph{First case: } in the case of multinomial arm distributions in $\mathcal{F}_{\{-1, 0, 1\}}^K$, we assume that, for any arm distributions $F = (F_1, \dots, F_K)$,
\begin{equation*}
    \frac{\log P_{F}\left(\tau(B_0, \pi)<\infty\right)}{A_F}  \leq - B_0,
\end{equation*}
and that there exist some arm distributions $\bar{F} = (\bar{F}_1, \dots, \bar{F}_K)$ and $C_{\bar{F}}>0$ such that
\begin{equation*}
    \frac{\log P_{\bar{F}}\left(\tau(B_0, \pi)<\infty\right)}{A_{\bar{F}}}  \leq - (B_0 + C_{\bar{F}}),
\end{equation*}
and we will show that there is contradiction.

By Lemma~\ref{sub_additivity_lemma}, for any arm distributions $F$ and for any $n\geq 1$,
\begin{align*}
    \frac{\log P_F\left(\tau(nB_0, \pi^{nB_0})<\infty\right)}{A_F} &\leq n\times \frac{\log P\left(\tau(B_0, \pi^{B_0})<\infty\right)}{A_F} \\
    &\leq -nB_0.
\end{align*}

\noindent Consequently, for any arm distributions $F$,
\begin{equation}
    \limsup_{n\rightarrow +\infty} \frac{1}{nB_0}\log P\left(\tau(nB_0, \pi)<\infty\right) \leq -A_F. \label{contradiction1}
\end{equation}

\noindent Furthermore, the same computation applied to $\bar{F}$ gives
\begin{align*}
    \frac{\log P_{\bar{F}}\left(\tau(nB_0, \pi^{nB_0})<\infty\right)}{A_{\bar{F}}} &\leq n\times \frac{\log P\left(\tau(B_0, \pi^{B_0})<\infty\right)}{A_{\bar{F}}} \\
    &\leq -n(B_0 + C_{\bar{F}}),
\end{align*}
which, in turn, implies
\begin{equation}
    \limsup_{n\rightarrow +\infty} \frac{1}{nB_0}\log P\left(\tau(nB_0, \pi)<\infty\right) \leq -\frac{B_0+C_{\bar{F}}}{B_0}A_{\bar{F}} < -A_{\bar{F}}. \label{contradiction2}
\end{equation}

\noindent Eqs.~\eqref{contradiction1} and \eqref{contradiction2} contradict Theorem~\ref{asymptotic_lower_bound_theorem}. Therefore, we deduce that if there exist some arm distributions $\bar{F}$ such that 
\begin{equation*}
    \frac{\log P_{\bar{F}}\left(\tau(B_0, \pi)<\infty\right)}{A_{\bar{F}}}  < - B_0,
\end{equation*}
then there also exist some arm distributions $F$ such that
\begin{equation*}
    \frac{\log P_{F}\left(\tau(B_0, \pi)<\infty\right)}{A_{F}}  > - B_0 ,
\end{equation*}
concluding the multinomial case and the proof of Theorem~\ref{non_asymptotic_lower_bound_theorem}.

\paragraph{Second case: } in the general case of arm distributions in $\mathcal{F}_{[-1, 1]}^K$, we assume that, for any arm distributions $F = (F_1, \dots, F_K)$,
\begin{equation*}
    \frac{\log P_{F}\left(\tau(B_0, \pi)<\infty\right)}{A_F}  \leq - (B_0+1),
\end{equation*}
and that there exist some arm distributions $\bar{F} = (\bar{F}_1, \dots, \bar{F}_K)$ and $C_{\bar{F}}>0$ such that
\begin{equation*}
    \frac{\log P_{\bar{F}}\left(\tau(B_0, \pi)<\infty\right)}{A_{\bar{F}}}  \leq - (B_0 + 1 + C_{\bar{F}}),
\end{equation*}
and show that there is contradiction.

By Lemma~\ref{sub_additivity_lemma}, for any arm distributions $F$ and for any $n\geq 1$,
\begin{align*}
    \frac{\log P_F\left(\tau(nB_0 + (n-1), \pi^{nB_0})<\infty\right)}{A_F} &\leq n\times \frac{\log P\left(\tau(B_0, \pi^{B_0})<\infty\right)}{A_F} \\
    &\leq -n(B_0+1).
\end{align*}

\noindent Consequently, for any arm distributions $F$,
\begin{equation}
    \limsup_{n\rightarrow +\infty} \frac{1}{nB_0 + (n-1)}\log P\left(\tau(nB_0 + (n-1), \pi)<\infty\right) \leq -A_F. \label{contradiction3}
\end{equation}

\noindent Furthermore, the same computation applied to $\bar{F}$ gives
\begin{align*}
    \frac{\log P_{\bar{F}}\left(\tau(nB_0+(n-1), \pi^{nB_0})<\infty\right)}{A_{\bar{F}}} &\leq n\times \frac{\log P\left(\tau(B_0, \pi^{B_0})<\infty\right)}{A_{\bar{F}}} \\
    &\leq -n(B_0 + 1 + C_{\bar{F}}),
\end{align*}
which in turn, implies,
\begin{equation}
    \limsup_{n\rightarrow +\infty} \frac{1}{nB_0+(n-1)}\log P\left(\tau(nB_0, \pi)<\infty\right) \leq -\frac{B_0+1+C_{\bar{F}}}{B_0+1} < -1. \label{contradiction4}
\end{equation}

\noindent Eqs.~\eqref{contradiction3} and \eqref{contradiction4} contradict Theorem~\ref{asymptotic_lower_bound_theorem}. Therefore, we deduce that if there exist some arm distributions $\bar{F}$ such that 
\begin{equation*}
    \frac{\log P_{\bar{F}}\left(\tau(B_0, \pi)<\infty\right)}{A_{\bar{F}}}  < - (B_0+1),
\end{equation*}
then there also exist some arm distributions $F$ such that
\begin{equation*}
    \frac{\log P_{F}\left(\tau(B_0, \pi)<\infty\right)}{A_{F}}  > - (B_0+1) ,
\end{equation*}
concluding the general case and the proof of Proposition~\ref{non_asymptotic_lower_bound_corollary}.
\hfill $\blacksquare$

\section{Proof of Lemma~\ref{proba_lemma}}\label{proba_lemma_appendix}

Please note that applying \eqref{asymptotic_lower_bound_basis_inequality} to the case of one single arm ($K = 1$) gives that, for any $\epsilon_0 \in \left(0, \frac{1}{3}\right)$ and any distribution $Q$ which has a negative expectation,
$$\liminf_{B\to +\infty} \frac{1}{B} \log P(\tau(B, 1)<\infty) \geq -\frac{\KL(Q \Vert F_1)}{\E_{X\sim Q}[-X]}.$$

\noindent By taking $\epsilon_0 \downarrow 0$, we deduce that
$$\liminf_{B\to +\infty} \frac{1}{B} \log P(\tau(B, 1)<\infty) \geq -\inf_{F: \E_{X\sim F}[X]<0} \frac{\KL(F \Vert F_1)}{\E_{X\sim F}[-X]}.$$

\noindent It thus remains to prove that for any $B>0$,
$$\frac{1}{B} \log P(\tau(B, 1)<\infty) \leq -\inf_{F: \E_{X\sim F}[X]<0} \frac{\KL(F \Vert F_1)}{\E_{X\sim F}[-X]}.$$

\noindent The result being trivial if $\E_{X\sim F_1}[X]\leq 0$, we assume that $\E_{X\sim F_1}[X]> 0$.
Let us define the logarithmic moment-generating function of $X$ by $\Lambda(\lambda) := \log \E\left[e^{\lambda X}\right]$.
By Lemma~\ref{lem_kl_lambda}, there exists $\lambda'<0$ such that
$\Lambda(\lambda')=0$
and it satisfies
\begin{equation*}
 \inf_{F: \E_{X\sim F}[X]<0} \frac{\KL(F\Vert F_1)}{\E_{X\sim F}[-X]}
 =-\lambda'.
\end{equation*}

\noindent Now, let $X\sim F_1$ and let $X_1, X_2,\dots \sim F_1$ be i.i.d copies of $X$.
We write $S_n = \sum_{i=1}^n X_i$.
We define $\tau := \inf\{n\geq 1: S_n\leq -B\}$ and $\tau_T := \min(\tau, T)$ for any $T\in \mathbb{N}$. Since $\tau_T$ is a bounded stopping time, by the optional stopping theorem, it holds for any $T$ that
$$\E\left[e^{\lambda' S_{\tau_T}}\right] = 1.$$

\noindent On the other hand, 
\begin{align*}
    1 &= \E\left[e^{\lambda' S_{\tau_T}}\right] \\
    &= \E\left[\1_{\tau_T<T}e^{\lambda' S_{\tau_T}}\right] + \E\left[\1_{\tau_T=T}e^{\lambda' S_{\tau_T}}\right] \\
    &\geq \E\left[\1_{\tau_T<T}e^{\lambda' S_{\tau_T}}\right] \\
    &\geq e^{-\lambda' B} P(\tau_T<T),
\end{align*}
which implies that
\begin{equation}
    P(\tau<\infty) = \lim_{T\to \infty} P(\tau_T<T) \leq e^{\lambda' B}.
    \nonumber
    %\label{inequality2_proba_lemma_proof}
\end{equation}
Therefore we obtain
\begin{equation*}
    \frac{1}{B} \log P(\tau<\infty) \leq \lambda' = -\inf_{F: \E_{X\sim F}[X]<0} \frac{\KL(F\Vert F_1)}{\E_{X\sim F}[-X]},
\end{equation*}
which completes the proof. 
\hfill $\blacksquare$

\section{Proof of the Upper Bound on the Reward of EXPLOIT Policies (Proposition~\ref{exploit_reward_upper_bound_proposition})}\label{exploit_reward_upper_bound_proposition_proof}

We introduce the following notation. Let $K_+$ be the number of arms such that $P\left(\tau\left(\frac{B}{K}, k\right) = \infty\right)>0$ for $k\in [K]$. Recall that we ordered the arms in order of decreasing expectation, and therefore $\mu_1\geq \dots \geq \mu_{K_+}>0$. Then, by definition of $K_+$,
\begin{equation*}
    \forall k\in [K_+], P\left(\tau\left(\frac{B}{K}, k\right) = \infty\right)>0 \ ; \ \forall j\in \left\{K_+ +1,\dots , K\right\}, P\left(\tau\left(\frac{B}{K}, j\right) = \infty\right)=0.
\end{equation*}

\subsection{Preliminary Lemma}

In this subsection, we express the expected cumulative reward of any policy in EXPLOIT as the product between $p^{\mathrm{EX}}$ and a convex combination of $\mu_1, \dots, \mu_{K_+}$, and we explicitly give the coefficients of the convex combination. This decomposition will be useful both as a first step in the proof of Proposition~\ref{exploit_reward_upper_bound_proposition} and in the proof of the reward bound of EXPLOIT-UCB, as a particular instance of policy in EXPLOIT.

\begin{comment}
Recall that
\begin{equation*}
    \tau_k^{<} := \inf\left\{t\geq 0: \sum_{s=1}^t Y_s^k < -\frac{B}{K} + 1\right\},
\end{equation*}
where we denoted by $Y_1^k, Y_2^k, \dots$ the rewards of any arm $k\in [K]$. Please note that for any $k\in[K]$,
$$\tau_k^{<} \leq \tau\left(\frac{B}{K}, k\right).$$
\end{comment}

For any $S\subseteq [K_+]$, we define the event
\begin{equation*}
    \Pi_{S} := \left\{\forall j\in S, \tau\left(\frac{B}{K}, j\right)\geq T \ \text{ and } \ \forall j\in [K_+]\setminus S, \tau\left(\frac{B}{K}, j\right)< \sqrt{T}\right\}.
\end{equation*}

\noindent Given a policy $\pi$, an arm $k\in [K_+]$ and a set $S\subseteq [K_+]$, we define the coefficient $n_{\pi, k}(S)$ as
\begin{equation*}
    n_{\pi, k}(S) := \frac{1}{T} \E\left[\sum_{t=1}^T \1_{\pi_t = k, \tau(B, \pi)\geq t-1} \Bigg| \Pi_S\right].
\end{equation*}

\noindent When there is no ambiguity on the policy $\pi$, we simply write $n_k(S)$. Please note that for any fixed policy $\pi$ and any set $S\subseteq [K_+]$,
\begin{equation*}
    \sum_{k\in S} n_k(S) \leq \sum_{k=1}^K n_k(S) = \frac{1}{T} \E\left[\sum_{t=1}^T \1_{\tau(B, \pi)\geq t-1} \Bigg| \Pi_S\right] \leq 1.
\end{equation*}

\noindent The following result holds.
\begin{lemma}\label{exploit_reward_lemma}
For any policy $\pi$ within the framework EXPLOIT, the expected cumulative reward satisfies
\begin{equation}\label{exploit_reward}
    \E\left[\sum_{t=1}^T X_t^{\pi_t}\1_{\tau(B, \pi)\geq t-1}\right] = \sum_{k=1}^{K_+} \left(\sum_{S\subseteq [K_+]: k\in S} P(\Pi_S) n_k(S)\right) \mu_k \times T + o(T).
\end{equation}
\end{lemma}

\begin{proof}
First, we write the reward as a sum over the arms of positive probability of survival:
\begin{align}
    \E\left[\sum_{t=1}^T X_t^{\pi_t}\1_{\tau(B, \pi)\geq t-1}\right] &= \sum_{k=1}^K \mu_k \E\left[\sum_{t=1}^T \1_{\pi_t=k}\1_{\tau(B, \pi)\geq t-1}\right] \nonumber \\
    &= \sum_{k=1}^{K_+} \mu_k \E\left[\sum_{t=1}^T \1_{\pi_t=k}\1_{\tau(B, \pi)\geq t-1}\right] + o(T). \label{base_eq}
\end{align}

\noindent Then, we examine the term $\E\left[\sum_{t=1}^T \1_{\pi_t=k}\1_{\tau(B, \pi)\geq t-1}\right]$. In order to analyse it, we will introduce the following events, for any $S, S'\subseteq [K_+]$ such that $S\cap S' = \emptyset$:
\begin{multline*}
    \Pi_{S, S'} := \bigg\{\forall j\in S, \tau\left(\frac{B}{K}, j\right)\geq T ; \ \forall j\in S', \sqrt{T}\leq \tau\left(\frac{B}{K}, j\right)< T ; \\
    \forall j\in [K_+]\setminus (S\cup S'),\tau\left(\frac{B}{K}, j\right)< \sqrt{T}\bigg\}.
\end{multline*}

\noindent Please note that $\Pi_{S} = \Pi_{S, \emptyset}$. We can then decompose, for any $k\in\{1, \dots, K_+\}$,
\begin{equation*}
    \E\left[\sum_{t=1}^T \1_{\pi_t=k}\1_{\tau(B, \pi)\geq t-1}\right] = \sum_{S, S'\subseteq [K_+]: S\cap S' = \emptyset} P\left(\Pi_{S, S'}\right) \E\left[\sum_{t=1}^T \1_{\pi_t = k} \1_{\tau(B, \pi)\geq t-1} \Bigg| \Pi_{S, S'}\right].
\end{equation*}

\noindent Consider the case $S'\ne \emptyset$ and let $k\in S'$. We can bound
\begin{align*}
    P(\Pi_{S, S'}) &\leq P\left(\sqrt{T}\leq \tau\left(\frac{B}{K}, k\right)< T\right) \\
    &= P\left(\tau\left(\frac{B}{K}, k\right)\geq \sqrt{T}\right) - P\left(\tau\left(\frac{B}{K}, k\right)\geq T\right).
\end{align*}

\noindent Indeed, the sequence $\left(P\left(\tau\left(\frac{B}{K}, k\right)\geq n\right)\right)_{n\geq 1}$ is increasing and upper-bounded by $1$, and thus it has a limit and it implies that
$$P\left(\tau\left(\frac{B}{K}, k\right)\geq \sqrt{T}\right) - P\left(\tau\left(\frac{B}{K}, k\right)\geq T\right) = o(1).$$

\noindent We deduce that, for any $S, S'\subseteq [K_+]$ such that $S\cap S' = \emptyset$, 
\begin{equation*}
    S'\ne \emptyset \implies P\left(\Pi_{S, S'}\right) = o(1).
\end{equation*}

\noindent This implies
\begin{equation*}
    \E\left[\sum_{t=1}^T \1_{\pi_t=k}\1_{\tau(B, \pi)\geq t-1}\right] = \sum_{S\subseteq [K_+]} P\left(\Pi_{S}\right) \E\left[\sum_{t=1}^T \1_{\pi_t = k} \1_{\tau(B, \pi)\geq t-1} \Bigg| \Pi_{S}\right] + o(T).
\end{equation*}

\noindent Re-injecting in \eqref{base_eq}, we have
\begin{align*}
    \E\left[\sum_{t=1}^T X_t^{\pi_t}\1_{\tau(B, \pi)\geq t-1}\right] &= \sum_{k=1}^{K_+} \mu_k \sum_{S\subseteq [K_+]} P\left(\Pi_{S}\right) \E\left[\sum_{t=1}^T \1_{\pi_t = k} \1_{\tau(B, \pi)\geq t-1} \Bigg| \Pi_{S}\right] + o(T) \\
    &= \sum_{S\subseteq [K_+]} P\left(\Pi_{S}\right) \sum_{k=1}^{K_+} \mu_k \E\left[\sum_{t=1}^T \1_{\pi_t = k} \1_{\tau(B, \pi)\geq t-1} \Bigg| \Pi_{S}\right] + o(T).
\end{align*}

\noindent Besides, it is clear, by definition of $\Pi_S$, that any policy $\pi$ in EXPLOIT satisfies
\begin{equation*}
    k\notin S \implies \E\left[\sum_{t=1}^T \1_{\pi_t = k} \1_{\tau(B, \pi)\geq t-1} \Bigg| \Pi_{S}\right] = o(T).
\end{equation*}

\noindent We deduce that
\begin{align*}
    \E\left[\sum_{t=1}^T X_t^{\pi_t}\1_{\tau(B, \pi)\geq t-1}\right] &= \sum_{S\subseteq [K_+]} P\left(\Pi_{S}\right) \sum_{k\in S} \mu_k \E\left[\sum_{t=1}^T \1_{\pi_t = k} \1_{\tau(B, \pi)\geq t-1} \Bigg| \Pi_{S}\right] + o(T) \\
    &= \sum_{S\subseteq [K_+]} P\left(\Pi_{S}\right) \sum_{k\in S} \mu_k n_k(S) T + o(T) \\
    &= \sum_{k=1}^{K_+} \left(\sum_{S\subseteq [K_+]: k\in S} P\left(\Pi_{S}\right) n_k(S)\right) \mu_k T + o(T), 
\end{align*}
which concludes the proof of the lemma.
\end{proof}

\subsection{Proof of Proposition~\ref{exploit_reward_upper_bound_proposition}}

It remains to provide an upper bound to the right-hand side of the equality~\eqref{exploit_reward} in Lemma~\ref{exploit_reward_lemma} in order to complete the proof of Proposition~\ref{exploit_reward_upper_bound_proposition}. In order to maximize the right term in \eqref{exploit_reward}, we should solve the following maximization problem:
\begin{equation*}
    \max \sum_{k=1}^{K_+} \left(\sum_{S\subseteq [K_+]: k\in S} P\left(\Pi_{S}\right) n_k(S)\right) \mu_k \quad \text{ s.t. } \quad \sum_{k\in S} n_k(S) \leq 1 \text{ and } \forall k\notin S, n_k(S) = 0.
\end{equation*}

\noindent We can re-order the terms in the above sum as
\begin{equation*}
    \sum_{k=1}^{K_+} \left(\sum_{S\subseteq [K_+]: k\in S} P\left(\Pi_{S}\right) n_k(S)\right) \mu_k = \sum_{S\subseteq [K_+]} P\left(\Pi_{S}\right) \sum_{k\in S} \mu_k n_k(S),
\end{equation*}
and since, by hypothesis, $\mu_1\geq \mu_2 \geq \dots \geq \mu_{K_+}>0$, we deduce the bound
\begin{equation*}
    \sum_{k=1}^{K_+} \left(\sum_{S\subseteq [K_+]: k\in S} P\left(\Pi_{S}\right) n_k(S)\right) \mu_k \leq \sum_{S\subseteq [K_+]} P\left(\Pi_{S}\right) \max_{k\in S} \mu_k,
\end{equation*}
where the above inequality is an equality for $n_k(S) = n^*_k(S)$, defined by
\begin{equation*}
    n^*_k(S) = \left\{
    \begin{array}{ll}
        1 &\text{ if } k = \min S \\
        0 &\text{ otherwise.} \\
    \end{array}
\right.
\end{equation*}

\noindent This gives, for any given $k\in \{1, \dots, K_+\}$,
\begin{align*}
    &\sum_{S\subseteq [K_+]: k\in S} P\left(\Pi_{S}\right) n_k^*(S) \\
    &= \sum_{S\subseteq \{k+1, \dots, K_+\}} P\left(\Pi_{\{k\}\cup S}\right) \\
    &= \sum_{S\subseteq \{k+1, \dots, K_+\}} P\bigg(\forall j\in \{k\}\cup S, \tau\left(\frac{B}{K}, j\right)\geq T; \ \forall j\in [K_+]\setminus (\{k\}\cup S), \tau\left(\frac{B}{K}, j\right)< \sqrt{T}\bigg) \\
    &= P\left(\forall j\in [k-1], \tau\left(\frac{B}{K}, j\right)< \sqrt{T}; \ \tau\left(\frac{B}{K}, k\right)\geq T\right) + o(1) \\
    &= \left(1 - p^{\mathrm{EX}}\right) \underbrace{\frac{1}{1 - p^{\mathrm{EX}}}\prod_{j=1}^{k-1} P\left(\tau\left(\frac{B}{K}, j\right)< \infty\right)P\left(\tau\left(\frac{B}{K}, k\right)=\infty\right)}_{w_k} + o(1).
\end{align*}

\noindent We deduce that
\begin{equation*}
    \E\left[\sum_{t=1}^T X_t^{\pi_t}\1_{\tau(B, \pi)\geq t-1}\right] \leq \left(1 - p^{\text{EX}}\right) \sum_{k=1}^{K_+} w_k \mu_k T + o(T),
\end{equation*}
which concludes the proof of the proposition. 
\hfill $\blacksquare$

\section{Proof of the Reward Bound of EXPLOIT-UCB (Proposition~\ref{exploit_bandit_reward_optimal})}\label{exploit_bandit_reward_optimal_proof}

In this appendix and in the next one, we will use the following notation. For any $k\in [K]$, let $Y_1^k, \dots, Y_T^k \sim F_k$ be i.i.d.~rewards drawn from arm $k$ and for any $t\geq 1$, and we denote by $N_k(t) := \sum_{s=1}^t \1_{\pi_s = k}$ the number of times arm $k$ has been pulled until round $t$. Please note that $X_t^{\pi_t} = \sum_{k=1}^K Y_{N_k(t)}^k \1_{\pi_t = k}$. Given that arm $k$ has been pulled $n_k$ times, we also introduce $\hat{Y}_{n_k}^k := \frac{1}{n_k}\sum_{n=1}^{n_k} Y_n^{k}$ as the empirical average of arm $k$ at round $t$. Please note that for any $k\in [K]$ and any round $t$,
$$\hat{Y}^k_{N_k(t)} = \hat{X}^k_t.$$

\noindent For any $t\geq 1$, let
\begin{equation*}
    C_k(t) := \hat{Y}_{N_k(t)}^k + \sqrt{\frac{6 \log (t)}{N_k(t)}}.
\end{equation*}

\subsection{Preliminary Lemma}

EXPLOIT-UCB is based on the classic bandit algorithm UCB1 (which has a sublinear regret in the classic stochastic MAB), and therefore has the following characteristic which is going to be useful in the proof of the reward bound of both EXPLOIT-UCB and EXPLOIT-UCB-DOUBLE. For the sake of clarity, we assume the arms are ordered in decreasing expectation: $\mu_1\geq \mu_2 \geq \dots \geq \mu_K$.

\begin{lemma}\label{classic_bandit_regret_bound}
Let $(\pi_t)_{t\geq 1}$ be the policy associated to EXPLOIT-UCB. Assume that $\mu_1>\mu_k$. Then
$$\E\left[\sum_{t=1}^T \1_{\pi_t = k, C_k(t)\geq C_1(t)}\right] = o(T).$$
\end{lemma}

\begin{proof}
This proof completely follows the proof of Theorem 1 in \cite{auer}. Let $r\leq T$, the quantity to bound can be written as
\begin{align*}
    &\E\left[\sum_{t=1}^T \1_{\pi_t = k, C_k(t) \geq C_1(t)}\right] \\
    &= \E\left[\sum_{t=1}^T \1_{\pi_t = k, \ \hat{Y}_{N_k(t)}^k(t) + \sqrt{\frac{6 \log t}{N_k(t)}} \geq \hat{Y}_{N_1(t)}^1(t) + \sqrt{\frac{6 \log t}{N_1(t)}}}\right] \\
    &= \E\left[\sum_{t=T^{1/2}}^T \1_{\pi_t = k, \ \hat{Y}_{N_k(t)}^k(t) + \sqrt{\frac{6 \log t}{N_k(t)}} \geq \hat{Y}_{N_1(t)}^1(t) + \sqrt{\frac{6 \log t}{N_1(t)}}}\right] + o(T) \\
    &\leq r + \E\left[\sum_{t\in \Delta_{T, r}} \1_{\pi_t = k, \ \max_{r\leq n_k<t}\hat{Y}_{n_k}^k + \sqrt{\frac{6 \log t}{n_k}} \geq \min_{1\leq n_1<t}\hat{Y}_{n_1}^1 + \sqrt{\frac{6 \log t}{n_1}}}\right] + o(T) \\
    &\leq r + \E\left[\sum_{t\in \Delta_{T, r}} \1_{\max_{r\leq n_k<t}\hat{Y}_{n_k}^k + \sqrt{\frac{6 \log t}{n_k}} \geq \min_{1\leq n_1<t}\hat{Y}_{n_1}^1 + \sqrt{\frac{6 \log t}{n_1}}}\right] + o(T) \\
    &\leq r + \E\left[\sum_{t\in \Delta_{T, r}}\sum_{n_1 = 1}^{t-1} \sum_{n_k = r}^{t-1} \1_{\hat{Y}_{n_k}^k + \sqrt{\frac{6 \log t}{n_k}} \geq \hat{Y}_{n_1}^1 + \sqrt{\frac{6 \log t}{n_1}}}\right] + o(T),
\end{align*}
where $\Delta_{T, r} = \left\{t\in \{1, \dots, T\}: \sum_{s=1}^t \1_{\pi_s = k}\geq r\right\}$. 
%$\Delta_{T, r} = \left\{t\in \{T^{\frac{1}{2}}, \dots, T\}: \sum_{s=1}^t \1_{\pi_s = k}\geq r\right\}$. 
Similarly as in \cite{auer}, we use the fact that the probability event $\left\{\hat{Y}_{n_k}^k + \sqrt{\frac{6 \log t}{n_k}} \geq \hat{Y}_{n_1}^1 + \sqrt{\frac{6 \log t}{n_1}}\right\}$ implies at least one of the following:
\begin{align*}
    \hat{Y}_{n_1}^1 &\leq \mu_1 - \sqrt{\frac{6 \log t}{n_1}} \\
    \hat{Y}_{n_k}^k &\geq \mu_k + \sqrt{\frac{6 \log t}{n_k}} \\
    \mu_1 &< \mu_k + 2\sqrt{\frac{6 \log t}{n_k}}. 
\end{align*}

\noindent Therefore, we can write
\begin{align*}
    &\E\left[\sum_{t=1}^T \1_{\pi_t = k, C_k(t) \geq C_1(t)}\right] \\
    &\leq r + \E\left[\sum_{t\in \Delta_{T, r}}\sum_{n_1 = 1}^{t-1} \sum_{n_k = r}^{t-1} \1_{\hat{Y}_{n_k}^k + \sqrt{\frac{6 \log t}{n_k}} \geq \hat{Y}_{n_1}^1 + \sqrt{\frac{6 \log t}{n_1}}}\right] + o(T) \\
    &\leq r + \E\left[\sum_{t\in \Delta_{T, r}}\sum_{n_1 = 1}^{t-1} \sum_{n_k = r}^{t-1} \left(\1_{\hat{Y}_{n_1}^1 \leq \mu_1 - \sqrt{\frac{6\log T}{n_1}}} + \1_{\hat{Y}_{n_k}^k \geq \mu_k + \sqrt{\frac{6\log T}{n_k}}} + \1_{\mu_1 < \mu_k + 2\sqrt{\frac{6\log T}{n_k}}}\right)\right] \\
    &\quad \quad \quad \quad \quad \quad \quad \quad \quad \quad \quad \quad \quad \quad \quad \quad \quad \quad \quad \quad \quad \quad \quad \quad \quad \quad \quad \quad \quad \quad \quad \quad \quad \quad \quad + o(T).
\end{align*}
%where the last inequality comes from the fact that for any $t\in \Delta_{T, r}, t\geq T^{1/2}$. 
The choice $r = \left\lceil\frac{24 \log T}{(\mu_1 - \mu_k)^2}\right\rceil$ ensures that, for any $n_k\geq r$,
$$\mu_1 - \mu_k - 2\sqrt{\frac{6\log T}{n_k}} \geq 0,$$
which implies
\begin{align}
    &\E\left[\sum_{t=1}^T \1_{\pi_t = k, C_k(t) \geq C_1(t)}\right] \nonumber \\
    &\leq \left\lceil\frac{24 \log T}{(\mu_1 - \mu_k)^2}\right\rceil + \E\left[\sum_{t\in \Delta_{T, r}} \sum_{n_1 = 1}^{t-1} \sum_{n_k = r}^{t-1} \left(\1_{\hat{Y}_{n_1}^1 \leq \mu_1 - \sqrt{\frac{6\log T}{n_1}}} + \1_{\hat{Y}_{n_k}^k \geq \mu_k + \sqrt{\frac{6\log T}{n_k}}}\right)\right] + o(T) \nonumber \\
    &= T\times \sum_{n_1 = 1}^{T} \sum_{n_k = r}^{T} \E\left[\left(\1_{\hat{Y}_{n_1}^1 \leq \mu_1 - \sqrt{\frac{6\log T}{n_1}}} + \1_{\hat{Y}_{n_k}^k \geq \mu_k + \sqrt{\frac{6\log T}{n_k}}}\right)\right] + o(T) \nonumber \\
    &= T\times \sum_{n_1 = 1}^{T} \sum_{n_k = r}^{T} \left(P\left(\hat{Y}_{n_1}^1 \leq \mu_1 - \sqrt{\frac{6\log T}{n_1}}\right) + P\left(\hat{Y}_{n_k}^k \geq \mu_k + \sqrt{\frac{6\log T}{n_k}}\right)\right) + o(T). \label{UCB_base_inequality}
\end{align}

\noindent Using Hoeffding's inequality, for any $n_1\in \{1, \dots, T\}$, we have
$$P\left(\hat{Y}_{n_1}^1 \leq \mu_1 - \sqrt{\frac{6\log T}{n_1}}\right) \leq \frac{1}{T^3}.$$

\noindent Similarly, for any $n_k\in \{r, \dots, T\}$, we have
$$P\left(\hat{Y}_{n_k}^k \geq \mu_k + \sqrt{\frac{6\log T}{n_k}}\right) \leq \frac{1}{T^3}.$$

\noindent We can then replace in \eqref{UCB_base_inequality}:
\begin{align*}
    \E\left[\sum_{t = 1}^T \1_{\pi_t = k, C_k(t) \geq C_1(t)}\right] &\leq T \sum_{n_1 = 1}^{T} \sum_{n_k = r}^{T} \frac{2}{T^3} + o(T) \\
    &\leq 2 + o(T) \\
    &= o(T),
\end{align*}
which concludes the proof of the lemma.
\end{proof}

\subsection{Proof of Proposition~\ref{exploit_bandit_reward_optimal}}

Let $S\subseteq [K_+]$. Recall that for any arm $k\in [K_+]$,
$$n_k(S) = \frac{1}{T} \E\left[\sum_{t=1}^T \1_{\pi_t = k, \tau(B, \pi)\geq t-1} \Bigg| \Pi_S\right].$$

\noindent Given that EXPLOIT-UCB is in EXPLOIT, by Lemma~\ref{exploit_reward_lemma}, it suffices to show that for any $S\subseteq [K_+]$ and any $k\in S$,
$$n_k(S) = \left\{
    \begin{array}{ll}
        1 + o(1) &\text{ if } k = \min S \\
        o(1) &\text{ otherwise.} \\
    \end{array}
\right.$$ 

\noindent Actually, by Proposition~\ref{exploit_reward_upper_bound_proposition}, it is sufficient to prove the above property for any $S\subseteq [K_+]$ such that there is no positive or zero arm $k\in [K_+]\setminus S$. Let then $S\subseteq [K_+]$ satisfying such a property.

Then, please note that $\sum_{k=1}^K n_k(S) = 1$. Therefore, it suffices to prove that for any $k\in S\setminus \{\min S\}, n_k(S) = o(1).$ Thus, let $k\in S\setminus \{\min S\}$. Then, on the one hand, we are going to provide a lower bound $P(\Pi_S)$ which is independent of $T$. Indeed, since there is no positive or zero arm $k\in [K_+]\setminus S$, we know that there exists $\epsilon>0$ such that
$$\forall k\in [K_+], P_{X\sim F_k}(X\leq -\epsilon)>0.$$

\noindent We fix such an $\epsilon$ and we deduce that
$$\prod_{k\in [K_+]\setminus S} P\left(\tau\left(\frac{B}{K}, k\right)\leq \frac{B}{\epsilon K}\right) >0.$$

\noindent We can therefore provide the following lower bound, independent of $T$ and positive by definition of $K_+$. For any $T\geq \left(\frac{B}{\epsilon K}\right)^2$,
\begin{align*}
    P(\Pi_S) &= \prod_{k\in S} P\left(\tau\left(\frac{B}{K}, k\right)\geq T\right) \prod_{k\in [K_+]\setminus S} P\left(\tau\left(\frac{B}{K}, k\right)<\sqrt{T}\right) \\
    &\geq \prod_{k\in S} P\left(\tau\left(\frac{B}{K}, k\right)=\infty\right) \prod_{k\in [K_+]\setminus S} P\left(\tau\left(\frac{B}{K}, k\right)<\frac{B}{\epsilon K}\right) \\
    &>0.
\end{align*}

\noindent On the other hand, an upper bound to $\E\left[\sum_{t=1}^T \1_{\pi_t = k, \Pi_S}\right]$ is obtained by 
\begin{align*}
    \E\left[\sum_{t=1}^T \1_{\pi_t = k, \Pi_S}\right] &\leq \E\left[\sum_{t=1}^T \1_{\pi_t = k, \forall j\in S, \tau\left(\frac{B}{K}, j\right)\geq T}\right] \\
    &\leq \E\left[\sum_{t=1}^T \1_{\pi_t = k, C_k(t-1)\geq C_{\min S}(t-1)}\right] \\
    &= o(T)
\end{align*}
by Lemma~\ref{classic_bandit_regret_bound}. We deduce the following bound on $n_k(S)$:
\begin{align*}
    n_k(S) &= \frac{1}{T} \frac{\E\left[\sum_{t=1}^T \1_{\pi_t = k, \Pi_S}\right]}{P(\Pi_S)} \\
    &\leq \frac{1}{T} \frac{\E\left[\sum_{t=1}^T \1_{\pi_t = k, C_k(t-1)\geq C_{\min S}(t-1)}\right]}{\prod_{k\in S} P\left(\tau\left(\frac{B}{K}, k\right)=\infty\right) \prod_{k\in [K_+]\setminus S} P\left(\tau\left(\frac{B}{K}, k\right)<\frac{B}{\epsilon K}\right)} \\
    &= o(1),
\end{align*}
which concludes the proof of the proposition. 
\hfill $\blacksquare$

\section{The Performance of EXPLOIT-UCB-DOUBLE}

In the proofs of the performance of EXPLOIT-UCB-DOUBLE, for the sake of clarity, we drop the exponent $n$ in the notation of EXPLOIT-UCB-DOUBLE, and $\pi^n$ becomes $\pi$. Recall the following notation from the previous section, for any $t\geq 1$,
\begin{equation*}
    C_k(t) := \hat{Y}_{N_k(t)}^k + \sqrt{\frac{6 \log (t)}{N_k(t)}}.
\end{equation*}

\subsection{Proof of Proposition~\ref{proba_survival_exploit_dbl}}\label{proba_survival_exploit_dbl_proof}

For any policy $\pi$ and any budget $B'$, we will denote $\pi\in \text{ EXPLOIT}(B')$ if at round $t$, $\pi$ only pulls arms $k\in [K]$ such that $\sum_{s=1}^t X_s^{\pi_s}\1_{\pi_s = k}\geq -\frac{B'}{K} + 1$. 

The probability of ruin of EXPLOIT-UCB-DOUBLE can be decomposed as
\begin{equation}\label{exploit_double_proba_decomposition}
    P(\tau(B, \pi)<T) = \sum_{j=0}^\infty P(\tau(B, \pi)<T \cap t_j\leq \tau(B, \pi)< t_{j+1}).
\end{equation}

\noindent Let us first examine the term in $j=0$. Then,
\begin{align*}
    &P(\tau(B, \pi)<T \cap \tau(B, \pi)< t_{1}) \\
    &\leq P\left(\tau(B, \pi)<T \text{ and } \forall t\leq T, B + \sum_{s=1}^t X_s^{\pi_s}\1_{\tau(B, \pi)\geq s-1}< nB^2\right) \\
    &\leq P\left(\tau(B, \pi)<T \text{ and } \pi\in \text{ EXPLOIT}(B)\right) \\
    &\leq p^{\text{EX}}.
\end{align*}

\noindent Then, let us examine the other terms in the sum. Let $j\geq 1$. For any $t\geq 1$, we will denote $\Tilde{\pi}_t := \pi_{t_j + t}$. Let us re-write each of the terms in the sum as
\begin{multline*}
    P\left(\tau(B, \pi)<T \text{ and } t_j\leq \tau(B, \pi)< t_{j+1}\right) = \\
    P\left(t_j\leq \tau(B, \pi)<t_{j+1}; \ B + \sum_{t=1}^T X_t^{\pi_t}\1_{\tau(B, \pi)\geq t-1}<0\right).
\end{multline*}

\noindent This is re-written as
\begin{multline*}
    P\left(\tau(B, \pi)<T \text{ and } t_j\leq \tau(B, \pi)< t_{j+1}\right) = \\
    P\left(t_j\leq \tau(B, \pi)<t_{j+1};
    B + \sum_{t=1}^T X_t^{\pi_t}\1_{\forall s\leq t-1, B + \sum_{r = 1}^s X_r^{\pi_r}>0}<0\right).
\end{multline*}

\noindent But then, by definition of $t_j$, under the condition that $t_j<T$, we have that
\begin{equation*}
    B + \sum_{t=1}^{t_j} X_t^{\pi_t} \geq jnB^2,
\end{equation*}
which implies that, for any $t\geq t_j + 1$,
\begin{equation*}
    B + \sum_{s=1}^t X_s^{\pi_s} \geq jnB^2 + \sum_{s=t_j+1}^t X_s^{\pi_s}.
\end{equation*}

\noindent We can then replace in the previous equation:
\begin{multline*}
    P\left(\tau(B, \pi)<T \text{ and } t_j\leq \tau(B, \pi)< t_{j+1}\right) \leq \\
    P\left(t_j\leq \tau(B, \pi)<t_{j+1}; \ jnB^2 + \sum_{t=t_j+1}^T X_t^{\pi_t}\1_{\forall s\leq t-1, jnB^2 + \sum_{r = t_j+1}^s X_r^{\pi_r}>0}<0\right).
\end{multline*}

\noindent This is re-written as
\begin{multline*}
    P\left(\tau(B, \pi)<T \text{ and } t_j\leq \tau(B, \pi)< t_{j+1}\right) \leq \\
    P\left(t_j\leq \tau(B, \pi)< t_{j+1}; \ jnB^2 + \sum_{t=1}^{T-t_j} X_{t+t_j}^{\Tilde{\pi}_{t}}\1_{\forall s\leq t-1, jnB^2 + \sum_{r = 1}^s X_{r+t_j}^{\Tilde{\pi}_{t_j}}>0}<0\right),
\end{multline*}
and then
\begin{align*}
    P\left(\tau(B, \pi)<T \text{ and } t_j\leq \tau(B, \pi)< t_{j+1}\right) &\leq P\left(t_j\leq \tau(B, \pi)< t_{j+1}; \ \tau(B, \Tilde{\pi})<T-t_j\right) \\
    &\leq P\left(\tau(B, \Tilde{\pi})<\infty, \Tilde{\pi} \in \text{ EXPLOIT }(jnB^2)\right) \\
    &\leq \left(p^{\mathrm{EX}}\right)^{jnB}.
\end{align*}

\noindent We can then replace in (\ref{exploit_double_proba_decomposition}):
\begin{align*}
    P(\tau(B, \pi)<T) &= \sum_{j=0}^\infty P(\tau(B, \pi)<T \cap t_j\leq \tau(B, \pi)< t_{j+1}) \\
    &\leq p^{\mathrm{EX}} + \sum_{j=1}^\infty \left(p^{\mathrm{EX}}\right)^{jnB} \\
    &= p^{\mathrm{EX}} + \frac{\left(p^{\mathrm{EX}}\right)^{nB}}{1 - \left(p^{\mathrm{EX}}\right)^{nB}},
\end{align*}
which gives the desired result. Let $\epsilon>0$, then
\begin{equation*}
    n \geq \frac{\log \frac{\epsilon}{1+\epsilon}}{B\log p^{\mathrm{EX}}} \implies \frac{\left(p^{\mathrm{EX}}\right)^{nB}}{1 - \left(p^{\mathrm{EX}}\right)^{nB}} \leq \epsilon \leq \frac{\epsilon}{\mu^*},
\end{equation*}
hence
\begin{equation*}
    P(\tau(B, \pi)<\infty) \leq p^{\mathrm{EX}} + \frac{\epsilon}{\mu^*},
\end{equation*}
which concludes the proof of the proposition. 
\hfill $\blacksquare$

\subsection{Proof of Proposition~\ref{reward_exploit_dbl}}\label{reward_exploit_dbl_proof}

We will assume that the arm with the biggest expectation is arm 1 and that it is unique for the sake of clarity. Let $j := \lceil\frac{T^{1/4}}{nB^2}-1\rceil$, recall that
\begin{equation*}
    t_j = \inf\left\{t\in \{0, \ldots, \min(\tau(B, \pi), T)\}: B + \sum_{s=1}^t X_s^{\pi_s}> jnB^2 \right\}.
\end{equation*}

\noindent We still denote $\Tilde{\pi}_t := \pi_{t_j + t}$ for any $t\geq 1$. We can then decompose the reward as follows:
\begin{align*}
    &\E\left[\sum_{t=1}^T X_t^{\pi_t}\1_{\tau(B, \pi)\geq t-1}\right] \\
    &= \sum_{k=1}^K \E\left[\sum_{t=1}^T X_t^{\pi_t}\1_{\pi_t = k}\1_{\tau(B, \pi)\geq t-1}\right] \\
    &= \mu_1 \E\left[\sum_{t=1}^T \1_{\pi_t = 1}\1_{\tau(B, \pi)\geq t-1}\right] + \sum_{k=2}^K \mu_k \E\left[\sum_{t=1}^T \1_{\pi_t = k}\1_{\tau(B, \pi)\geq t-1}\right] \\
    &= \mu_1 \underbrace{\E\left[\sum_{t=1}^T \1_{\pi_t = 1}\1_{\tau(B, \pi)\geq t-1}\right]}_{(A)} + \sum_{k=2}^K \mu_k \underbrace{\E\left[\sum_{t=t_j+1}^T \1_{\pi_t = k}\1_{\tau(B, \pi)\geq t-1}\right]}_{(B_k)} + O\left(\E[t_j]\right).
\end{align*}

\noindent If we prove that
\begin{equation*}
    (A) = P(\tau(B, \pi) = \infty)T + o(T), \quad \forall k\in \{2, \dots, K\}, (B_k) = o(T), \quad \E[t_j] = o(T),
\end{equation*}
then, we can write that
\begin{equation*}
    \E\left[\sum_{t=1}^T X_t^{\pi_t}\1_{\tau(B, \pi)\geq t-1}\right] = \mu_1 P(\tau(B, \pi) = \infty)T + o(T),
\end{equation*}
and using Proposition~\ref{proba_survival_exploit_dbl} gives the result:
\begin{equation*}
    \E\left[\sum_{t=1}^T X_t^{\pi_t}\1_{\tau(B, \pi)\geq t-1}\right] \geq \mu_1 \left(1 - p^{\mathrm{EX}} - \frac{\left(p^{\mathrm{EX}}\right)^{nB}}{1 - \left(p^{\mathrm{EX}}\right)^{nB}}\right)T + o(T).
\end{equation*}

\noindent Let $\epsilon>0$. Then by Proposition~\ref{proba_survival_exploit_dbl},
\begin{equation*}
    n \geq \frac{\log \frac{\epsilon}{1+\epsilon}}{B\log p^{\mathrm{EX}}} \implies P(\tau(B, \pi)=\infty)\geq 1-p^{\mathrm{EX}}-\epsilon.
\end{equation*}

\noindent In particular, the choice of $\epsilon = \frac{T^{B\log p^{\mathrm{EX}}}}{1 - T^{B\log p^{\mathrm{EX}}}} = o_T(1)$ gives $\frac{\log \frac{\epsilon}{1+\epsilon}}{B\log p^{\mathrm{EX}}} = \log T$, and hence
\begin{equation*}
    n \geq \log T \implies \E\left[\sum_{t=1}^T X_t^{\pi_t}\1_{\tau(B, \pi)\geq t-1}\right] \geq \mu_1 (1-p^{\mathrm{EX}}) T + o(T),
\end{equation*}
which concludes the proof. It only remains to study each of the terms $(A), (B_k)$ for $k\geq 2$ and $\E[t_j]$.

\subsection*{Study of $(B_k)$}

Let $k\in \{2, \dots, K\}$. We can decompose the term $(B_k)$ as follows:
\begin{align*}
    (B_k) &= \E\left[\sum_{t=t_j+1}^T \1_{\pi_t =k}\1_{\tau(B, \pi)\geq t-1}\right] \\
    &= \E\left[\sum_{t=t_j+1}^T \1_{\pi_t =k}\1_{\tau(B, \pi)\geq t_j}\1_{\forall s\leq t-1, B + \sum_{r=1}^s X_r^{\pi_r}>0}\right] \\
    &= \E\left[\sum_{t=t_j+1}^T \1_{\pi_t =k}\1_{\tau(B, \pi)\geq t_j}\1_{\forall s\in \{t_j+1, \dots, t-1\}, B + \sum_{r=1}^{t_j} X_r^{\pi_r} + \sum_{r=t_j+1}^s X_r^{\pi_r}>0}\right].
\end{align*}

\noindent But then, by definition of $t_j$, if $t_j\leq \min(\tau(B, \pi), T)$,
\begin{equation*}%\label{def_tj}
    B + \sum_{r=1}^{t_j} X_r^{\pi_r}< jnB^2 + 1 = T^{1/4} + 1. 
\end{equation*}

\noindent This implies that, if $t_j\leq \min(\tau(B, \pi), T)$, denoting $\Tilde{\pi}_s := \pi_{s+t_j}$ for any $s\geq 1$,
\begin{align*}
    \1_{\forall s\in \{t_j+1, \dots, t-1\}, B + \sum_{r=1}^{t_j} X_r^{\pi_r} + \sum_{r=t_j+1}^s X_r^{\pi_r}>0} &\leq \1_{\forall s\in \{t_j+1, \dots, t-1\}, T^{1/4}+1 + \sum_{r=t_j+1}^s X_r^{\pi_r}>0} \\
    &= \1_{\forall s\in \{1, \dots, t-t_j-1\}, T^{1/4}+1 + \sum_{r=1}^{s-t_j} X_{r+t_j}^{\Tilde{\pi}_r}>0} \\
    &= \1_{\tau(T^{1/4}+1, \Tilde{\pi})\geq t-t_j-1},
\end{align*}
and thus, if $t_j\leq \min(\tau(B, \pi), T)$,
\begin{align*}
    (B_k) &= \E\left[\sum_{t=t_j+1}^T \1_{\pi_t =k}\1_{\tau(B, \pi)\geq t_j}\1_{\forall s\in \{t_j+1, \dots, t-1\}, B + \sum_{r=1}^{t_j} X_r^{\pi_r} + \sum_{r=t_j+1}^s X_r^{\pi_r}>0}\right] \\
    &\leq \E\left[\sum_{t=t_j+1}^T \1_{\pi_t =k}\1_{\tau(B, \pi)\geq t_j}\1_{\tau(T^{1/4}+1, \Tilde{\pi})\geq t-t_j-1}\right] \\
    &= \E\left[\sum_{t=1}^{T-t_j} \1_{\pi_{t+t_j} =k}\1_{\tau(B, \pi)\geq t_j}\1_{\tau(T^{1/4}+1, \Tilde{\pi})\geq t-1}\right].
\end{align*}

\noindent This inequality being a trivial equality if $t_j = T+1$ (because the sum is a sum on an empty set) or if $t_j = \tau(B, \pi)+1$ ($0\leq 0$), we deduce that in any case,
\begin{equation*}
    (B_k) \leq \E\left[\sum_{t=1}^{T-t_j} \1_{\pi_{t+t_j} =k}\1_{\tau(B, \pi)\geq t_j}\1_{\tau(T^{1/4}+1, \Tilde{\pi})\geq t-1}\right].
\end{equation*}

\noindent Then, denoting by $s_j$ the realized value of $t_j$ and decomposing classically, we have
\begin{align*}
    (B_k) &\leq \E\left[\sum_{t=1}^{T-t_j} \1_{\pi_{t+t_j} =k}\1_{\tau(B, \pi)\geq t_j}\1_{\tau(T^{1/4}+1, \Tilde{\pi})\geq t-1}\right] \\
    &= \sum_{s_j = T^{1/4}}^T \E\left[\sum_{t=1}^{T-s_j} \1_{\pi_{t+s_j} =k}\1_{\tau(B, \pi)\geq s_j}\1_{\tau(T^{1/4}+1, \Tilde{\pi})\geq t-1}\1_{t_j = s_j}\right] \\
    &= \sum_{s_j = T^{1/4}}^T \sum_{t=1}^{T-s_j} \E\left[ \1_{\pi_{t+s_j} =k}\1_{\tau(B, \pi)\geq s_j}\1_{\tau(T^{1/4}+1, \Tilde{\pi})\geq t-1}\1_{t_j = s_j}\right] \\
    &= \sum_{s_j = T^{1/4}}^T \sum_{t=1}^{T-s_j} P\left(\tau(B, \pi)\geq s_j, \tau(T^{1/4}+1, \Tilde{\pi})\geq t-1, t_j = s_j, \pi_{t+s_j} = k\right).
\end{align*}

\noindent Let us then study the probability $P\left(\tau(B, \pi)\geq s_j, \tau(T^{1/4}+1, \Tilde{\pi})\geq t-1, t_j = s_j, \pi_{t+s_j} = k\right)$
and decompose it as
\begin{multline}\label{suboptimal_arm_proba_decomposition}
    P\left(\tau(B, \pi)\geq s_j, \tau(T^{1/4}+1, \Tilde{\pi})\geq t-1, t_j = s_j, \pi_{t+s_j} = k\right) = \\
    P\Bigg(\tau(B, \pi)\geq s_j, \tau(T^{1/4}+1, \Tilde{\pi})\geq t-1, t_j = s_j, \pi_{t+s_j} = k, \\
    \forall t\geq s_j, \sum_{s=1}^t X_s^1 \1_{\pi_s = 1}> -\frac{T^{1/4}+1}{K}\Bigg)\\
    + P\Bigg(\tau(B, \pi)\geq s_j, \tau(T^{1/4}+1, \Tilde{\pi})\geq t-1, t_j = s_j, \pi_{t+s_j} = k, \\
    \exists t\geq s_j, \sum_{s=1}^t X_s^1 \1_{\pi_s = 1}\leq -\frac{T^{1/4}+1}{K}\Bigg).
\end{multline}

\noindent The second term on the right, though, can easily be bounded as follows:
\begin{multline*}
    P\Bigg(\tau(B, \pi)\geq s_j, \tau(T^{1/4}+1, \Tilde{\pi})\geq t-1, t_j = s_j, \pi_{t+s_j} = k, \\
    \exists t\geq s_j, \sum_{s=1}^t X_s^1 \1_{\pi_s = 1}\leq -\frac{T^{1/4}+1}{K}\Bigg) \\
    \leq P\left(\exists t\geq T^{1/4}, \sum_{s=1}^t X_s^1 \1_{\pi_s = 1}\leq -\frac{T^{1/4}+1}{K}\right).
\end{multline*}

\noindent Then, let us denote by $\delta_1^{\pi}<\delta_2^{\pi}<...$ the rounds $t\geq T^{1/4}$ at which $\pi_t = 1$, in other words, denoting $\delta_0^{\pi} := T^{1/4}$,
\begin{equation*}
    \forall j\geq 1, \delta_j^{\pi} := \inf\{t\geq \delta_{j-1}^{\pi}: \pi_t = 1\}.
\end{equation*}

\noindent Then, we can bound the probability as
\begin{align*}
    &P\left(\exists t\geq T^{1/4}, \sum_{s=1}^t X_s^1 \1_{\pi_s = 1}\leq -\frac{T^{1/4}}{K}\right) \\
    &= P\left(\exists j\geq 0, \sum_{s=1}^{\delta_j^{\pi}} X_s^1 \1_{\pi_s = 1}\leq -\frac{T^{1/4}+1}{K}\right) \\
    &= \sum_{n_1 = T^{1/4}/K}^{\infty} P\left(\exists j\geq 0, \sum_{s=1}^{\delta_j^{\pi}} X_s^1 \1_{\pi_s = 1}\leq -\frac{T^{1/4}+1}{K}, \sum_{s=1}^{T^{1/4}} \1_{\pi_s = 1} = n_1\right) \\
    &\leq \sum_{n_1 = T^{1/4}/K}^{\infty} \sum_{j=0}^{\infty} P\left(\sum_{s=1}^{\delta_j^{\pi}} X_s^1 \1_{\pi_s = 1}\leq -\frac{T^{1/4}+1}{K}, \sum_{s=1}^{T^{1/4}} \1_{\pi_s = 1} = n_1\right),
\end{align*}
because the rewards are bounded in $[-1, 1]$. Then, using Hoeffding's inequality, for any $n_1\geq \frac{T^{1/4}}{K}$ and $j\geq 0$,
\begin{align*}
    &P\left(\sum_{s=1}^{\delta_j^{\pi}} X_s^1 \1_{\pi_s = 1}\leq -\frac{T^{1/4}+1}{K}, \sum_{s=1}^{T^{1/4}} \1_{\pi_s = 1} = n_1\right) \\
    &\leq P\left(\frac{\sum_{s=1}^{\delta_j^{\pi}} X_s^1 \1_{\pi_s = 1}}{\sum_{s=1}^{\delta_j^{\pi}} \1_{\pi_s = 1}} - \mu_1 \leq -\mu_1, \sum_{s=1}^{T^{1/4}} \1_{\pi_s = 1} = n_1\right) \\
    &\leq \exp\left(-\frac{(n_1+j)\mu_1^2}{2}\right).
\end{align*}

\noindent Summing over $n_1$ and $j$ gives 
\begin{align*}
    &\sum_{n_1 = T^{1/4}/K}^{\infty} \sum_{j=0}^{\infty} P\left(\sum_{s=1}^{\delta_j^{\pi}} X_s^1 \1_{\pi_s = 1}\leq -\frac{T^{1/4}+1}{K}, \sum_{s=1}^{T^{1/4}} \1_{\pi_s = 1} = n_1\right) \\
    &\leq \sum_{n_1 = T^{1/4}/K}^{\infty} \sum_{j=0}^{\infty} \exp\left(-\frac{(n_1+j)\mu_1^2}{2}\right) \\
    &= \frac{\exp\left(-\frac{T^{1/4}\mu_1^2}{2K}\right)}{\left(1 - e^{-\frac{\mu_1^2}{2}} \right)^2}.
\end{align*}

\noindent We then deduce that
\begin{multline*}
    P\Bigg(\tau(B, \pi)\geq s_j, \tau(T^{1/4}+1, \Tilde{\pi})\geq t-1, t_j = s_j, \pi_{t+s_j} = k, \\
    \exists t\geq s_j, \sum_{s=1}^t X_s^1 \1_{\pi_s = 1}\leq -\frac{T^{1/4}+1}{K}\Bigg) = O\left(\exp\left(-\frac{T^{1/4}\mu_1^2}{2K}\right)\right),
\end{multline*}
and thus, plugging it in (\ref{suboptimal_arm_proba_decomposition}),
\begin{multline}\label{suboptimal_arm_first_simplification}
    P\left(\tau(B, \pi)\geq s_j, \tau(T^{1/4}+1, \Tilde{\pi})\geq t-1, t_j = s_j, \pi_{t+s_j} = k\right) = \\
    P\Bigg(\tau(B, \pi)\geq s_j, \tau(T^{1/4}+1, \Tilde{\pi})\geq t-1, t_j = s_j, \pi_{t+s_j} = k, \\
    \forall t\geq s_j, \sum_{s=1}^t X_s^1 \1_{\pi_s = 1}> -\frac{T^{1/4}+1}{K}\Bigg) + O\left(\exp\left(-\frac{T^{1/4}\mu_1^2}{2K}\right)\right).
\end{multline}

\noindent Let us also bound the first term in the decomposition in (\ref{suboptimal_arm_proba_decomposition}):
\begin{multline*}
    P\Bigg(\tau(B, \pi)\geq s_j, \tau(T^{1/4}+1, \Tilde{\pi})\geq t-1, t_j = s_j, \pi_{t+s_j} = k, \\
    \forall t\geq s_j, \sum_{s=1}^t X_s^1 \1_{\pi_s = 1}> -\frac{T^{1/4}+1}{K}\Bigg) \\
    \leq P\left(\pi_{t+s_j} = k, t_j = s_j, C_k(t+s_j)\geq C_1(t+s_j)\right).
\end{multline*}

\noindent We can then re-write the previous term in the sum, as
\begin{align*}
    &\sum_{s_j = T^{1/4}}^T \sum_{t=1}^{T-s_j} P\left(\pi_{t+s_j} = k, t_j = s_j, C_k(t+s_j)\geq C_1(t+s_j)\right) \\
    &= \sum_{s_j = T^{1/4}}^T \sum_{t=1}^{T-s_j} \E\left[\1_{\pi_{t+s_j} = k, t_j = s_j, C_k(t+s_j)\geq C_1(t+s_j)}\right] \\
    &= \E\left[\sum_{t=1}^{T-t_j}\1_{\pi_{t+t_j} = k, C_k(t+t_j)\geq C_1(t+t_j)}\right] \\
    &\leq \sum_{t=1}^T \E\left[\1_{\pi_{t} = k, C_k(t)\geq C_1(t)}\right] \\
    &= o(T),
\end{align*}
by Lemma~\ref{classic_bandit_regret_bound}. Overall, we can replace in (\ref{suboptimal_arm_first_simplification}) and deduce that
\begin{equation*}
    \sum_{s_j = T^{1/4}}^T \sum_{t=1}^{T-s_j} P\left(\tau(B, \pi)\geq s_j, \tau(T^{1/4}+1, \Tilde{\pi})\geq t-1, t_j = s_j, \pi_{t+s_j} = k\right) = o(T),
\end{equation*}
which straightforwardly implies that
\begin{equation*}
    (B_k) = o(T).
\end{equation*}

\subsection*{Study of $\E[t_j]$}

We can decompose
\begin{align*}
    \E[t_j] &= \E\left[t_j \1_{t_j = \min(\tau(B, \pi), T) + 1}\right] + \E\left[t_j\1_{t_j \leq \min(\tau(B, \pi), T)}\right] \\
    &= (T+1) \underbrace{P(t_j = T+1)}_{(C)} + \underbrace{\E\left[(\tau(B, \pi)+1)\1_{t_j = \tau(B, \pi) + 1}\right]}_{(D)} + \underbrace{\E\left[t_j\1_{t_j \leq \min(\tau(B, \pi), T)}\right]}_{(E)}.
\end{align*}

\noindent We can first bound the term $(D)$, as
\begin{equation*}
    (D) \leq \left(\sqrt{T}+1\right) P\left(\tau(B, \pi)\geq \sqrt{T}\right) + (T+1) P\left(\sqrt{T}<\tau(B, \pi)\leq T\right).
\end{equation*}

\noindent But since $(P(\tau(B, \pi)\leq n))_{n\geq 1}$ is a sequence which converges to $P(\tau(B, \pi)<\infty)$, we deduce that
\begin{equation*}
    P\left(\sqrt{T}<\tau(B, \pi)\leq T\right) = P(\tau(B, \pi)\leq T) - P\left(\tau(B, \pi)\leq \sqrt{T}\right) \rightarrow_{T\rightarrow +\infty} 0,
\end{equation*}
and therefore,
\begin{equation*}
    (D) = o(T).
\end{equation*}

\noindent We then deduce that
\begin{equation*}
    \E[t_j] = (C) + (E) + o(T).
\end{equation*}

\noindent Then, let us study the term $(C)$ and bound it by:
\begin{align}
    (C) &\leq P\left(\tau(B, \pi)\geq T, \ \forall t\leq T, B + \sum_{s=1}^t X_s^{\pi_s} \leq T^{\frac{1}{4}}, \ t_j = T + 1\right) \nonumber \\
    &\leq P\left(\tau(B, \pi)\geq T, \ B + \sum_{s=1}^{T^{3/4}} X_s^{\pi_s} \leq T^{\frac{1}{4}}, \ t_j = T + 1\right) \nonumber \\
    &\leq P\left(\tau(B, \pi)\geq T^{\frac{3}{4}}, \ B + \sum_{s=1}^{T^{3/4}} X_s^{\pi_s} \leq T^{\frac{1}{4}}, \ t_j \geq T^{\frac{3}{4}}\right). \label{bound_c}
\end{align}

\noindent Then, let us study the term $(E)$. Actually,
\begin{align*}
    &P(t_j\geq T^{3/4}, t_j\leq \min(\tau(B, \pi), T)) \\
    &= P\left(\tau(B, \pi)\geq T^{\frac{3}{4}}, \ \forall t\leq T^{3/4}, B + \sum_{s=1}^t X_s^{\pi_s} \leq T^{1/4}, \ t_j\geq T^{\frac{3}{4}}\right) \\
    &\leq P\left(\tau(B, \pi)\geq T^{\frac{3}{4}}, \ B + \sum_{s=1}^{T^{3/4}} X_s^{\pi_s}\leq T^{1/4}, \ t_j\geq T^{\frac{3}{4}}\right).
\end{align*}

\noindent We then deduce that
\begin{align}
    (E) &\leq T P\left(t_j\geq T^{3/4}, t_j\leq \min(\tau(B, \pi), T)\right) + T^{3/4} P\left(t_j< T^{3/4}, t_j\leq \min(\tau(B, \pi), T)\right) \nonumber \\
    &= T P\left(\tau(B, \pi)\geq T^{\frac{3}{4}}, \ B + \sum_{s=1}^{T^{3/4}} X_s^{\pi_s}\leq T^{1/4}, \ t_j\geq T^{\frac{3}{4}}\right) + o(T). \label{bound_e}
\end{align}

\noindent Using \eqref{bound_c} and \eqref{bound_e}, we deduce that
\begin{equation*}
    \E[t_j] \leq 2T P\left(\tau(B, \pi)\geq T^{\frac{3}{4}}, \ B + \sum_{s=1}^{T^{3/4}} X_s^{\pi_s}\leq T^{1/4}, \ t_j\geq T^{\frac{3}{4}}\right) + o(T).
\end{equation*}

\noindent Then, consider the set of conditions $\left\{\tau(B, \pi)\geq T^{\frac{3}{4}}, \ B + \sum_{s=1}^{T^{3/4}} X_s^{\pi_s} \leq T^{\frac{1}{4}}, \ t_j \geq T^{\frac{3}{4}} \right\}$ and assume there exists an arm $k_0\in [K]$ such that $\sum_{s=1}^{T^{3/4}} X_s^{\pi_s} \1_{\pi_s = k_0}> T^{\frac{1}{3}}$. Since $T^{3/4}\leq t_j$, we know that for any $k\in [K]$,
\begin{equation*}
    \sum_{t=1}^{T^{3/4}} X_t^{\pi_t} \1_{\pi_t = k}\geq -\frac{T^{\frac{1}{4}}}{K} - 1,
\end{equation*}
and hence,
\begin{align*}
    B + \sum_{t=1}^{T^{3/4}} X_t^{\pi_t} &= B + \sum_{t=1}^{T^{3/4}} X_t^{\pi_t}\1_{\pi_t = k_0} + \sum_{k\ne k_0} \sum_{t=1}^{T^{3/4}} X_t^{\pi_t}\1_{\pi_t = k} \\
    &\geq T^{\frac{1}{3}} - \frac{K-1}{K} T^{\frac{1}{4}} - (K-1) \\
    &= \Omega(T^{\frac{1}{3}}),
\end{align*}
which contradicts the hypothesis $B + \sum_{t=1}^{T^{3/4}} X_t^{\pi_t} \leq T^{\frac{1}{4}}$. We deduce that 
\begin{align*}
    &P\left(\tau(B, \pi)\geq T^{\frac{3}{4}}, \ B + \sum_{s=1}^{T^{3/4}} X_s^{\pi_s}\leq T^{1/4}, \ t_j\geq T^{\frac{3}{4}}\right) \\
    &\leq P\left(\tau(B, \pi)\geq T^{\frac{3}{4}}, \ \forall k\in [K], \sum_{t=1}^{T^{3/4}} X_t^{\pi_t}\1_{\pi_t = k} \leq T^{\frac{1}{3}}\right) \\
    &\leq P\left(\exists k\in [K], \sum_{t=1}^{T^{3/4}} \1_{\pi_t = k}\geq \frac{T^{3/4}}{K} \text{ and } \sum_{t=1}^{T^{3/4}} X_t^{\pi_t}\1_{\pi_t = k} \leq T^{\frac{1}{3}}\right) \\
    &= P\left(\exists k\in [K]: \ -\frac{T^{1/4}}{K}-1 <\sum_{s=1}^{T^{3/4}} X_s^{\pi_s}\1_{\pi_s = k}< T^{1/3}, \sum_{s=1}^{T^{3/4}} \1_{\pi_s = k}\geq \frac{T^{3/4}}{K} \right) \\
    &\leq P\left(\exists k\in [K]: \ \left|\sum_{s=1}^{T^{3/4}} X_s^{\pi_s}\1_{\pi_s = k}\right|< T^{1/3}, \sum_{s=1}^{T^{3/4}} \1_{\pi_s = k}\geq \frac{T^{3/4}}{K} \right) \\
    &\leq \sum_{k=1}^K P\left( \left|\sum_{s=1}^{T^{3/4}} X_s^{\pi_s}\1_{\pi_s = k}\right|< T^{1/3}, \sum_{s=1}^{T^{3/4}} \1_{\pi_s = k}\geq \frac{T^{3/4}}{K} \right).
\end{align*}

\noindent Then, for any arm $k\in [K]$, there are two cases: either $\E[X_1^k]\ne 0$, or $\E[X_1^k] = 0$. In the former case, we can use Hoeffding's inequality to bound the above probability:
\begin{align*}
    &P\left( \left|\sum_{s=1}^{T^{3/4}} X_s^{\pi_s}\1_{\pi_s = k}\right|< T^{1/3}, \sum_{s=1}^{T^{3/4}} \1_{\pi_s = k}\geq \frac{T^{3/4}}{K} \right) \\
    &\leq P\left( \left|\sum_{s=1}^{T^{3/4}} X_s^{\pi_s}\1_{\pi_s = k} - \E[X_1^k]\right|< \frac{K}{T^{5/12}} - \E[X_1^k], \sum_{s=1}^{T^{3/4}} \1_{\pi_s = k}\geq \frac{T^{3/4}}{K} \right) \\
    &\leq \exp\left(\frac{T^{3/4}\left(\frac{K}{T^{5/12}} - \E[X_1^k]\right)^2}{4}\right) \\
    &= o(1).
\end{align*}

\noindent In the latter case, we can bound this probability as follows:
\begin{align*}
    &P\left( \left|\sum_{s=1}^{T^{3/4}} X_s^{\pi_s}\1_{\pi_s = k}\right|< T^{1/3}, \sum_{s=1}^{T^{3/4}} \1_{\pi_s = k}\geq \frac{T^{3/4}}{K} \right) \\
    &\leq P\left( \left|\frac{\sum_{s=1}^{T^{3/4}} X_s^{\pi_s}\1_{\pi_s = k}}{\sqrt{\sum_{s=1}^{T^{3/4}} \1_{\pi_s = k}}}\right|< \frac{K}{T^{1/24}}, \ \sum_{s=1}^{T^{3/4}} \1_{\pi_s = k}\geq \frac{T^{3/4}}{K}\right).
\end{align*}

\noindent Then, under the assumption that there is no zero arm, $\text{Var}(X_1^k)>0$ and 
\begin{equation*}
    \frac{\sum_{s=1}^{T^{3/4}} X_s^{\pi_s}\1_{\pi_s = k}}{\sqrt{\sum_{s=1}^{T^{3/4}} \1_{\pi_s = k}}} \xrightarrow{d} \mathcal{N}(0, \text{Var}(X_1^k)),
\end{equation*}
and since $\frac{1}{T^{1/24}} \xrightarrow{T\rightarrow +\infty} 0$, we deduce that
\begin{equation*}
    P\left(t_j\geq T^{3/4}\right) = o(1).
\end{equation*}

\noindent In any case, we have that
\begin{equation*}
    P\left( \left|\sum_{s=1}^{T^{3/4}} X_s^{\pi_s}\1_{\pi_s = k}\right|< T^{1/3}, \sum_{s=1}^{T^{3/4}} \1_{\pi_s = k}\geq \frac{T^{3/4}}{K} \right) = o(1),
\end{equation*}
and hence,
\begin{equation*}
    P\left(\tau(B, \pi)\geq T^{\frac{3}{4}}, \ B + \sum_{s=1}^{T^{3/4}} X_s^{\pi_s}\leq T^{1/4}, \ t_j\geq T^{\frac{3}{4}}\right) = o(1).
\end{equation*}

\noindent We then deduce that
\begin{equation*}
    \E[t_j] = o(T).
\end{equation*}

\subsection*{Study of $(A)$}

This term is the main one in the previous decomposition.
\begin{align*}
    &\E\left[\sum_{t=1}^T \1_{\tau(B, \pi)\geq t-1}\right] \\
    &= \E\left[\sum_{t=1}^T \1_{\pi_t=1}\1_{\tau(B, \pi)\geq t-1}\right] + \sum_{k=2}^K \E\left[\sum_{t=1}^T \1_{\pi_t=k}\1_{\tau(B, \pi)\geq t-1}\right] \\
    &= (A) + \sum_{k=2}^K \E\left[\sum_{t=1}^{t_j} \1_{\pi_t=k}\1_{\tau(B, \pi)\geq t-1}\right] + \sum_{k=2}^K \E\left[\sum_{t=t_j+1}^T \1_{\pi_t=k}\1_{\tau(B, \pi)\geq t-1}\right] \\
    &= (A) + \sum_{k=2}^K (B_k) + O(\E[t_j]).
\end{align*}

\noindent But then, using the previous bounds on $(B_k)$ and $\E[t_j]$, we deduce that
\begin{equation*}
    (A) = \E\left[\sum_{t=1}^T \1_{\tau(B, \pi)\geq t-1}\right] + o(T).
\end{equation*}

\noindent Then, we can simply replace the factor with the expectation by the probability of survival, as
\begin{align*}
    \E\left[ \sum_{t=1}^T \1_{\tau(B, \pi)\geq t-1}\right] &=  \sum_{t=1}^T P(\tau(B, \pi)\geq t-1) \\
    &= TP(\tau(B, \pi) = \infty) + o(T).
\end{align*}

\noindent Hence,
\begin{equation*}
    (A) = \mu_1 P(\tau(B, \pi)=\infty)T + o(T),
\end{equation*}
which concludes the proof of the proposition. 
\hfill $\blacksquare$

\section{Proof of the Pareto-optimality of EXPLOIT-UCB-DOUBLE (Theorem~\ref{teaser_theorem})}\label{final_trivial_steps_section}

In the proofs of the performance of EXPLOIT-UCB-DOUBLE, for the sake of clarity, we drop the exponent $n$ in the notation of EXPLOIT-UCB-DOUBLE, and $\pi^n$ becomes $\pi$. The main objective of this section is to prove that EXPLOIT-UCB-DOUBLE is regret-wise Pareto-optimal in the case of rewards in $\{-1, 0, 1\}$ and with parameter $n = \log T$.

The first subsection provides a preliminary lemma, useful for the proof of the Pareto-optimality exposed in the second subsection. The last subsection makes use of the preliminary lemma to derive an upper bound on the relative regret of EXPLOIT-UCB-DOUBLE in the general case.

\subsection{Preliminary Lemma}\label{preliminary_lemma_appendix}

\begin{lemma}\label{reward_expressions_lemma}
Let $\pi$ be any policy. Then, it holds that
\begin{equation}\label{reward_basic_inequality}
    \Rew(\pi) \leq P(\tau(B, \pi)\geq \sqrt{T}) \times \max_{k\in [K]} \mu_k T + o(T).
\end{equation}
Furthermore, if $\pi$ is an anytime policy, it holds that
\begin{equation*}
    \Rew(\pi) = P(\tau(B, \pi)=\infty)\E\left[\sum_{t=1}^T X_t^{\pi_t} \Big| \tau(B, \pi)\geq T\right] + o(T).
\end{equation*}
\end{lemma}

\begin{proof}
In order to prove the first statement of the lemma, we decompose the expected cumulative reward as follows:
\begin{align*}
    \Rew(\pi) &= \E\left[\sum_{t=1}^{\tau(B, \pi)} X_t^{\pi_t}\1_{\tau(B, \pi)<\sqrt{T}}\right] + \E\left[\sum_{t=1}^T X_t^{\pi_t}\1_{\tau(B, \pi)\geq t-1}\1_{\tau(B, \pi)\geq \sqrt{T}}\right] \\
    &\leq -B + P(\tau(B, \pi)\geq \sqrt{T}) \times \E\left[\sum_{t=1}^T X_t^{\pi_t}\1_{\tau(B, \pi)\geq t-1} \bigg| \tau(B, \pi)\geq \sqrt{T}\right] \\
    &\leq -B + \sqrt{T} + P(\tau(B, \pi)\geq \sqrt{T}) \times \E\left[\sum_{t=\sqrt{T}+1}^T X_t^{\pi_t}\1_{\tau(B, \pi)\geq t-1} \bigg| \tau(B, \pi)\geq \sqrt{T}\right] \\
    &\leq -B + \sqrt{T} + P(\tau(B, \pi)\geq \sqrt{T}) \times \max_{k\in [K]} \mu_k (T - \sqrt{T}) \\
    &\leq P(\tau(B, \pi)\geq \sqrt{T}) \times \max_{k\in [K]} \mu_k T + 2\sqrt{T},
\end{align*}
which gives the desired result. For the second statement, we start by writing the reward as
\begin{align*}
    \Rew(\pi) &= \E\left[\sum_{t=1}^T X_t^{\pi_t}\1_{\tau(B, \pi)\geq t-1}\right] \\
    &= \E\left[\sum_{t=1}^T X_t^{\pi_t}\1_{\tau(B, \pi)\geq t-1}\1_{\tau(B, \pi)<T}\right] + \E\left[\sum_{t=1}^T X_t^{\pi_t}\1_{\tau(B, \pi)\geq T}\right] \\
    &= \E\left[\sum_{t=1}^{\tau(B, \pi)} X_t^{\pi_t}\1_{\tau(B, \pi)<T}\right] + \E\left[\sum_{t=1}^T X_t^{\pi_t}\1_{\tau(B, \pi)\geq T}\right].
\end{align*}

\noindent By definition of $\tau(B, \pi)$, 
$$\E\left[\sum_{t=1}^{\tau(B, \pi)} X_t^{\pi_t}\1_{\tau(B, \pi)<T}\right] \leq -B < \E\left[\sum_{t=1}^{\tau(B, \pi) - 1} X_t^{\pi_t}\1_{\tau(B, \pi)<T}\right],$$
and since the rewards are bounded in $[-1, 1]$, we deduce that
$$-(B+1) < \E\left[\sum_{t=1}^{\tau(B, \pi)} X_t^{\pi_t}\1_{\tau(B, \pi)<T}\right] \leq -B,$$
which implies
\begin{equation*}
    \Rew(\pi) = \E\left[\sum_{t=1}^T X_t^{\pi_t}\1_{\tau(B, \pi)\geq T}\right] + o(T).
\end{equation*}

\noindent It is trivial that, for any anytime policy $\pi$,
\begin{equation}\label{manip_on_proba_T}
    P(\tau(B, \pi)\geq T) = P(\tau(B, \pi) = \infty) + o_{T\to +\infty}(1).
\end{equation}

\noindent This implies
\begin{align*}
    \Rew(\pi) &= \E\left[\sum_{t=1}^T X_t^{\pi_t}\1_{\tau(B, \pi)\geq T}\right] + o(T) \\
    &= P(\tau(B, \pi)\geq T)\E\left[\sum_{t=1}^T X_t^{\pi_t}\Bigg| \tau(B, \pi)\geq T\right] + o(T) \\
    &= P(\tau(B, \pi)=\infty)\E\left[\sum_{t=1}^T X_t^{\pi_t}\Bigg| \tau(B, \pi)\geq T\right] + o(T),
\end{align*}
which concludes the proof of the lemma.
\end{proof}

\subsection{Proof of the Pareto-optimality}

We denote by $\pi^n$ the anytime policy EXPLOIT-UCB-DOUBLE with parameter $n\geq 1$. Then, Propositions~\ref{proba_survival_exploit_dbl} and \ref{reward_exploit_dbl}, along with Lemma~\ref{reward_expressions_lemma} give the cumulative reward of EXPLOIT-UCB-DOUBLE:
\begin{equation*}
    \Rew(\pi^n) \geq \left(1 - p^{\text{EX}} - \frac{\left(p^{\text{EX}}\right)^{nB}}{1 - \left(p^{\text{EX}}\right)^{nB}}\right) \max_{k\in [K]} \mu_k T + o(T).
\end{equation*}

\noindent In particular, we deduce the reward of the (non-anytime) policy EXPLOIT-UCB-DOUBLE with parameter $n = \log T$:
\begin{equation*}
    \Rew(\pi^{\log T}) = \left(1 - p^{\text{EX}}\right) \max_{k\in [K]} \mu_k T + o(T).
\end{equation*}

\noindent Recall from Lemma~\ref{reward_expressions_lemma} that, for any policy $\tilde{\pi}$,
\begin{equation*}
    \Rew(\tilde{\pi}^T) \leq P\left(\tau(B, \tilde{\pi}^T)\geq \sqrt{T}\right) \max_{k\in [K]} \mu_k T + o(T),
\end{equation*}
and as a result,
\begin{equation*}
    \frac{\Rew(\tilde{\pi}^T) - \Rew(\pi^{\log T})}{T} \leq \left(P\left(\tau(B, \tilde{\pi}^T)\geq \sqrt{T}\right) - (1-p^{\text{EX}})\right) \max_{k\in [K]} \mu_k + o(1).
\end{equation*}

\noindent For $T\geq \frac{9B^2}{\Delta_F^2}$, it holds that
\begin{equation*}
    \frac{\Rew(\tilde{\pi}^T) - \Rew(\pi^{\log T})}{T} \leq \left(P\left(\tau(B, \tilde{\pi}^T)\geq \frac{3B}{\Delta_F}\right) - (1-p^{\text{EX}})\right) \max_{k\in [K]} \mu_k + o(1),
\end{equation*}
and taking the limit gives
\begin{equation*}
    \Reg(\pi\|\tilde{\pi}) \leq \left(p^{\text{EX}} - P\left(\tau(B, \tilde{\pi}^T)< \frac{3B}{\Delta_F}\right)\right) \max_{k\in [K]} \mu_k,
\end{equation*}
where $\pi$ denotes $(\pi^{\log T})_{T\geq 1}$ the optimally-tuned EXPLOIT-UCB-DOUBLE. Then, assume that 
\begin{equation*}
    \sup_F \Reg(\pi\|\tilde{\pi}) > 0,
\end{equation*}
which implies that there exists some arm distributions $\tilde{F}$ such that
\begin{equation*}
    p^{\text{EX}} - P_{\tilde{F}}\left(\tau(B, \tilde{\pi}^T)< \frac{3B}{\Delta_{\tilde{F}}}\right) > 0.
\end{equation*}

\noindent Then, by Theorem~\ref{non_asymptotic_lower_bound_theorem}, there exist some other arm distributions $F$ such that
\begin{equation*}
    p^{\text{EX}} - P_{F}\left(\tau(B, \tilde{\pi}^T)< \frac{3B}{\Delta_{F}}\right) < 0,
\end{equation*}
and since $\max_{k\in [K]} \mu_k >0$, this implies
\begin{equation*}
    \inf_F \Reg(\pi\|\tilde{\pi}) < 0,
\end{equation*}
proving that $(\pi^{\log T})_{T\geq 1}$ is regret-wise Pareto-optimal.

\subsection{Relative Regret of EXPLOIT-UCB-DOUBLE in the general case}\label{relative_regret-exploit_dbl_general_section}

In the general case, Propositions~\ref{proba_survival_exploit_dbl} and \ref{reward_exploit_dbl} and Lemma~\ref{reward_expressions_lemma} imply that the cumulative reward of EXPLOIT-UCB-DOUBLE $\pi^n$ with parameter $n\geq 1$ is
\begin{equation}
    \Rew(\pi^n) \geq \left(1 - p^{\text{EX}} - \frac{\left(p^{\text{EX}}\right)^{nB}}{1 - \left(p^{\text{EX}}\right)^{nB}}\right) \max_{k\in [K]} \mu_k T + o(T). \label{exploit_dbl_almost_pareto_ineq}
\end{equation}

\noindent Let $\pi'$ be any policy. The previous result implies that there exist some arm distributions $F$ such that
\begin{equation*}
    \lim_{T\to+\infty} \frac{\Rew(\pi') - \left(1 - p^{\text{EX}}\right) \max_{k\in [K]} \mu_k T}{T} < 0.
\end{equation*}

\noindent With \eqref{exploit_dbl_almost_pareto_ineq}, this implies that EXPLOIT-UCB-DOUBLE $\pi^n$ achieves, for any policy $\pi'$,
\begin{equation*}
    \inf_F \Reg(\pi^n \|\pi') = \inf_F \lim_{T\to+\infty} \frac{\Rew(\pi') - \Rew(\pi^n)}{T} < \frac{\left(p^{\text{EX}}\right)^{nB}}{1 - \left(p^{\text{EX}}\right)^{nB}} \max_{k\in [K]} \mu_k.
\end{equation*}

\section{Practical Performance of the Algorithms Introduced}\label{experiments_appendix}

In this appendix, we provide some additional experimental results on the performance of the algorithms EXPLOIT-UCB and EXPLOIT-UCB-DOUBLE introduced in this paper. More precisely, we compare their practical performance to the classic bandit algorithms
UCB (\citealp{auer}) and  Multinomial Thompson Sampling (MTS, \citealp{riou}), the latter of which is chosen for its optimality in the classic MAB of multinomial arms of given support.
We used the UCB index of form $\Bar{x}_i + \sqrt{\frac{\log t}{2 n_i}}$ for UCB, EXPLOIT-UCB and EXPLOIT-UCB-DOUBLE, where $\Bar{x}_i$ and $n_i$ are the mean reward and the number of samples from arm $i$, respectively. We will further look at the impact of the hyperparameter $n$ on the survival regret of the algorithm EXPLOIT-UCB-DOUBLE, and therefore, we will consider EXPLOIT-UCB-DOUBLE for various values of $n$, including the case where $n$ is properly tuned as $n = \lceil\log T\rceil$.

For all the experiments performed, we consider a bandit setting with $K = 3$ multinomial arms of common support $\{-1, 0, 1\}$ and distributions $F^{(i_1)}, F^{(i_2)}$ and $F^{(i_3)}$, where $i_1, i_2, i_3 \in \{1, \dots, 10\}$. The distributions $F^{(i)}$ for $i\in\{1, \dots, 10\}$ are described below:

\begin{center}
\begin{tabular}{l l l} 
 $F^{(1)} = \Mult(0.4, 0.12, 0.48)$; & $F^{(2)} = \Mult(0.04, 0.88, 0.08)$; & $F^{(3)} = \Mult(0.5, 0.1, 0.4)$; \\ 
 $F^{(4)} = \Mult(0.48, 0, 0.52)$; & $F^{(5)} = \Mult(0.04, 0.91, 0.05)$; & $F^{(6)} = \Mult(0.45, 0, 0.55)$; \\
 $F^{(7)} = \Mult(0.05, 0.85, 0.1)$; & $F^{(8)} = \Mult(0.5, 0, 0.5)$; & $F^{(9)} = \Mult(0.495, 0, 0.505)$; \\
 $F^{(10)} = \Mult(0.049, 0.9, 0.051)$. & & \\
\end{tabular}
\end{center}

\noindent We set the horizon $T$ equal to $20000$ and we further assume that the initial budget $B$ is a multiple of $3$, the number of arms. As a consequence, in this framework, EXPLOIT-UCB-DOUBLE with the parameter $n = \lceil\log T\rceil = 10$ is the theoretically recommended parameter. We will also study EXPLOIT-UCB-DOUBLE with parameters $n=1$ and $n=100$ and see if their survival regret differs much from EXPLOIT-UCB-DOUBLE with parameter $n=\lceil\log T\rceil$ in practice. The average survival time $\min\{T, \tau\}$, as well as the proportion of ruins (i.e. the proportion of simulations for which the ruin occurred) of each algorithm in every setting considered, are gathered in Tables \ref{tab:average_ruin} and \ref{tab:proportion_ruins}, where EX-UCB denotes EXPLOIT-UCB, and EX-D denotes EXPLOIT-UCB-DOUBLE. They are the averages over $200$ trials, except for the last setting with $B = 30$ and arm distributions $\{F^{(9)}, F^{(10)}, F^{(3)}\}$ whose result is the average over $500$ trials.

\begin{table}[H]
\caption{\label{tab:average_ruin} Average Survival Time.}
\begin{tabular}{|c|c||c|c|c|c|c|c|} 
 \hline
  \multicolumn{2}{|c||}{Setting} & \multicolumn{6}{|c|}{Average Survival Time} \\ [0.5ex] 
 \hline
 Budget & Arms & UCB & MTS & EX-UCB & EX-D$(\log T)$ & EX-D$(1)$ & EX-D$(100)$ \\ 
 \hline\hline 
 $B = 9$ & $\{F^{(1)}, F^{(2)}, F^{(3)}\}$ & 17128 & 13957 & 18210 & 18510 & 18712 & 18613 \\ 
 \hline
 $B = 9$ & $\{F^{(4)}, F^{(5)}, F^{(3)}\}$ & \phantom{0}8856 & \phantom{0}5218 & 12160 & 10016 & \phantom{0}9176 & \phantom{0}9680 \\
 \hline
 $B = 9$ & $\{F^{(6)}, F^{(7)}, F^{(8)}\}$ & 18211 & 16427 & 18617 & 19305 & 18408 & 18910 \\
 \hline
 $B = 30$ & $\{F^{(1)}, F^{(2)}, F^{(3)}\}$ & 20000 & 19904 & 20000 & 20000 & 20000 & 20000 \\ 
 \hline
 $B = 30$ & $\{F^{(9)}, F^{(10)}, F^{(3)}\}$ & \phantom{0}8931 & \phantom{0}6708 & 11912 & 11358 & 11136 & 11273 \\
 \hline
\end{tabular}
\end{table}

\begin{table}[H]
\caption{\label{tab:proportion_ruins} Proportion of Ruin of the Algorithms.}
\begin{tabular}{|c|c||c|c|c|c|c|c|} 
 \hline
  \multicolumn{2}{|c||}{Setting} & \multicolumn{6}{|c|}{Proportions of Ruins on $200$ Trials} \\ [0.5ex] 
 \hline
 Budget & Arms & UCB & MTS & EX-UCB & EX-D$(\log T)$ & EX-D$(1)$ & EX-D$(100)$ \\ 
 \hline\hline 
 $B = 9$ & $\{F^{(1)}, F^{(2)}, F^{(3)}\}$ & 0.15 & 0.31 & 0.09 & 0.08 & 0.07 & 0.07 \\ 
 \hline
 $B = 9$ & $\{F^{(4)}, F^{(5)}, F^{(3)}\}$ & 0.57 & 0.75 & 0.4 & 0.52 & 0.56 & 0.54 \\
 \hline
 $B = 9$ & $\{F^{(6)}, F^{(7)}, F^{(8)}\}$ & 0.09 & 0.18 & 0.07 & 0.04 & 0.08 & 0.06 \\
 \hline
 $B = 30$ & $\{F^{(1)}, F^{(2)}, F^{(3)}\}$ & 0 & 0.01 & 0 & 0 & 0 & 0 \\
 \hline
 $B = 30$ & $\{F^{(9)}, F^{(10)}, F^{(3)}\}$ & 0.68 & 0.76 & 0.5 & 0.56 & 0.57 & 0.58 \\
 \hline
\end{tabular}
\end{table}

\subsection{The Case of Small Budget}

In Figure \ref{fig:smab:long_full1}, we first consider the setting where the (multinomial) arm distributions are $\{F^{(1)}, F^{(2)},$ $F^{(3)}\}$ and the budget is set to $B = 9$. It is very clear that, in this setting, EXPLOIT-UCB-DOUBLE, whatever the value of the parameter $n\in \{1, \lceil\log T\rceil, 100\}$, outperforms UCB, MTS and EXPLOIT-UCB. The performance of EXPLOIT-UCB-DOUBLE does not seem to be so significantly affected by the value of the parameter $n$, as the proportion of ruins is very similar for those three algorithms, around $0.07$ (see Table \ref{tab:proportion_ruins}). On the other hand, UCB and MTS suffer from more frequent ruins than EXPLOIT-UCB and EXPLOIT-UCB-DOUBLE, and hence have a poor performance.

\begin{figure}[t]
\centering
\includegraphics[scale=0.8]{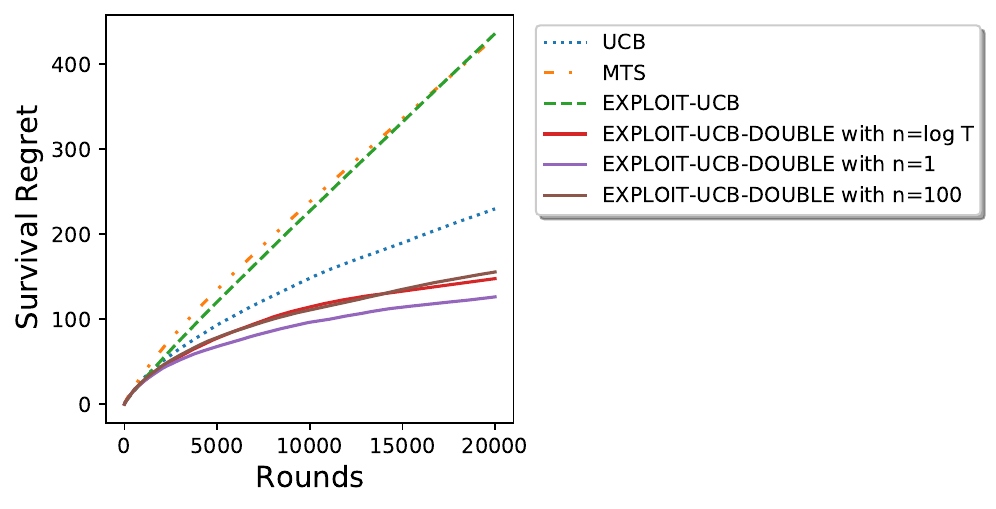}
\caption{Survival regret for $B=9$ and arms $\{F^{(1)}, F^{(2)}, F^{(3)}\}$.}
\label{fig:smab:long_full1}
\end{figure}

In the classic stochastic bandit setting, MTS is theoretically and practically optimal and its performance is much better than the one of UCB (see \citealp{riou}). On the other hand, in the S-MAB setting, its survival regret is much larger than the one of UCB, showing that the randomization factor in Thompson sampling is the source of more frequent ruins: MTS suffers a ruin for a proportion of $0.31$ of the simulations, while UCB suffers a ruin only for a proportion of $0.15$ of the simulations (see Table \ref{tab:proportion_ruins}). On the other hand, EXPLOIT-UCB has a poor performance, comparable to the one of MTS in this setting. While the risk of ruin of EXPLOIT-UCB is very low and comparable to the one of the various EXPLOIT-UCB-DOUBLE algorithms considered, it suffers from its lack of exploration.

\begin{figure}[t]
\centering
\includegraphics[scale=0.8]{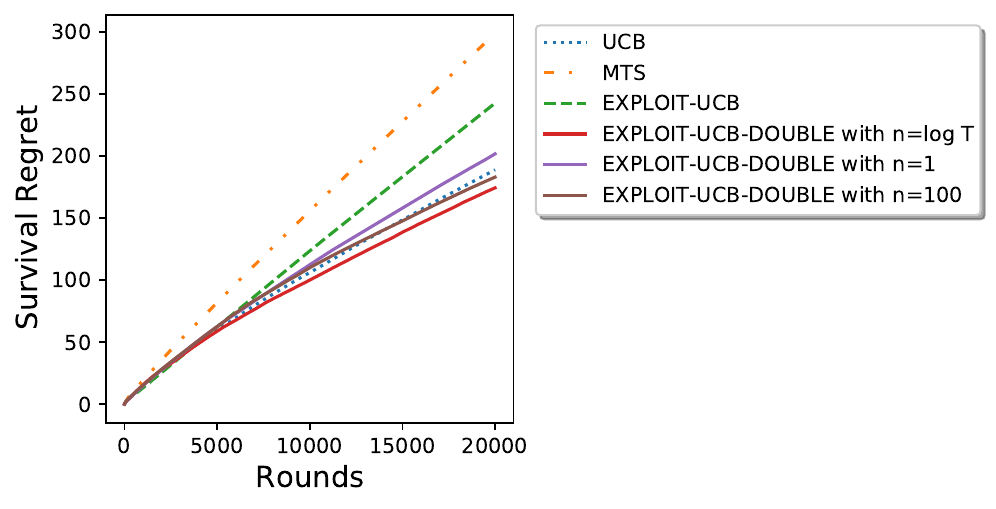}
\caption{Survival regret for $B=9$ and arms $\{F^{(4)}, F^{(5)}, F^{(3)}\}$.}
\label{fig:smab:long_full2}
\end{figure}

In Figure \ref{fig:smab:long_full2}, the budget is still set to $B = 9$ while the arm distributions are $\{F^{(4)}, F^{(5)}, F^{(3)}\}$. In this setting, the arm of distribution $F^{(3)}$ has a negative expectation $-0.1$ (and a probability of survival equal to $0$), while the arms of distributions $F^{(4)}$ and $F^{(5)}$ are hard to set apart: their expectations are both low, at $0.04$ and $0.01$ respectatively, with fairly different probabilities of survival at $0.08$ and $0.2$ respectively. Here, UCB performs very well, even at a comparable level to the three EXPLOIT-UCB-DOUBLE algorithms, because it performs a strong exploitation between arms $F^{(4)}$ and $F^{(5)}$, and avoids the worst arm $F^{(3)}$. In contrast, due to the small budget, the three EXPLOIT-UCB-DOUBLE sometimes pull arm $F^{(3)}$. However, they take advantage of their budgeted exploitation to pull more often arm $F^{(5)}$, whose probability of survival is larger than the one of arm $F^{(4)}$, catching up with UCB's regret. This explains why UCB has a smaller survival time and a larger proportion of ruins than EXPLOIT-UCB-DOUBLE algorithms for any parameter $n\in \{1, \lceil\log T\rceil, 100\}$ (see Tables \ref{tab:average_ruin} and \ref{tab:proportion_ruins}). Please note that EXPLOIT-UCB-DOUBLE with the properly tuned parameter $n = \lceil\log T\rceil$, as well as EXPLOIT-UCB-DOUBLE with $n = 100$, perform better than UCB.

On the other hand, EXPLOIT-UCB has a much larger survival regret than its doubling-tricked counterparts. Looking carefully at the average time of ruin of the algorithms however, it is clear that, as expected, EXPLOIT-UCB has a smaller risk of ruin (see Tables \ref{tab:average_ruin} and \ref{tab:proportion_ruins}). The reason behind the larger survival regret of EXPLOIT-UCB is again its lack of exploration. Indeed, recall that, in this setting, the arms expectations are quite low (resp. $0.04, 0.01$ and $-0.1$), and hence, even when the ruin does not occur, it is frequent that the arm with the largest expectation $0.04$ has exhausted its budget share $\frac{B}{K} = 3$, while the other arm of positive expectation $0.01$ has not. In such cases, EXPLOIT-UCB is falls in the trap of pulling arm $F^{(4)}$ permanently, implying a large survival regret. The closeness of the expectations of the arms in this setting also explains the low performance of MTS, which suffers frequent ruins, with a very low average time of ruin around $\tau = 5000$ and a ruin which occurs on $3/4$ of the simulations (see Tables \ref{tab:average_ruin} and \ref{tab:proportion_ruins}).

\begin{figure}[t]
\centering
\includegraphics[scale=0.8]{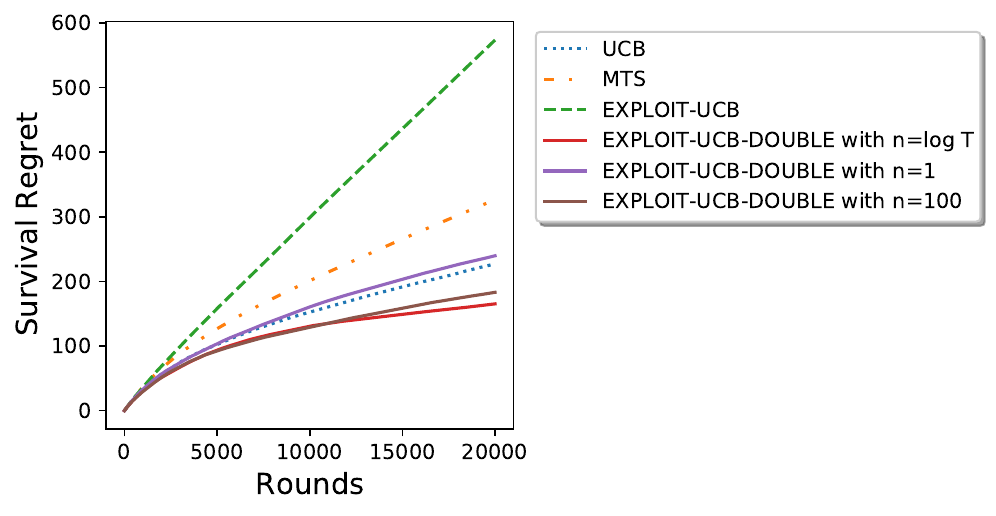}
\caption{Survival regret for $B=9$ and arms $\{F^{(6)}, F^{(7)}, F^{(8)}\}$.}
\label{fig:smab:long_full3}
\end{figure}

In Figure \ref{fig:smab:long_full3}, the budget is still set to $B = 9$ while the arm distributions are $\{F^{(6)}, F^{(7)}, F^{(8)}\}$. The arm of distribution $F^{(8)}$ has an expectation equal to $0$ and a probability of ruin equal to $0$. The arms of distributions $F^{(6)}$ and $F^{(7)}$ have relatively close expectations $0.1$ and $0.05$, but very different probabilities of survival $0.18$ and $0.5$ respectively. While the results in this case are very similar to the ones of Figure \ref{fig:smab:long_full1}, there are two notable differences we would like to point out. The first one is that in this setting, while the survival regret of UCB is close to the one of EXPLOIT-UCB-DOUBLE with parameter $n=1$, it is much larger than the one of EXPLOIT-UCB-DOUBLE with parameter $n = 100$ or $n = \lceil\log T\rceil$. Theoretically, $n$ is the exploration parameter of EXPLOIT-UCB-DOUBLE: the larger the parameter $n$, the later the exploration. Therefore, if $n$ is very small, EXPLOIT-UCB-DOUBLE will start the exploration at a very early stage, and its behavior will not be so different from the one of UCB. In this case, it will start pulling arm $F^{(6)}$ too early, instead of focusing on the safer arm $F^{(7)}$ which has a large probability of survival $0.5$. This considerably increases its risk of ruin, and hence increases its survival regret to UCB level. On the other hand, if $n$ is very large, then the exploration will be delayed and this will result in a lower risk of ruin, explaining the better performance of EXPLOIT-UCB-DOUBLE for $n = \lceil\log T\rceil$ and $n = 100$.

The second phenomenon we wish to point out is that in this setting, EXPLOIT-UCB has a survival regret which is even larger than the one of MTS, contrary to the setting of Figure \ref{fig:smab:long_full1}, where both algorithms had a similar survival regret. The reason behind this is that this setting is easier compared to the one of Figure \ref{fig:smab:long_full1}, in the sense that the arm expectations are quite different. As a result, having a stronger exploration component yields a better regret than being too conservative.

Overall, this phenomenon can be easily observed looking carefully at the three previous figures. In Figure \ref{fig:smab:long_full2}, the arm expectations are very small (respectively $0.04$ and $0.01$ for arms $1$ and $2$), making this setting very difficult in practice. As a result, algorithms with a strong exploitation component perform better and in particular, EXPLOIT-UCB performs much better than MTS. Then, the setting of Figure \ref{fig:smab:long_full1} is a little easier to handle, with arms whose expectations are slightly larger (respectively $0.08$ and $0.04$ for arms $1$ and $2$), and in this setting, EXPLOIT-UCB and MTS perform comparably. Eventually, in Figure \ref{fig:smab:long_full3}, the arms expectations are very distinct, being respectively $0.1$ and $0.05$ for arms $1$ and $2$, and in this setting, MTS outperforms EXPLOIT-UCB.

The reason underlying this phenomenon is that, while the risk of ruin of MTS is always higher than the one of EXPLOIT-UCB, the difference in the risk of ruin between the two algorithms becomes lower as the arm expectations grow: Table \ref{tab:proportion_ruins} shows that the difference between the proportion of ruins of MTS and EXPLOIT-UCB goes from $0.35$ for the setting of Figure \ref{fig:smab:long_full2}, to $0.22$ for the setting of Figure \ref{fig:smab:long_full1}, to only $0.11$ for the setting of Figure \ref{fig:smab:long_full3}. The lesson to remember from this is that, depending on the arm parameters, the strong performance of a good exploration may sometimes compensate for frequent ruins.

\subsection{The Case of Large Budget}

In Figures \ref{fig:smab:long_full4} and \ref{fig:smab:long_full55}, we eventually study if some better performance can be achieved by traditional algorithms when the budget is large, i.e. of the order of $\log T$. 

In Figure~\ref{fig:smab:long_full4}, we take the setting of Figure~\ref{fig:smab:long_full1} with arm distributions $\{F^{(1)}, F^{(2)}, F^{(3)}\}$, but we increase the budget from $B = 9$ to $B = 30$. Previously, MTS suffered from many ruins and its survival regret was the largest among all the algorithms studied. In the same setting but with such a large budget, the ruin never happens, as we can see in Table \ref{tab:proportion_ruins}. Therefore, it is natural that MTS, which has the best theoretical and practical performance in this case, outperforms the other algorithms. Nevertheless, please note that the survival regret of EXPLOIT-UCB-DOUBLE, whatever the parameter $n$, is comparable to the performance of UCB. Another fact is that, the lower the parameter $n$ of EXPLOIT-UCB-DOUBLE, the more it explores and as a consequence, EXPLOIT-UCB-DOUBLE performs better in this case when the parameter $n$ is small.

\begin{figure}[t]
\centering
\includegraphics[scale=0.8]{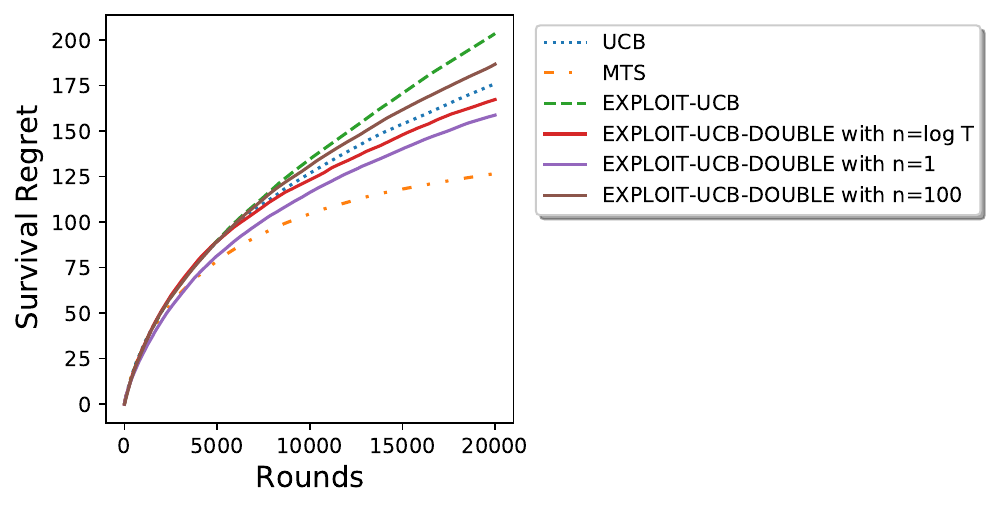}
\caption{Survival regret for $B=30$ and arms $\{F^{(1)}, F^{(2)}, F^{(3)}\}$.}
\label{fig:smab:long_full4}
\end{figure}

Please note that in some settings though, even such a large budget cannot prevent the ruin from happening. Looking carefully at Figure \ref{fig:smab:long_full55}, for which the budget is still set to $B = 30$ and the arm distributions are $\{F^{(9)}, F^{(10)}, F^{(3)}\}$, it is clear that the stronger the exploration, the larger the survival regret. In this setting, the arm expectations are very small, and therefore, a lack of exploitation leads to very frequent ruins. In this setting, EXPLOIT-UCB is the only algorithm which suffers a ruin less than for half of the simulations, while the proportion of ruins of EXPLOIT-UCB-DOUBLE is around $0.57$ (whatever the parameter $n$). This is in stark contrast to UCB and MTS, which suffer a proportion of ruins around $0.68$ and $0.76$ respectively.

\begin{figure}[t]
\centering
\includegraphics[scale=0.8]{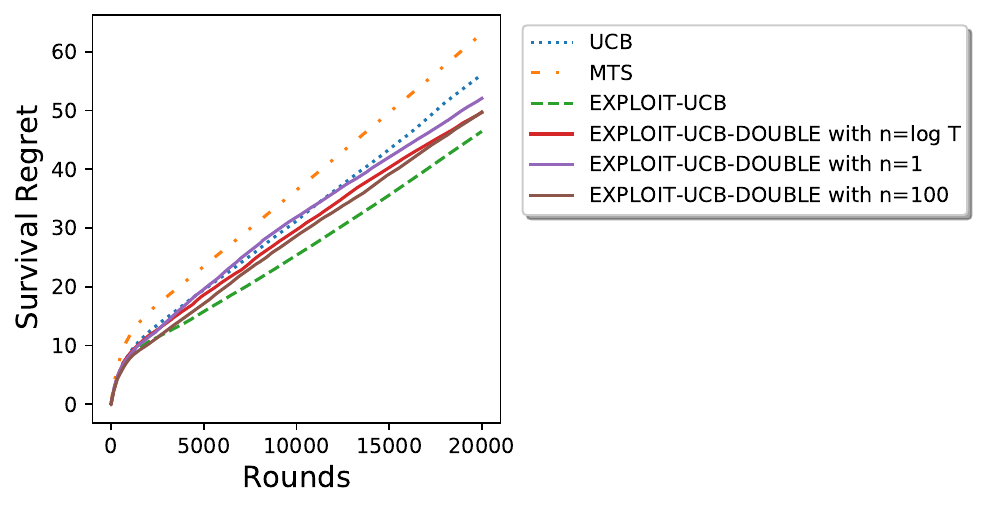}
\caption{Survival regret for $B=30$ and arms $\{F^{(9)}, F^{(10)}, F^{(3)}\}$.}
\label{fig:smab:long_full55}
\end{figure}

\end{document}